\documentclass[12pt]{article}
\usepackage{preprint-layout} 
\usepackage{preprint-notation} 
\usepackage{hyperref}
\usepackage[english]{babel}
\usepackage{doi}

\DeclareMathOperator{\vol}{vol}
\DeclareMathOperator{\sgn}{sign}

\title{Implicit Hypersurface Approximation Capacity in Deep ReLU Networks}

\usepackage{bm} 

\newcommand{\BR}{B_R^d}

\newcommand\labelAndRemember[2]
  {\expandafter\gdef\csname labeled:#1\endcsname{#2}%
   \label{#1}#2}
\newcommand\recallLabel[1]
   {\csname labeled:#1\endcsname\tag{\ref{#1}}}
   
 \newtheoremstyle{TheoremNum}
        {\topsep}{\topsep}              
        {\itshape}                      
        {}                              
        {\bfseries}                     
        {.}                             
        { }                             
        {\thmname{#1}\thmnote{ \bfseries #3}}
    \theoremstyle{TheoremNum}

\newtheoremstyle{TheoremNum}
        {\topsep}{\topsep}              
        {\itshape}                      
        {}                              
        {\bfseries}                     
        {.}                             
        { }                             
        {\thmname{#1}\thmnote{ \bfseries #3}}
    \theoremstyle{TheoremNum}
    \newtheorem{lemn}{Lemma} 
    
    \newtheoremstyle{TheoremNum}
        {\topsep}{\topsep}              
        {\itshape}                      
        {}                              
        {\bfseries}                     
        {.}                             
        { }                             
        {\thmname{#1}\thmnote{ \bfseries #3}}
    \theoremstyle{TheoremNum}
    \newtheorem{propn}{Proposition}

\newcommand\restr[2]{{
  \left.\kern-\nulldelimiterspace 
  #1 
  \littletaller 
  \right|_{#2} 
  }}

\newcommand{\littletaller}{\mathchoice{\vphantom{\big|}}{}{}{}}

\author{Jonatan Vallin, \quad Karl Larsson, \quad Mats G. Larson}

\date{\today}
\begin{document}
\maketitle 
\begin{abstract}
We develop a geometric approximation theory for deep feed-forward neural networks with ReLU activations. Given a $d$-dimensional hypersurface in $\IR^{d+1}$ represented as the graph of a $C^2$-function $\phi$, we show that a deep fully-connected ReLU network of width $d+1$ can implicitly construct an approximation as its zero contour with a precision bound depending on the number of layers.
This result is directly applicable to the binary classification setting where the sign of the network is trained as a classifier, with the network's zero contour as a decision boundary.
Our proof is constructive and relies on the geometrical structure of ReLU layers provided in \cite[\doi{10.48550/arXiv.2310.03482}]{vallin2023geometry}. Inspired by this geometrical description, we define a new equivalent network architecture that is easier to interpret geometrically, where the action of each hidden layer is a projection onto a polyhedral cone derived from the layer's parameters. By repeatedly adding such layers, with parameters chosen such that we project small parts of the graph of $\phi$ from the outside in, we, in a controlled way, construct a network that implicitly approximates the graph over a ball of radius $R$. The accuracy of this construction is controlled by a discretization parameter $\delta$ and we show that the tolerance in the resulting error bound scales as $(d-1)R^{3/2}\delta^{1/2}$ and the required number of layers is of order $d\big(\frac{32R}{\delta}\big)^{\frac{d+1}{2}}$.
\end{abstract}



\section{Introduction}
Deep neural networks (DNNs) consist of a sequence of nonlinear activations, allowing them to learn complex patterns in data. In recent years, DNNs have been extremely successful in many areas, including image recognition \cite{krizhevsky2012imagenet, he2016deep}, natural language processing \cite{wu2016google, brown2020language, vaswani2017attention, ray2023chatgpt, bubeck2023sparks} and reinforcement learning \cite{silver2017mastering}. Due to these remarkable accomplishments, the theoretical properties of DNNs have regained much attention lately. One such theoretical aspect is their approximation capacity, which is the central topic of this paper. It was already known in the 1980s that neural networks with a single hidden layer can approximate any continuous function arbitrarily well on compact domains as long as there is no limitation on the number of units \cite{cybenko1989approximation}. Nowadays, deep learning (techniques using networks with multiple layers) has achieved state-of-the-art performance in many machine learning areas, and these deep networks have outperformed their shallow predecessors. Hence, a rigid mathematical theory for deep networks is highly desirable.

\paragraph{Contributions.} In this paper, we develop a geometric approximation theory for deep fully-connected networks with the Rectified Linear Unit (ReLU) activation. The proof is constructive and builds upon our earlier work \cite{vallin2023geometry} on the geometrical structure of fully-connected ReLU layers. Our main results can be summarized as follows:
\begin{itemize}

\item Any $d$-dimensional hypersurface $\Gamma_\phi \subset \IR^{d+1}$ represented as the graph of a sufficiently regular function $\phi:\Omega\to \IR$, where $\Omega\subset \IR^d$ is a bounded domain with diameter $R$, can be implicitly approximated to any accuracy using a deep ReLU network $F:\Omega\times\IR\to\IR$ of constant width $d+1$. The proof is based on explicitly constructing such a network, and the accuracy is controlled via a discretization parameter~$\delta$.

\item For points further than a tolerance $\varepsilon$ from $\Gamma_\phi$,
the sign of the constructed network correctly classifies whether a point in $\Omega\times\IR$ lies above or below $\Gamma_\phi$. This tolerance depends on  $\delta,R,d$ as
\begin{align}
\varepsilon \lesssim (d-1)R^{3/2}\delta^{1/2}
\end{align}
As a consequence, the zero contour $\Gamma$ of the network $F$ will approximate $\Gamma_\phi$ in the sense that any point on $\Gamma$ will be closer than $\varepsilon$ from $\Gamma_\phi$.

\item The number of layers $N$ required for this construction satisfies the bound
\begin{align}
N \lesssim d\biggl(\frac{32R}{\delta}\biggr)^{\frac{d+1}{2}}
\end{align}

\end{itemize}
These results directly apply to the binary classification setting where the objective is to train a network such that its sign is a classifier for two classes of data points, and the network's zero contour is a decision boundary separating the two classes.

\paragraph{Related Works.} Several papers concerning the approximation ability and complexity of DNNs have been published in recent years. For instance, in \cite{montufar2014number}, it is shown that the number of linear pieces can grow exponentially with the number of layers in deep ReLU networks in contrast to shallow ones. However, as also noted in \cite{vallin2023geometry}, the linear pieces are not entirely independent of each other, and this feature might improve the generalization ability in deep multi-layered networks. In contrast, \cite{hanin2019deep} provides theoretical and experimental evidence that the number of activation regions (closely related to the linear regions) of the actual functions learned by deep networks during training is considerably fewer than the theoretical maximum. Along the same lines, it is demonstrated in \cite{ba2013deep} that shallow networks can be trained to mimic deep neural nets and thus obtain similar accuracy. However, that analysis is empirical and limited to only two data sets. Also, the shallow networks are trained on the outputs from the pre-trained deep networks. Still, when training directly on the actual data sets, the accuracy drops significantly for the shallow models. Accordingly, it is suggested that the success of deep learning might partly be explained by these deep architectures fitting the current training algorithms well and thus making learning easier.

Other works, such as \cite{telgarsky2015representation, telgarsky2016benefits}, provide classification problems that can be solved by deep networks much more efficiently than what is possible for shallow networks. More generally, in \cite{yarotsky2017error}, it is proven that deep networks can approximate smooth functions in Sobolev-type spaces more efficiently than shallow models. Thus, one possible explanation for the success of deep learning might be that deep networks are much more efficient at computing functions typically encountered in applications.

A universal approximation result for deep ReLU networks is provided in \cite{lu2017expressive} where it is shown that any Lebesgue integrable function in $d$ variables can be approximated to any given precision by a deep ReLU network as long as the widths (i.e., the number of units in each layer) are greater than or equal to $d+4$. Moreover, it is also proven that this universal approximation property breaks down when the widths are bounded above by the input dimension $d$. Hence, even in the limit of infinite depth, ReLU networks of widths bounded by the input dimension have limited approximation capacities. The minimal width allowing universal approximation was later settled in \cite{hanin2017approximating}, where it is shown that deep ReLU networks of width $d+1$ can approximate any real-valued continuous function on compact domains. That proof relies on a specific construction of the layers, allowing building chains of nested max and min operations of affine functions, which require the parameters in each layer to attain specific values. More precisely, the first $d$ units in all layers are unaffected in their construction and thus always hold the input variables, which can be repeatedly reused to build their approximation in the $d+1$ unit. Related results concerning universal approximation for different function classes and activation functions are provided in \cite{park2020minimum}.

Other works consider networks in the classification setting and provide results concerning topological properties of the decision regions, that is, the subsets of the input space $\IR^d$ where the network will predict each class label. For instance, in \cite{nguyen2018neural}, it is shown that pyramidal networks, where the widths form a non-increasing sequence bounded above by the input dimension $d$, cannot realize disconnected decision regions for certain activation functions, as leaky-ReLU, if the weight matrices have full rank. Thus, all possible decision regions are connected for such networks; hence, these models are limited as classifiers. A related result is found in \cite{beise2021decision}, where it is proven that networks with strictly monotonic or ReLU activations of width bounded by the input dimension $d$ can only generate unbounded decision regions.

The effect of the geometrical properties of decision boundaries in relation to the robustness of the networks, as classifiers, has also been investigated \cite{fawzi2018empirical,moosavi2019robustness}. More precisely, it is shown that networks with decision boundaries with large curvature will be more sensitive to perturbations in the input data. In the related study \cite{liu2022some}, tools from differential geometry are used to provide conditions on the network parameters guaranteeing a flat or developable decision boundary. Other works, such as \cite{lee2024defining, berzins2023polyhedral, huchette2023deep}, consider the polyhedral structure of ReLU networks and its relationship to the geometrical structure of data sets in the classification setting.

Decision boundaries of ReLU networks have also been studied through the perspective of tropical geometry where the network complexity and other properties are related to tropical objects \cite{pmlr-v80-zhang18i, piwek2023exact, alfarra2022decision}. For instance, it is shown that the decision boundary will lie inside the convex hull of two zonotopes defined by the parameters of the network.

\paragraph{Outline.}
In Section~\ref{sec:2}, we describe the problem of implicitly approximating a hypersurface using deep ReLU networks, describe our main result on the approximation capacity of such networks, and outline the steps in its proof. We also comment on the applicability of this result to the binary classification problem.
In Section~\ref{sec:structure}, we review how the layers in the network can be geometrically interpreted as projections onto hypercones.
In Section~\ref{sec: projecting the graph}, we project pieces of the hypersurface onto hyperplanes, starting from the outside and moving inwards until only a final small piece remains. By the geometric description, each projection gives the parameters of a layer in the network. In Section~\ref{sec:Decision Boundary}, we show how to approximate the final central piece using a hyperplane, yielding the parameters of the final layer, and how the preimage of the complete sequence of layers is an approximation to the hypersurface.
Finally, in Section~\ref{sec:conclusions}, we give some conclusions.

\section{Approximation Problem}
\label{sec:2}

The problem we are considering is the implicit approximation of a $d$-dimensional hypersurface in $\IR^{d+1}$ by the zero contour of a deep neural network of width $d+1$. We show that such an approximation can be constructed to any desired accuracy, given a sufficient number of layers in the network. In this section, we begin by defining the hypersurface and specifying how the approximation to the hypersurface is modeled via a specific deep neural network. We present our main approximation result and summarize the main steps taken to prove this result. Finally, we describe how this applies to the binary classification problem.

\begin{definition}[Exact Hypersurface] \label{def:exact-surface}
Let $\BR$ be an open ball in $\IR^d$ of radius $R$ centered at the origin, and let $\phi:\BR\to\IR$ be a $C^2(\BR)$ function for which the second derivatives are bounded by a constant $D$.
The exact hypersurface in $\IR^{d+1}$ is defined as the graph of $\phi$;
\begin{align}
\boxed{
\Gamma_\phi = \{ (x,y) \in \BR\times\IR \, : \, y=\phi(x) \}
}
\end{align}
\end{definition}

\begin{definition}[Hypersurface Approximation Representation] \label{def:approx-surface}
 Approximative hypersurfaces are implicitly represented as the zero contour of a deep neural network $F:\BR\times\IR \to \IR$, i.e.,
 \begin{align} \label{eq:approx-hypersurface}
  \boxed{
  \Gamma = \{ (x,y) \in \BR\times\IR \,:\, F((x,y))=0 \}
  }
  \end{align}
  where the architecture of $F$ is defined as follows.
  Let $F:\BR\times\IR \to \IR$ be a fully-connected ReLU network of width $d+1$ and depth $N$ (illustrated in Figure~\ref{fig:fcc_2}), written as the composition
  \begin{align}
  F(x)=L \circ T_N \circ \cdots \circ T_1(x)
  \label{eq:deepnet}
  \end{align}
  Here, $L:\IR^{d+1}\to \IR$ is an affine function, and $T_k:\IR^{d+1}\to \IR^{d+1}_+$, $k=1,\hdots, N$, are functions, so-called ReLU layers, of the form
  \begin{align} 
  T_k(x)=\mathrm{ReLU}(A_kx+b_k)
  \label{def:T_n}
  \end{align}
  where the activation function $\mathrm{ReLU}(x)=\max(x,0)$ is applied elementwise to the vector $A_kx+b_k$,
  and $\IR^d_+$ denotes the non-negative orthant of $\IR^d$.
  Hence, the parameters of the network are the matrices $A_k\in \IR^{(d+1) \times (d+1)}$, the bias vectors $b_k\in \IR^{d+1}$, and the parameters of the affine function $L$.
\end{definition}

\begin{rem}[Piecewise Linear Hypersurface Approximation.]
 Since the ReLU activation is a piecewise linear function it follows directly that a fully-connected ReLU network is itself a piecewise linear function, and in turn $\Gamma$ will be a piecewise linear hypersurface. However, the partition of the input domain on which the network is piecewise linear depends non-trivially on the network parameters, and the same goes for the polygonal pieces of the hypersurface.
\end{rem}

\begin{figure}
\centering
\includegraphics[scale=0.5]{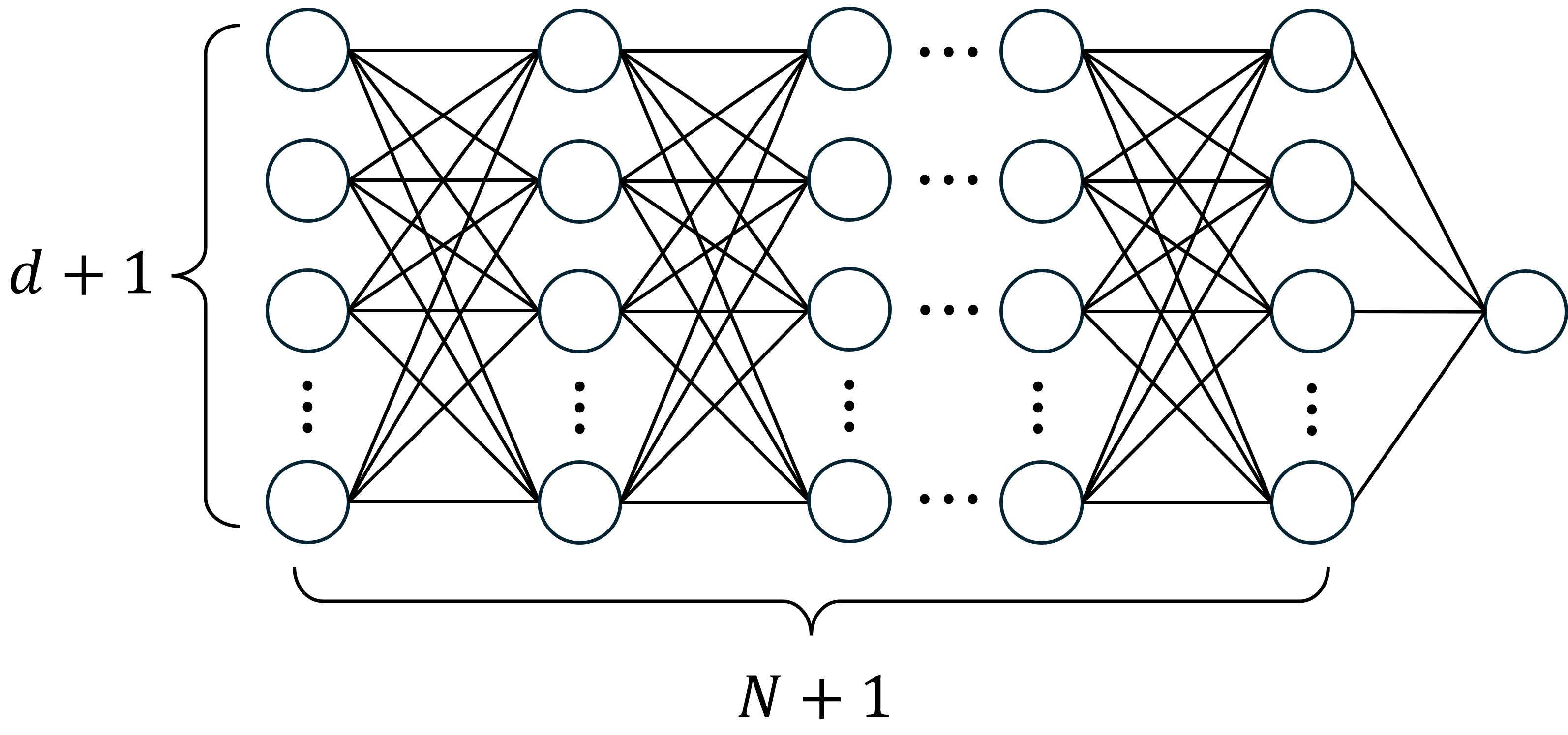}
\caption{\textit{Fully-Connected Network.} The schematics of a fully-connected network of constant width $d+1$ and $N$ hidden layers.
}
\label{fig:fcc_2}
\end{figure}

\subsection{Main Result}

We here present our main approximation result and then outline the main steps taken in its proof.
To express the approximation result, we first introduce the concept of a $\varepsilon$-band, which is used for describing a $(d+1)$-dimensional neighborhood close to the exact hypersurface $\Gamma_\phi$.

\begin{definition}[\boldmath{$\mathbf{\varepsilon}$}-band]
  \label{def:eps_band}
  Let $f:\Omega \to \IR$ where $\Omega\subseteq \IR^d$. Given a set $\omega \subseteq \IR^d$ and a non-negative constant $\varepsilon$ we define the $\varepsilon$-band of $f$ over $\omega$ as 
  \begin{align} \label{eq:eps_band}
  \Sigma_{f}^{\varepsilon}(\omega)=\{(x,y)\in (\Omega \cap \omega)\times \IR: |f(x)-y|\leq \varepsilon\}
  \end{align}
  \end{definition}

\begin{rem}
\label{rem:band}
If we have two continuous functions $f,g:\Omega\to \IR$ for some $\Omega \subset \IR^d$ and the graph $\Gamma_{g}$ of $g$ restricted to a subset $\omega\subseteq \Omega$ fulfills $\Gamma_g \subseteq \Sigma_{f}^{\varepsilon}(\omega)$ for some $\varepsilon\geq 0$ then it is clear that
\begin{align}
\label{eq:rem-band}
\sup_{x\in \omega}|f(x)-g(x)|\leq \varepsilon
\end{align}
Likewise, if \eqref{eq:rem-band} holds then $\Gamma_g \subseteq \Sigma_{f}^{\varepsilon}(\omega)$.
\end{rem}

\begin{lem}[Implicit Approximation Capacity] \label{lem:implicit-approximation-capacity}
Let $\phi \in C^2(\BR)$ be a scalar valued function with graph $\Gamma_\phi$, see Defintion~\ref{def:exact-surface}, and let $\Sigma_\phi^\varepsilon(\BR) \subset \BR\times\IR$ be the $\varepsilon$-band covering $\Gamma_\phi$
, see Definition~\ref{def:eps_band}.
Given a positive discretization parameter $\delta\leq R\left(1-2^{-1/2}\right)$, there exists a deep fully-connected ReLU network $F:\BR \times \IR \to \IR$, see Definition~\ref{def:approx-surface}, of constant width $d+1$ and a depth of at most $N=Cd\big(\frac{32R}{\delta}\big)^{\frac{d+1}{2}}$, such that 
\begin{align}
\boxed{
\sgn(F(x,y)) = \sgn(\phi(x) - y) \,, \quad (x,y) \in (\BR \times \IR) \setminus \Sigma_\phi^\varepsilon(\BR)
}
\end{align}
with a tolerance
\begin{align}
\varepsilon \leq C_5(d-1)R^{3/2}\delta^{1/2}
\end{align}
where $C_5$ is a constant independent on $\delta, R$ and $d$.
The zero contour to $F$ hence satisfies
\begin{align}
\boxed{
\Gamma
\subseteq \Sigma_\phi^\varepsilon(\BR)
}
\end{align}
\end{lem}

\begin{rem}[$\boldsymbol{\Gamma}$ as the Graph of a Function]
  Below, we prove a slightly stronger version of this approximation bound -- Theorem~\ref{thm:main-result}.
  There, 
  we by construction show that there exists a network $F$ of depth $N$ such that its zero contour $\Gamma$ coincides with the graph of a continuous piecewise linear function $\hat\phi:B_R^d \to \IR$ fulfilling the more direct error bound
  \begin{align}
   \sup_{x\in \BR}|\phi(x)-\hat{\phi}(x)|\leq \varepsilon
   \end{align}
  which in turn implies Lemma~\ref{lem:implicit-approximation-capacity}.
  In practical applications, the parameters of the network are trained rather than set according to our construction, so such a $\hat\phi$ corresponding to a trained network $F$ might not always exist. This, however, does not limit the applicability of the theorem since it still gives a lower bound on the approximation capacity.
\end{rem}

\subsection{Proof Outline}
As the proof of Theorem~\ref{thm:main-result} is rather extensive we will here give a brief description of the key steps involved and an overview of the structure of this paper. The overall idea is that the layers in the network will successively project small pieces of $\Gamma_{\phi}$ onto hyperplanes in a controlled way such that its image will consist of a small unaffected central piece of the original graph and a projected part contained in a small neighborhood around its boundary. We will then approximate the remaining piece by the hyperplane defined by the kernel of the last affine function in the network. As the decision boundary is the preimage of this hyperplane we will show that if the aforementioned neighborhood is small then the decision boundary will be close to $\Gamma_{\phi}$.
\begin{itemize}

\item In Section~\ref{sec:structure} we start by reexamining the geometrical structure of standard ReLU layers, detailed in \cite{vallin2023geometry}. Based on this geometrical description, we in Section~\ref{sec:modified-structure} propose a modified network architecture, equivalent to the fully-connected ReLU networks we are interested in, with a greatly simplified action of the layers. In Section~\ref{sec:compact_domains} we show that, given a half-space $\mathbf{U}$ and a vector $\xi$, such a layer can realize a map that is the identity on $\mathbf{U}$ while $\mathbf{U}^c$ is projected along $\xi$ onto the boundary hyperplane, see Figure~\ref{fig:projection-bounded}. Hence, we can identify each layer by a half-space and a projection direction.

\item In Section~\ref{sec: projecting the graph} we describe how to choose each half-space and projection direction and how $\Gamma_{\phi}$ evolves when passing through the network. Each half-space will be chosen orthogonal to $\IR^d$ and in this setting the part of the graph projected by one layer will be contained in a neighborhood around the intersection of the graph and the corresponding hyperplane, see Figure~\ref{fig:Projectedfirst}. 

We will group the layers in the network in such a way that the half-spaces associated to each group define a convex polytope $\mathcal{P}$ and the projected part of $\Gamma_{\phi}$ will be contained in a neighborhood over the boundary $\partial \mathcal{P}$ as depicted in Figure~\ref{fig:Projection-first-poly}. In this way, we will get a nested sequence of polytopes and the number of them (i.e., the number of groups of layers) is chosen such that the last one is contained in the closed ball $\bar{B}^d_{\delta}$ for a predefined parameter $\delta>0$, see Figure~\ref{fig:Sequence of Polytopes}. We will estimate the size $\varepsilon=\varepsilon(\delta)$ of the final neighborhood we obtain.

\item In Section~\ref{sec:Decision Boundary} we start by showing how to define the last affine function in our network such that its kernel approximates $\Gamma_{\phi}$ over the last polytope, see Figure~\ref{fig:Hyperplane}. Then, we describe the preimages of the layers in our network. We show that the resulting decision boundary will be the graph of a continuous piecewise linear function, defined on $B^d_R$, that $2\varepsilon$-approximates $\phi$, see Figure~\ref{fig:decision boundary-approx}.
\end{itemize} 
 
\subsection{Application to Binary Classification} 
\label{sub-section:Binary Classification}
A commonly encountered problem in machine learning is that you have a data set as a finite set of points where each point comes with a corresponding class label. When there are two possible labels we call it a binary classification problem and the task is to build a model that correctly classifies labels of unseen data.
We here describe how our approximation result applies to that setting.

\paragraph{Binary Classification.}
We assume that two bounded disjoint sets $X_1,X_2 \subset\IR^d$ are associated with two respective classes. These two sets represent the ground truth in the sense that all possible data points from the first class are assumed to be elements in $X_1$ and points from the second class are always elements in $X_2$. Since they are bounded, there is a point $p\in \IR^d$ and a $R>0$ such that $X_1\cup X_2 \subseteq B^d_R(p)$ where $B^d_R(p)$ is the $d$-dimensional closed ball of radius $R$ centered at $p$. Without loss of generality, we may assume that the ball is centered at the origin and simply denote it by~$B^d_R$.

If there exists a function $\phi:B^d_R \to \IR$ such that
\begin{align}
  \begin{cases}
  \phi(x)>0 \quad\text{if $x\in X_1$}\\ \phi(x)<0 \quad\text{if $x \in X_2$}
  \end{cases}
  \label{eq:levelsetphi}
\end{align}
then the sign of $\phi(x)$ is a classifier on $B^d_R$ and the hypersurface 
\begin{align}
  \gamma_\phi = \{ x\in B^d_R \, : \, \phi(x)=0 \}
\end{align}
is a decision boundary separating the two datasets.
Note that $\gamma_\phi$ does not need to generate connected (or simply connected) decision regions, see the example in Figure~\ref{fig:Datapoints}. Also note that $\gamma_\phi$ is not unique to $\phi$, since, for instance, $\phi$ multiplied with any strictly positive function has a sign that is a classifier with the same decision boundary.

\paragraph{Embedment in Higher Dimension.}
To each datapoint we append a $0$, embedding the data in $B^d_R \times \IR$.
This embedment in a higher dimension is required for applying our approximation result to the classification problem, and it has the well-known benefit of simplifying the topology of the decision boundary, see for instance \cite{nguyen2018neural, beise2021decision}, where it is shown that the decision regions of networks of width less than or equal to the input dimension are always connected and unbounded (with some additional assumption on the activation functions and parameter matrices).

An alternative interpretation is that one may add a linear embedding layer in the network, mapping $\IR^{d}$ to $\IR^{d+1}$, and then apply our result. This setting also occures when the data lies in a hyperplane in $\IR^{d+1}$ so that the points can be described using a $d$-dimensional coordinate system in that plane.

Given a function $\phi$ fulfilling \eqref{eq:levelsetphi}, its graph $\Gamma_\phi$ will be a decision boundary in the embedding space $B^d_R \times \IR$. Note that the resulting decision regions will always be simply connected as $\phi$ is bounded, where the datapoints of the two classes will lie above respectively below $\Gamma_\phi$, see the example in Figure~\ref{fig:Gammaphi}.

\paragraph{The Network as a Classifier.}
Given a sufficiently smooth $\phi$ and a tolerance $\varepsilon$, our approximation result gives the existence of a network $F$ guaranteed to correctly classify points outside an $\varepsilon$-band to $\Gamma_\phi$. Hence, its zero contour $\Gamma$, i.e., the network's decision boundary, is contained within that $\varepsilon$-band.
 
The error bound in our approximation result directly relates to the tolerance in the additional embedding dimension, whereas a more relevant tolerance for the binary classification problem is the distance to $\gamma_\phi$ within $B^d_R$. For $\phi$ with a very small gradient over $\gamma_\phi$, a very small tolerance in the embedding direction is required to achieve the desired tolerance in $B^d_R$.
However, given a sufficiently regular decision boundary $\gamma_\phi \subset B^d_R$, there exists a corresponding smooth $\phi$ that is a distance function in a neighborhood to $\gamma_\phi$, and hence the gradient over $\gamma_\phi$ is unit length which means that the tolerances in both the embedding dimension and in $B^d_R$ are equivalent.

\begin{rem}[Local Accuracy] \label{rem:local accuracy}
Our approximation result gives the existence of a network whose sign acts as a classifier with a tolerance $2\varepsilon$ everywhere throughout the domain. Most practical situations are however less strict, where the needed tolerance may vary greatly over the domain with high precision required only in a small part. This is illustrated in the example presented in Figure~\ref{fig:region}.
\end{rem}

\begin{figure}
  \centering
  \includegraphics[width=0.7\linewidth]{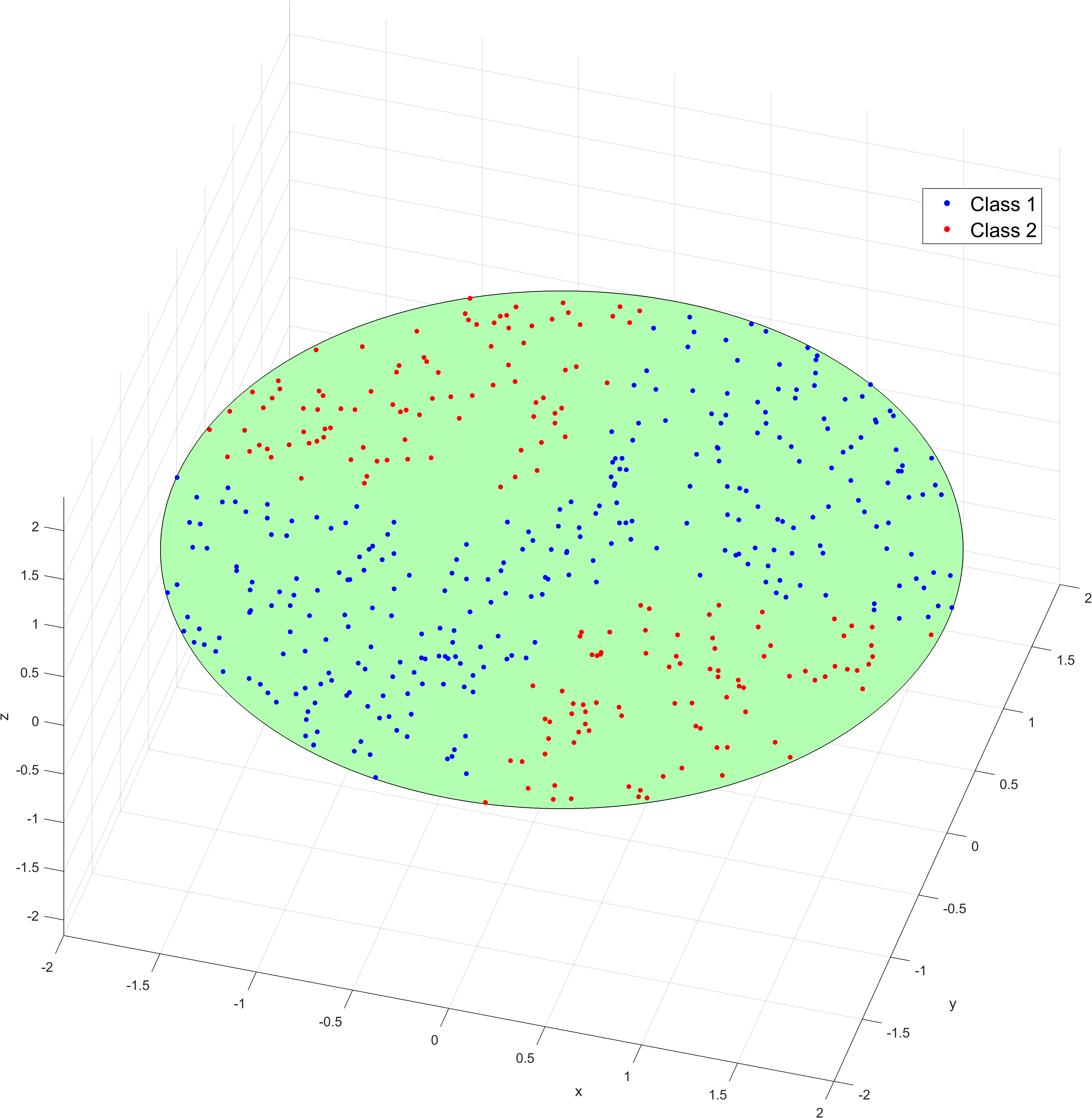}
  \caption{\textit{Binary Classification.} Samples from a data set consisting of two distinct classes of points. The data is contained in a ball $B^d_R$ and separated by a hypersurface (the union of the black curves).}
  \label{fig:Datapoints}
\end{figure}

\begin{figure}
\centering
\includegraphics[scale=0.4]{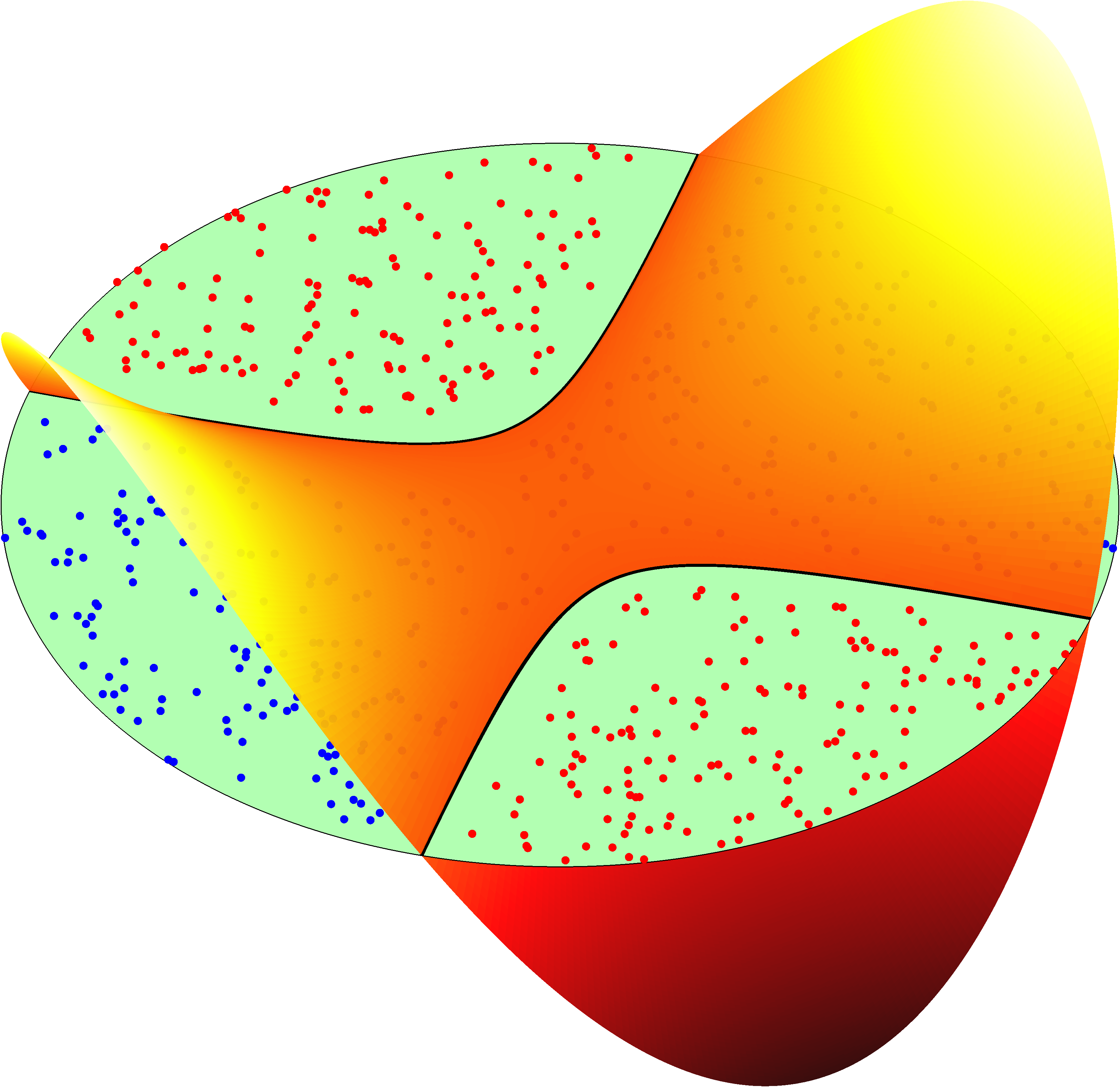}
\caption{\textit{Level-Set Function.} The separating hypersurface is represented as the zero-contour of a level-set function $\phi$. The graph of $\phi$ partitions $B^d_R\times \IR$ into two simply connected sets, each containing the embedded points of one class label.}
\label{fig:Gammaphi}
\end{figure}

\begin{figure}
  \centering
  \includegraphics[scale=0.95]{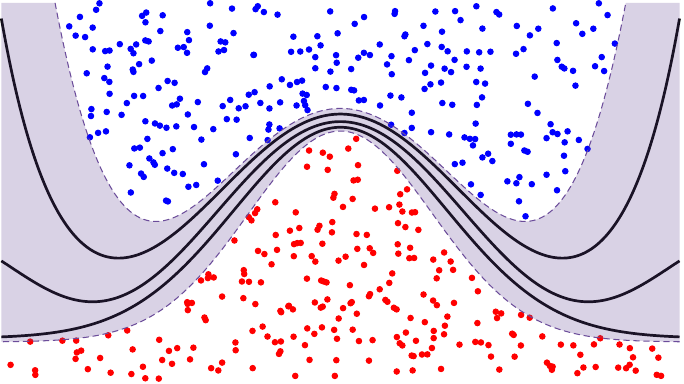}
  \caption{\textit{Separating Hypersurfaces.} Samples from a data set with two class labels and three possible separating hypersurfaces are depicted. These are contained in the region between the two disjoint sets of points. On parts of the domain where the two sets are close, the separating hypersurfaces are forced to approach each other in order to avoid intersecting the sets of points.}
  \label{fig:region}
\end{figure}

\section{Geometrical Structure of ReLU Layers}
\label{sec:structure}
By Definition~\ref{def:approx-surface} the network $F$, which implicitly approximates $\Gamma_{\phi}$, is constructed as a sequence of ReLU layers $T_k:\IR^{d+1} \rightarrow \IR^{d+1}_+$ of the form
\begin{align}
T_k(x)=\text{ReLU}(A_kx+b_k)
\label{eq:map_T}
\end{align}
where $A_k\in \IR^{(d+1)\times (d+1)}$ and $b_k\in \IR^{d+1}$, and the entire network $F:B^d_R\times \IR \to \IR$ can be written as
\begin{align}
F(x)=L\circ T_N\circ \hdots \circ T_1(x)
\label{eq:original_network}
\end{align}
for some $N$, where $L:\IR^{d+1}\to \IR$ is an affine function. In our earlier contribution \cite{vallin2023geometry}, we showed that maps on the form \eqref{eq:map_T} can be described geometrically as a projection onto a $(d+1)$-dimensional polyhedral cone followed by an affine transformation mapping the cone onto the non-negative orthant $\IR^{d+1}_+$. We will now give a brief summary of the description we provided there. 

\paragraph{Dual Vectors.}For now, we drop the indices on the ReLU layer in \eqref{eq:map_T} and just consider one single map $T(x)=\text{ReLU}(Ax+b)$. We define the index set $I=\{1,2,\hdots, d+1\}$ and we will assume that the matrix $A$ has full rank. Let $a_i\in \IR^{d+1}$ be the $i$:th row vector in $A$ and $b_i\in \IR$ the $i$:th component of $b$. Now we can define a set of dual vectors $\{a_i^*:i\in I\}$ by the equation
\begin{align}
a_j\cdot a_i^*=\delta_{ij}, \quad \text{for } i,j\in I
\label{eq:deltaij}
\end{align}
It follows that these dual vectors span $\IR^{d+1}$ since $A$ has full rank by assumption. Further, we also know that there is a point $x_0\in \IR^{d+1}$ which is the unique solution to the linear system of equations
\begin{align}
Ax_0+b=0
\label{eq:x_0}
\end{align}
\paragraph{Partition.} To understand the map $T$ geometrically we introduce a partition of $\IR^{d+1}$ using the dual vectors. Given two disjoint index subsets $I_+,I_-\subseteq I$ we define the set
\begin{align}
S_{(I_+,I_-)}=\{x\in \IR^{d+1}: x=x_0 + \sum_{i\in I_+} \alpha_i a^*_i - \sum_{i\in I_-} \alpha_i a^*_i,\alpha_i > 0\}
\label{eq:partition}
\end{align}
Every point in $\IR^{d+1}$ belongs to precisely one set $S_{(I_+,I_-)}$ for a suitable pair $(I_+,I_-)$. 
Thus, the family of sets of the form in \eqref{eq:partition} constitutes a partition of $\IR^{d+1}$. Note that the closure of the special set $S_{(I,\emptyset)}$ defines a $(d+1)$-dimensional polyhedral cone with apex at $x_0$. We will denote this cone simply by $S$
\begin{align}
S=\{x\in \IR^{d+1}: x=x_0 + \sum_{i\in I} \alpha_i a^*_i,\alpha_i\geq 0\}
\label{eq:cone}
\end{align}
\paragraph{Projection on Cones.} We showed in \cite{vallin2023geometry} that this polyhedral cone plays a central role in understanding the action of the map $T$ geometrically. More precisely, we can decompose the map as
\begin{align}
\boxed{
T=A_b \circ \pi
}
\label{eq:conical_decomposition}
\end{align}
where 
\begin{itemize}
\item $\pi:\IR^{d+1}\rightarrow S$ is a surjective map projecting $\IR^{d+1}$ onto the cone $S$,
\item $A_b:S\rightarrow \IR^{d+1}_+$ is a bijective affine transformation given by $A_b(x)=Ax+b$ mapping $S$ to $\IR^{d+1}_+$.
\end{itemize}
We call this decomposition of $T$ its \textit{conical decomposition}. Thus, apart from a translation, the nonlinearity of $T$ is fully captured by the projection $\pi$ which is piecewise defined on the sets $S_{(I_+,I_-)}$ in the following way. For a point $x\in S_{(I_+,I_-)}$ with expansion
\begin{align}
x=x_0+\sum_{i\in I_+}\lambda_i a_i^*-\sum_{i\in I_-}\lambda_i a_i^*
\end{align}
for some positive scalars $\lambda_i$, the projection is defined as 
\begin{align}
\pi(x)=\pi\bigg(x_0+\sum_{i\in I_+}\lambda_i a_i^*-\sum_{i\in I_-}\lambda_i a_i^* \bigg)=x_0+\sum_{i\in I_+}\lambda_i a_i^*\in S
\label{eq:projection_pi}
\end{align}
and by the equations \eqref{eq:deltaij}-\eqref{eq:x_0} we get
\begin{align}
T(x)=A_b\circ \pi(x)=\sum_{i\in I_+}\lambda_i e_i\in \IR^{d+1}_+ 
\end{align}
where $e_i$ is the $i$:th Euclidean basis vector. Thus, $\pi$ projects $S_{(I_+,I_-)}$ onto the set $S_{(I_+,\emptyset)} \subset S$ which is a $|I_+|$-dimensional face of the polyhedral cone and the projection is parallel to the subspace spanned by $\{a_i^*:i\in I_-\}$. Hence, all points $x\in \IR^{d+1}\setminus S$ outside the cone will be projected onto some part of its boundary, and points $x\in S$ in the cone will be left unaffected by $\pi$. Thus, $T$ reduces to an affine map on $S$. This gives a complete geometrical description of the action of the ReLU layer $T$. A polyhedral cone and the corresponding projection is depicted in Figure~\ref{fig:Cone-projection}.
\begin{figure} 
\centering
\includegraphics[scale=0.35]{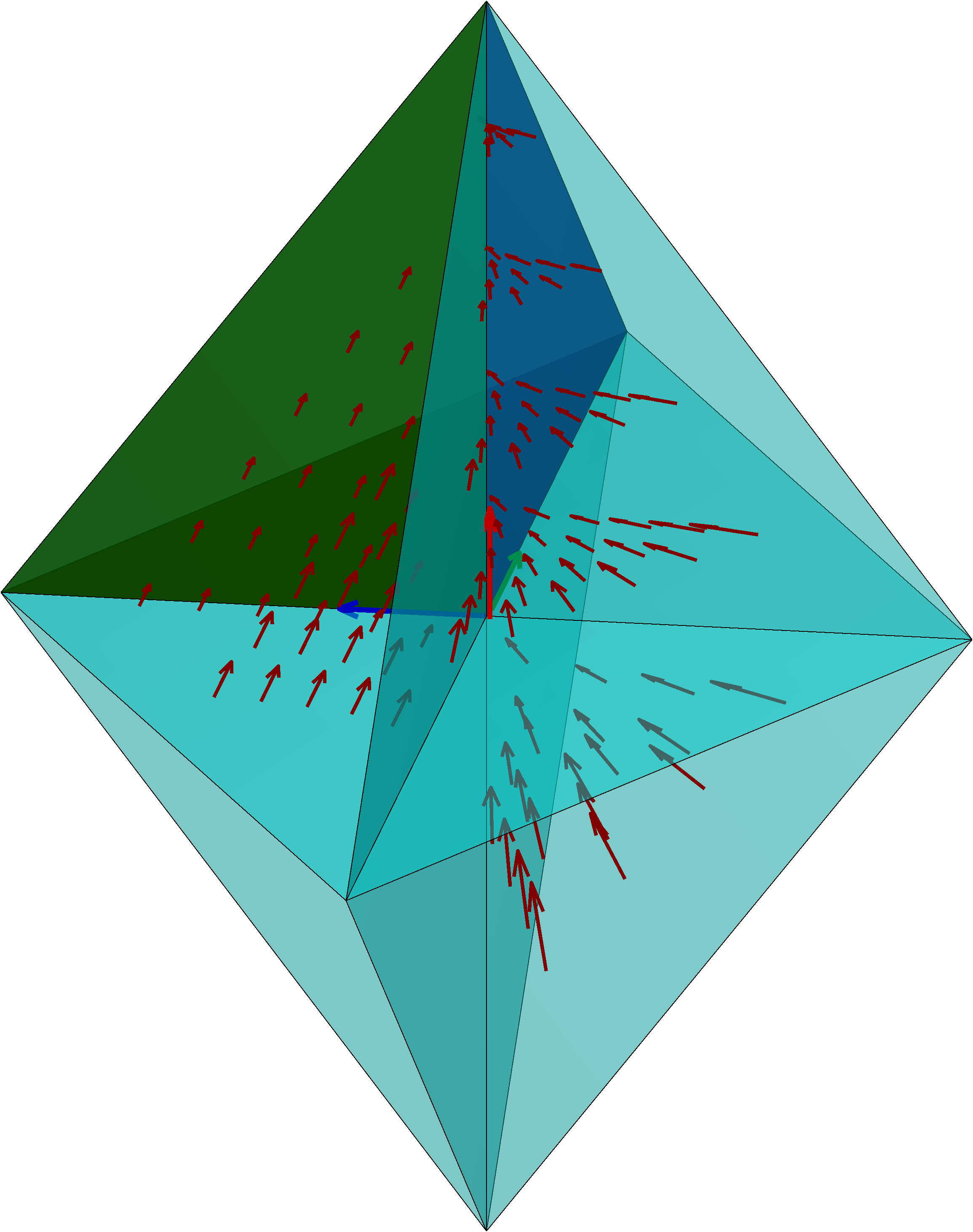}
\caption{\textit{Projection on a Polyhedral Cone.} Points are projected onto the polyhedral cone $S$ spanned by the dual vectors and with apex $x_0$. The dual vectors are parallel to the one dimensional edges of the cone, here depicted in blue, red and green. The projection is piecewise defined on the induced partition and the projection direction in each sector is parallel to a subset of the dual vectors. Projection directions in three different sets in the partition are illustrated.}
\label{fig:Cone-projection}
\end{figure}

\paragraph{Geometrical Interpretation.} The equations \eqref{eq:deltaij} and \eqref{eq:cone} give algebraic definitions of the dual vectors and the cone respectively. To see how these objects can be interpreted geometrically we can define the open half-spaces 
\begin{align}
\mathbf{U}_i=\{x\in \IR^{d+1}: a_i\cdot x+b_i \geq 0\}, \quad i\in I
\label{eq:halfspace}
\end{align}
with the corresponding hyperplanes 
\begin{align}
\mathbf{P}_i=\{x\in \IR^{d+1}: a_i\cdot x+b_i = 0\}, \quad i\in I
\label{eq:hyperplane}
\end{align}
as boundaries. Note that the row vector $a_i$ is a normal to $\mathbf{P}_i$. We can now express the cone as the intersection of these half-spaces
\begin{align}
S=\bigcap_{i\in I} \mathbf{U}_i
\end{align}
Thus, each hyperplane is tangent to one facet of the polyhedral cone. Moreover, if we define the lines
\begin{align}
L_j=\bigcap_{i\in I\setminus \{j\}} \mathbf{P}_i
\end{align}
it follows that the dual vector $a^*_j$ is parallel to the line $L_j$ and scaled such that $a_j\cdot a_j^*=1$, see Figure~\ref{fig:dual-basis}. Thus, the directions of the projection $\pi$ on each set $S_{(I_+,I_-)}$ is determined by the different lines of intersections (the 1-dimensional faces of $S$) in the hyperplane arrangement. 
\begin{figure} 
\centering
\includegraphics[scale=0.4]{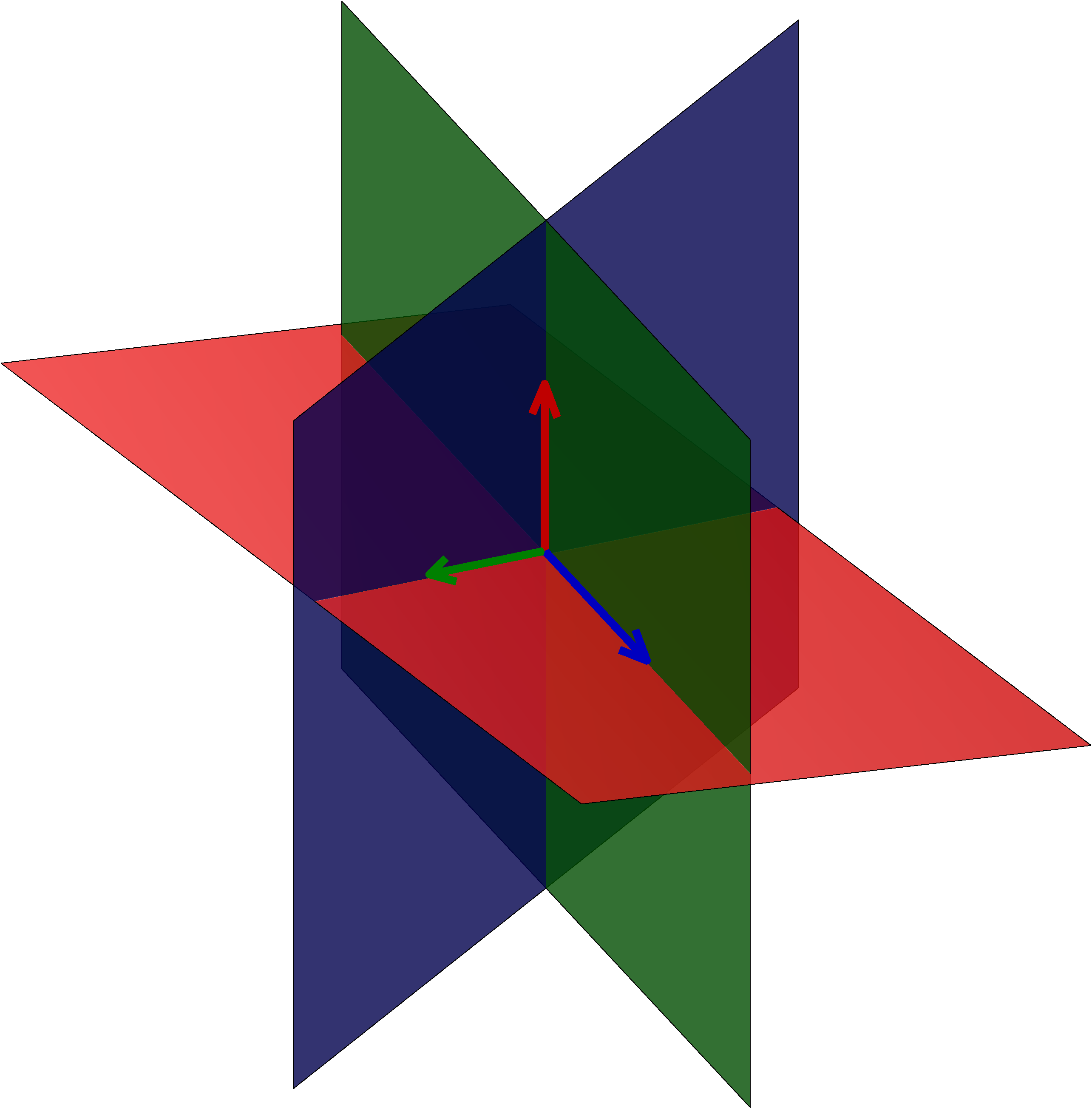}
\caption{\textit{Geometrical Construction.} Geometrically, the dual vectors are parallel to the lines of intersection of the hyperplanes defined by the parameters in the layer.}
\label{fig:dual-basis}
\end{figure}

\subsection{A Modified Network Architecture} 
\label{sec:modified-structure}


With the conical decomposition \eqref{eq:conical_decomposition} in mind we will now introduce a modified network architecture where we retain the projection but disregard the affine transformation in each ReLU layer. Thus, we consider networks $\tilde{F}$ of the form
\begin{align}
\boxed{
\tilde{F}(x)=\tilde{L} \circ \pi_N\circ \pi_{N-1} \circ \hdots \circ \pi_1(x)
}
\label{eq:modified_network}
\end{align}
where $\pi_k:\IR^{d+1}\rightarrow S_k$ for $k\in \{1,2,\hdots, N\}$ are projections and $\tilde{L}:\IR^{d+1}\to \IR$ is an affine function.

 Here, each $S_k$ is a polyhedral cone  
\begin{align}
S_k=\{x\in \IR^{d+1}:x=x_{k,0}+\sum_{i\in I}\alpha_ia^*_{k,i}, \alpha_i\geq 0\}
\end{align}
defined by an apex $x_{k,0}\in \IR^{d+1}$ and a set of $d+1$ linearly independent dual vectors $a^*_{k,i}\in \IR^{d+1}$ for $i\in I$. Recall that the definition of such a projection is given by equation \eqref{eq:projection_pi} and it is piecewise defined on the corresponding partition \eqref{eq:partition} generated by the dual vectors in each layer. Since we have disregarded the affine maps in this network, each layer is completely determined by $x_{k,0}$ and $\{a_{k,i}^*:i\in I\}$.

\begin{rem} By construction, it is clear that scaling the dual vectors by a positive constant or permuting them does not change the cone. Hence, there are many different affine maps generating the same polyhedral cone.
\end{rem}

We introduce this network architecture since it will simplify the proof of the approximation result later. In fact, for every network $\tilde{F}$ of the form in \eqref{eq:modified_network} there is a fully-connected ReLU network $F$ (of the form defined in \eqref{eq:original_network}) computing the same function.
\begin{lem}[Equivalence of Architectures]
\label{lem:Equivalence of Architectures}
Given a network $\tilde{F}:\IR^{d+1}\to\IR$ as in \eqref{eq:modified_network}, then there is a fully-connected ReLU network $F:\IR^{d+1}\to\IR$ with the same number of layers such that
\begin{align}
\tilde{F}(x)=F(x), \quad \text{for all } x\in \IR^{d+1}
\end{align} 
\end{lem}
\begin{proof}
To see this, suppose $\tilde{F}$ is a network of the form in \eqref{eq:modified_network}. We want to find a fully-connected ReLU network $F$ of the form \eqref{eq:original_network}
such that $F(x)=\tilde{F}(x)$ for all $x\in \IR^{d+1}$. We start by letting $S_k$ denote the polyhedral cone corresponding to layer $\pi_k$ in \eqref{eq:modified_network}. For each $k\in \{1,2,\hdots,N\}$ let $A_{b,k}(x)$ be one of the possible affine maps generating $S_k$. Each such affine map is invertible since the vectors generating the cone are linearly independent (the cone is non-degenerate). Then, as inverses and compositions of affine maps are themselves affine maps we can define the fully-connected ReLU layer $T_k$ as
\begin{align}
T_k(x)=\text{ReLU}\big(A_{b,k} \circ A_{b,k-1}^{-1}(x)\big), \quad k=1,2,\hdots,N
\end{align}
with $A_{b,0}(x)=x$ (the identity map). Similar to the conical decomposition described in Section~\ref{sec:structure} we can decompose $T_k$ as
\begin{align}
T_k = A_{b,k}\circ \pi_k \circ A_{b,k-1}^{-1} 
\end{align}
We can interpret this decomposition as follows:
\begin{itemize}
\item First, $A_{b,k-1}^{-1}$ maps $\IR^{d+1}_+$ back onto the previous cone, namely $S_{k-1}$.
\item Then, $\pi_k$ projects $\IR^{d+1}$ onto the cone $S_k$ associated to the affine map $A_{b,k}$.
\item At last, $A_{b,k}$ maps $S_k$ onto $\IR^{d+1}_+$.
\end{itemize}
By construction, we then get that the composition of all $N$ ReLU layers reduces to
\begin{align}
T_N\circ T_{N-1}\circ \hdots \circ T_1(x)=A_{b,N}\circ \pi_N \circ \pi_{N-1}\circ \hdots \circ \pi_1(x)
\end{align}
Hence, if let $L(x)=\tilde{L}\circ A_{b,N}^{-1}(x)$ we can conclude that $F(x)=\tilde{F}(x)$ for all $x\in \IR^{d+1}$.
\end{proof}
Thus, if we can show that a network of the modified architecture can approximate $\Gamma_{\phi}$ by its decision boundary then the same result holds for a fully-connected ReLU network as well.

\subsection{Restrictions to Bounded Sets} 
\label{sec:compact_domains}
We will now show how the modified layers can be reduced to projections on hyperplanes when restricted to bounded sets. First, consider a projection $\pi$ and the set $S_{(I\setminus \{j\},\{j\})}$ for some $j\in I$. This set is an element in the partition of $\IR^{d+1}$ defined by \eqref{eq:partition}. It is the set satisfying $S_{(I\setminus \{j\},\{j\})}\subset \mathbf{U}_i$ for $i\neq j$ and $S_{(I\setminus \{j\},\{j\})} \subset \mathbf{U}_j^c$. Thus, the action of $\pi$ on this set is precisely
\begin{align}
\pi\bigg(S_{(I\setminus\{j\},\{j\})}\bigg)=S_{(I\setminus\{j\},\emptyset)}\subset \mathbf{P}_j 
\end{align}
as $\pi$ projects every point in $S_{(I\setminus\{j\},\{j\})}$ along $a_j^*$ onto the facet $S_{(I\setminus\{j\},\emptyset)}$ of the cone and this facet is also contained in the hyperplane $\mathbf{P}_j$ defined in \eqref{eq:hyperplane}.
Recall that the projection direction $a_j^*$ is parallel to the line $L_j$ given by the intersection of the set of the remaining hyperplanes when $\mathbf{P}_j$ is removed.
\begin{lem}[Projections on Hyperplanes]
Let $K\subset \IR^{d+1}$ be a bounded set. Given a closed half-space $\mathbf{U}$ with the hyperplane $\mathbf{P}$ as its boundary and a vector $\xi \in \IR^{d+1}$ not parallel to $\mathbf{P}$, we can construct a projection $\pi:\IR^{d+1} \to S$ defined as in \eqref{eq:projection_pi} such that $\pi$ projects $  \mathbf{U}^c \cap K$ on $\mathbf{P}$ along $\xi$ while acting as the identity on $ \mathbf{U} \cap K$.
\label{lemma:ReLU_on_compact}
\end{lem}
The proof of Lemma~\ref{lemma:ReLU_on_compact} can be found in Appendix~\ref{Appendix A}. Figure~\ref{fig:projection-bounded}(a) illustrates such a hyperplane projection and the idea is that we can always orient and shape the cone such that the two maps coincide on a bounded set, as in Figure~\ref{fig:projection-bounded}(b).
\begin{figure}
\centering
\begin{subfigure}[b]{0.3\textwidth}
\centering
\includegraphics[width=\textwidth]{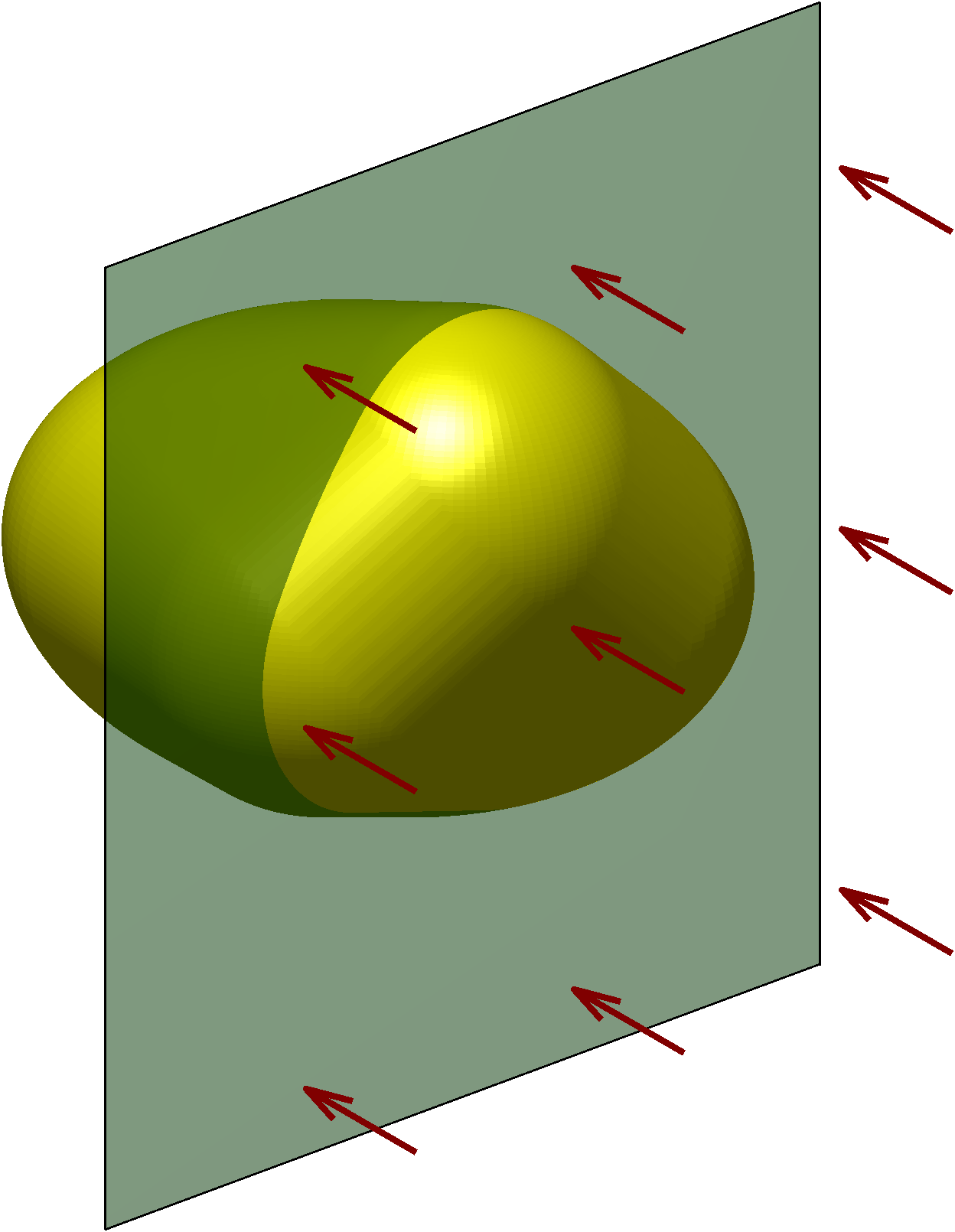}
\caption{}
\end{subfigure}
\hfill
\begin{subfigure}[b]{0.63\textwidth}
\centering
\includegraphics[width=\textwidth]{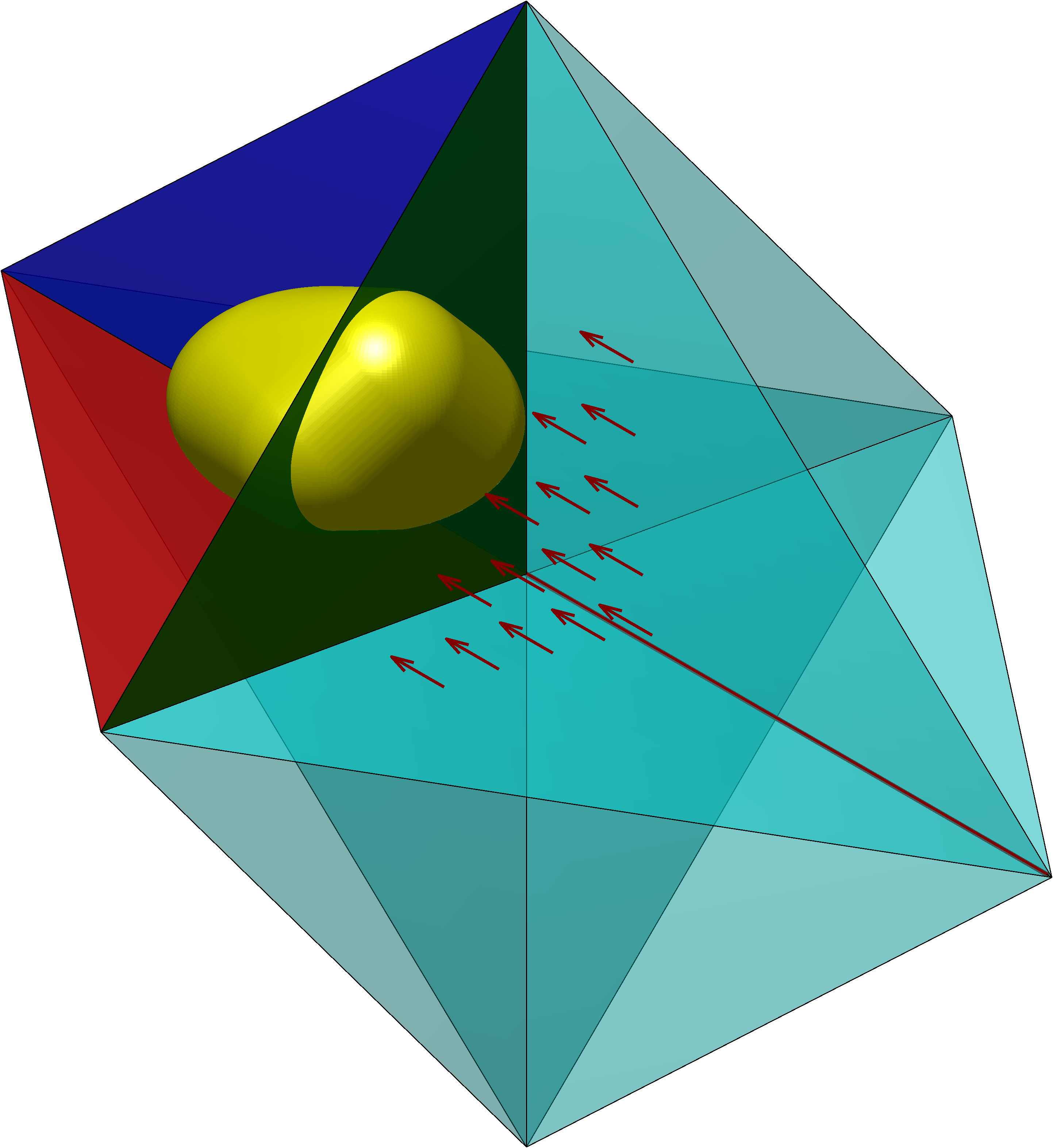}
\caption{ }
\end{subfigure}
\hfill
\caption{\textit{Projection on a Hyperplane.} \textbf{(a)} A map that projects one half-space onto a hyperplane along a given direction, whereas points in the complementary half-space are left unchanged. \textbf{(b)} A realization of the map depicted in \textbf{(a)} using a map $\pi$ in the conical decomposition of a ReLU layer. The depicted projection direction is parallel to the line of intersection shown in red.}
\label{fig:projection-bounded}
\end{figure}
As $B^d_R\times [-c,c]$ is bounded for any positive real number $c$ we can construct layers such that their actions, restricted to this set, can be described by a projection direction and a halfspace in accordance with Lemma~\ref{lemma:ReLU_on_compact}. Hence, we will assume $c$ is large enough such that all sets we will be considering is contained in $B^d_R\times [-c,c]$.

\section{Projecting the Graph}
\label{sec: projecting the graph}
\subsection{A Single Projection}
\label{sec: single projection}
The procedure will be to repeatedly project parts of $\Gamma_{\phi}$ by cutting off small pieces of the domain by the half-spaces associated to the layers. We will always choose these half-spaces orthogonal to $\IR^d$. Such an half-space $\mathbf{U}$ in $\IR^{d+1}$ can be decomposed as $\mathbf{U}=U\times \IR$ where $U$ is a half-space in $\IR^d$. Similarly, the boundary hyperplane $\mathbf{P}$ of $\mathbf{U}$ can then be written as $\mathbf{P}=P\times \IR$ where $P$ is the boundary hyperplane of $U$. We will continue to use bold symbols to denote the associated half-space/hyperplane in $\IR^{d+1}$ while the same symbols without the bold typesetting will refer to the corresponding half-space/hyperplane in $\IR^d$.

We will now consider the image of $\Gamma_{\phi}$ by a single projection $\pi$. In this setting, $\mathbf{U}^c \cap \Gamma_{\phi}$ will be projected while $\mathbf{U} \cap \Gamma_{\phi}$ will be left unchanged when applying $\pi$. After the projection, the set $\mathbf{U}^c \cap \Gamma_{\phi}$ will lie in a neighborhood in $\mathbf{P}$ around the intersection $\mathbf{P}\cap \Gamma_{\phi}$ as depicted in Figure~\ref{fig:Projectedfirst}. These neighborhoods are formally described using the $\varepsilon$-band $\Sigma_\phi^\varepsilon(\omega)$, see Definition~\ref{def:eps_band}, which pads $\Gamma_\phi$ over a set $\omega$ with all points within a tolerance $\varepsilon$ in the $(d+1)$-direction.
Figure~\ref{proj_in_plane} illustrates the resulting $\varepsilon$-band in the projection-plane after the mapping depicted in Figure~\ref{fig:Projectedfirst}.

\begin{figure}
  \centering
  \begin{subfigure}[b]{0.45\textwidth}
  \centering
  \includegraphics[width=\textwidth]{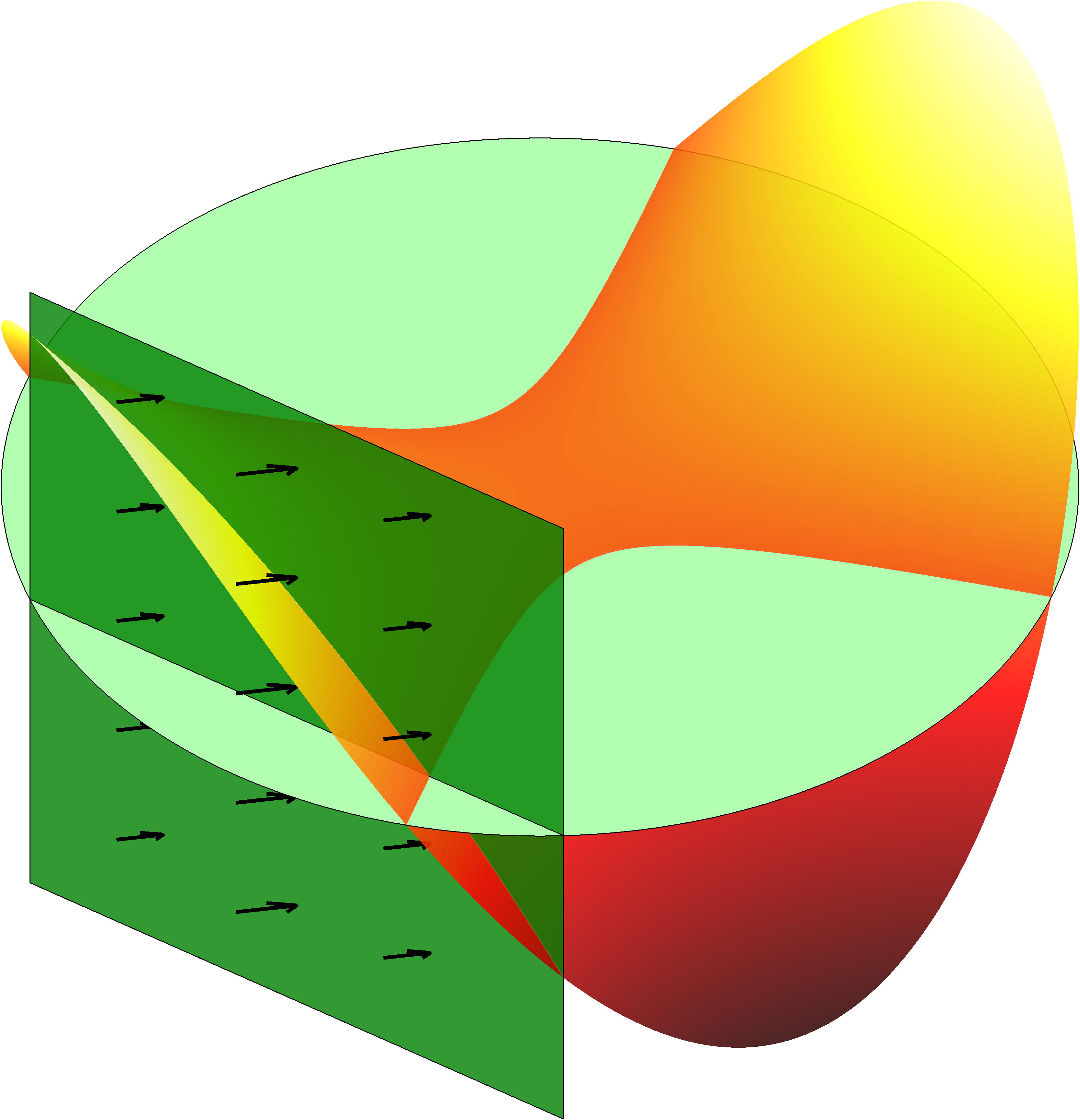}
  \caption{}
  \end{subfigure}
  \hfill
  \begin{subfigure}[b]{0.45\textwidth}
  \centering
  \includegraphics[width=\textwidth]{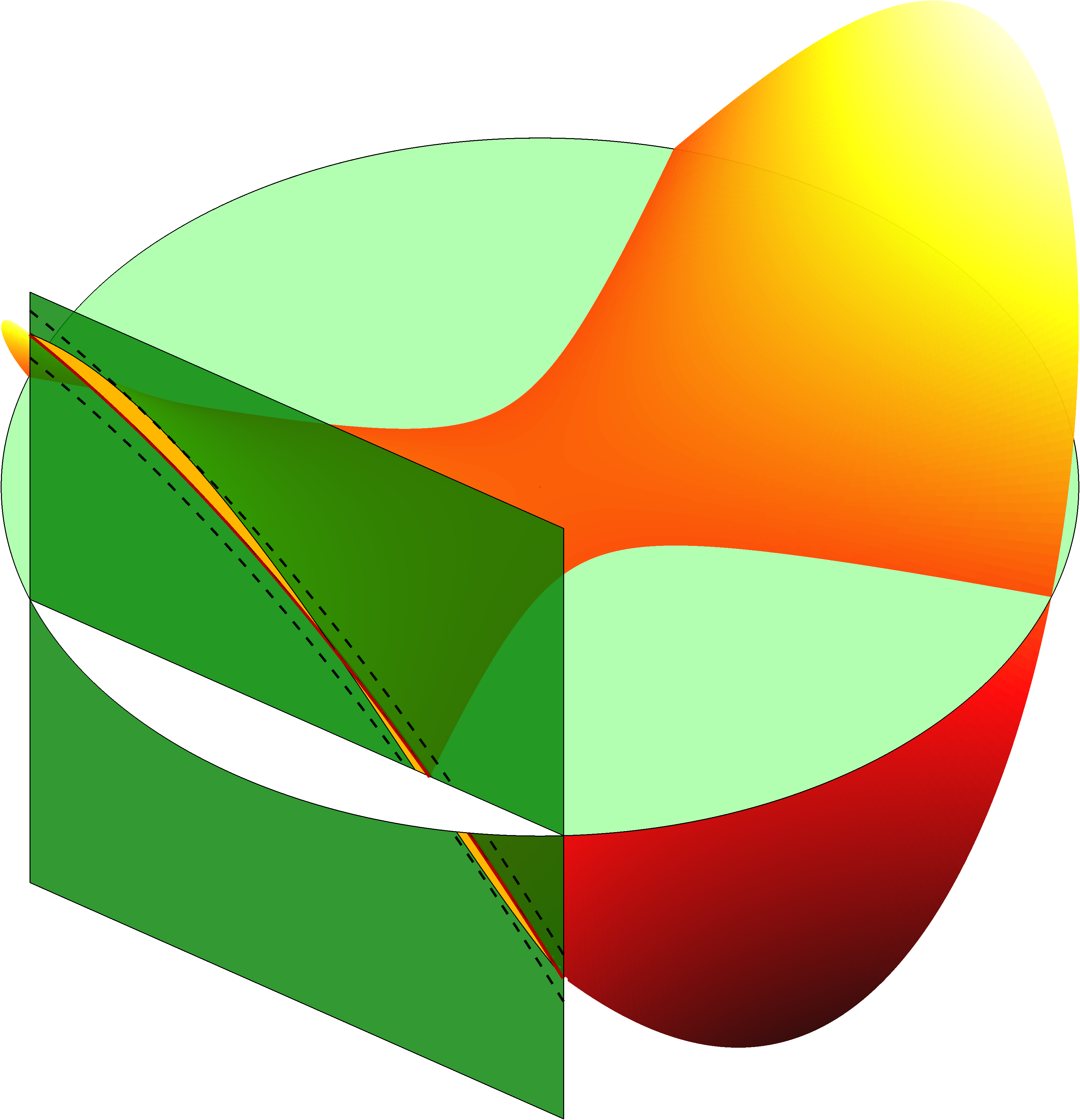}
  \caption{ }
  \end{subfigure}
  \hfill
  \caption{\textit{Projecting the Graph.} \textbf{(a)} A part of $\Gamma_{\phi}$ is projected on a hyperplane $\mathbf{P}$ along a direction $\xi$. \textbf{(b)} The projected part will lie in a neighborhood around the intersection of $\Gamma_{\phi}$ and $\mathbf{P}$.}
  \label{fig:Projectedfirst}
  \end{figure}

\begin{figure} 
\centering
\includegraphics[scale=0.45]{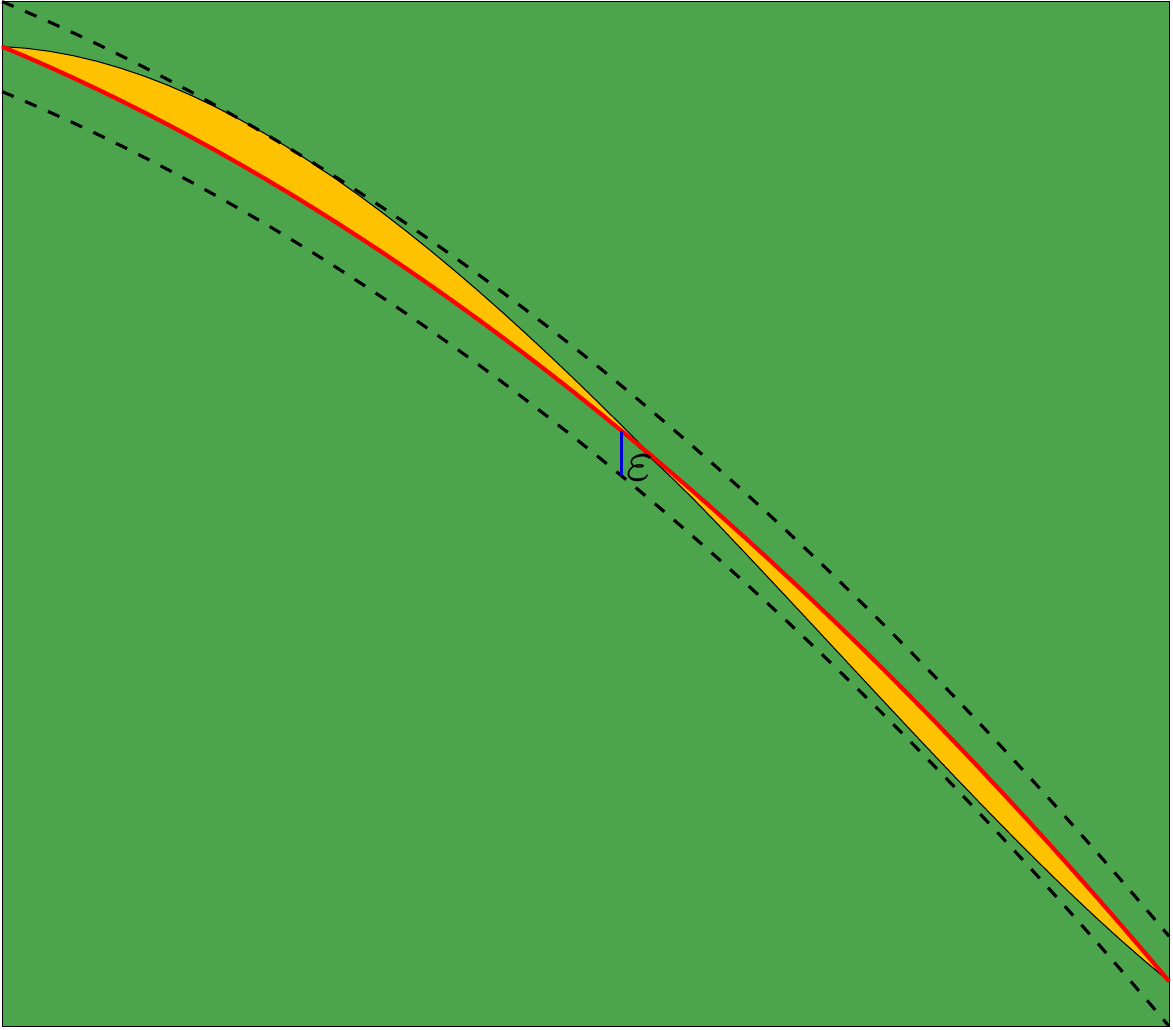}
\caption{\textit{$\varepsilon$-band.} An illustration of the $\varepsilon$-band seen in the hyperplane from Figure~\ref{fig:Projectedfirst}.}
\label{proj_in_plane}
\end{figure}
With this notation, we will have
\begin{align}
\begin{split}
&\pi(\mathbf{U}^c \cap \Gamma_{\phi})\subseteq \Sigma_{\phi}^{\varepsilon}(P)\\
&\pi(\mathbf{U} \cap \Gamma_{\phi}) = \mathbf{U} \cap \Gamma_{\phi}
\end{split}
\label{eq:im}
\end{align}
for some $\varepsilon \geq 0$. As we will repeatedly cut off the domain $B^d_R$ by half-spaces the image of $\Gamma_{\phi}$ after several projections will consist of one unaffected part as the graph of the restriction of $\phi$ to some subset $V\subseteq B^d_R$ and a projected part contained in an $\varepsilon$-band of $\phi$ over $\partial V$. The set $V$ is the set in $B^d_R$ contained in every half-space associated to the projections. Figure~\ref{fig:multi_projected} shows an example of the image of $\Gamma_{\phi}$ under a composition of a number of projections. 
\begin{figure} 
\centering
\includegraphics[scale=0.35]{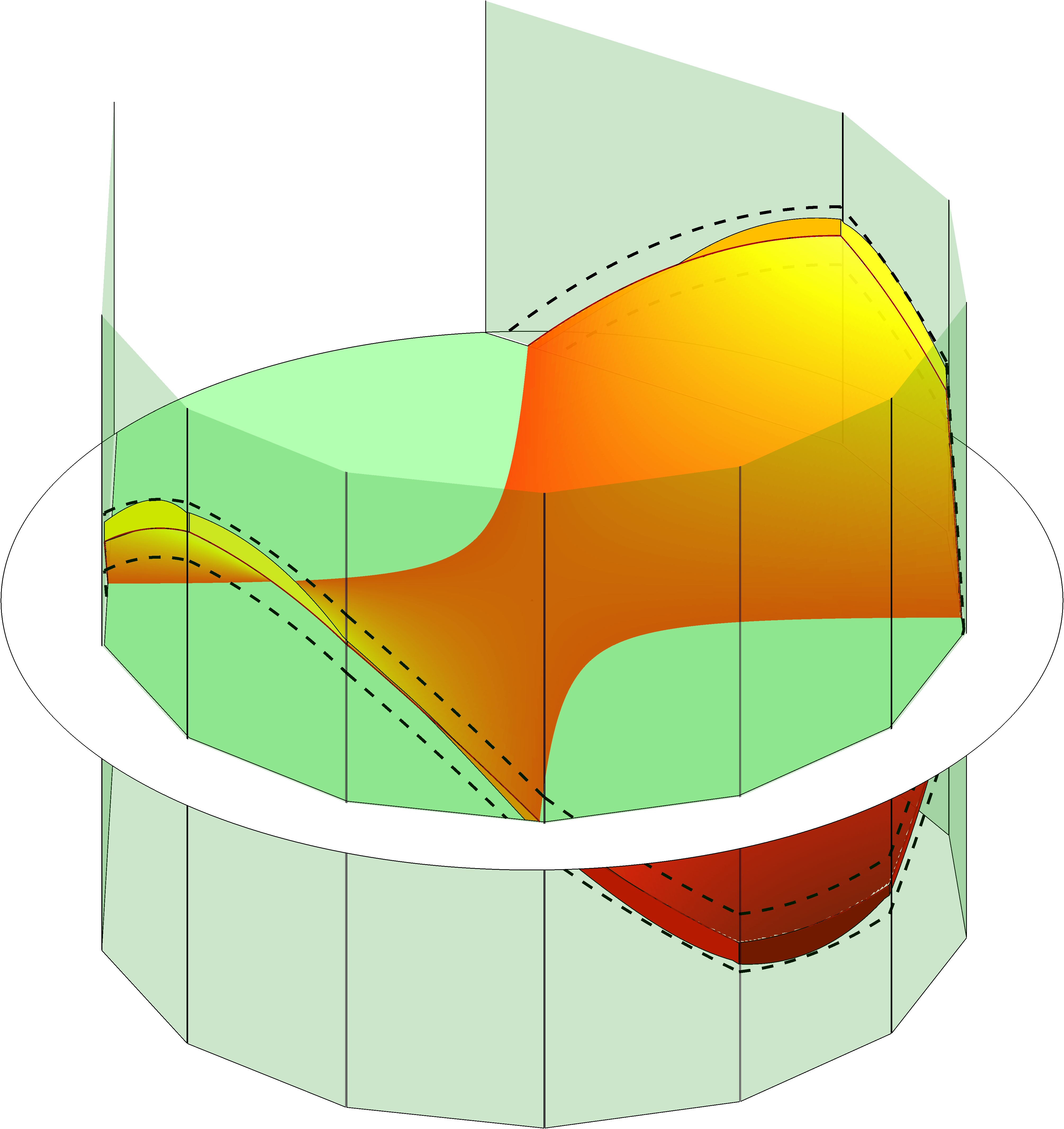}
\caption{\textit{Composition of Projections.} Illustration of the image of the graph after applying several projections. The part of $\Gamma_{\phi}$ inside every half-space will be unaffected. Meanwhile, the projected part has been mapped into an $\varepsilon$-band of $\phi$.}
\label{fig:multi_projected}
\end{figure}

To suppress the notation, given a set $V$ we let $\Gamma_{\phi}(V)$ denote the part of the graph of $\phi$ defined by
\begin{align}
\Gamma_{\phi}(V)=\{(x,\phi(x))\in \IR^{d+1}:x\in B^d_R \cap V\}
\end{align}
Consider the graph $\Gamma_{\phi}(\bar{B}^d_r)$ for some $0<r< R$. We would like to estimate the size of the resulting $\varepsilon$-band after one projection of $\Gamma_{\phi}(\bar{B}^d_r)$. Therefore, we will now take a closer look at how a point on $\Gamma_{\phi}(\bar{B}^d_r)$ is mapped for a special choice of projection direction $\xi$.

\paragraph{Mapping a Point.} Let $U,P \subset \IR^d $ be the half-space and boundary hyperplane associated to a layer $\pi$ respectively and let $\beta\in \IR^{d}$ be the inward pointing unit normal to $U$. Assume that $P$ is tangent to the ball $\bar{B}^d_{r-\delta}$ of radius $r-\delta$, for some $0<\delta<r$, concentric with $\bar{B}^d_r$. Then, $U^c \cap \bar{B}^d_r$ is a spherical cap of height $\delta$ and the base $P\cap \bar{B}^d_r$ is a $d-1$ dimensional ball of radius $\sqrt{2r\delta-\delta^2}$ with a center point denoted by $p$. The geometrical situations for $d=2$ and $d=3$ are depicted in Figure \ref{fig:P1_geom}.
\begin{figure} 
\centering
\begin{subfigure}[b]{0.38\textwidth}
\centering
\includegraphics[width=\textwidth]{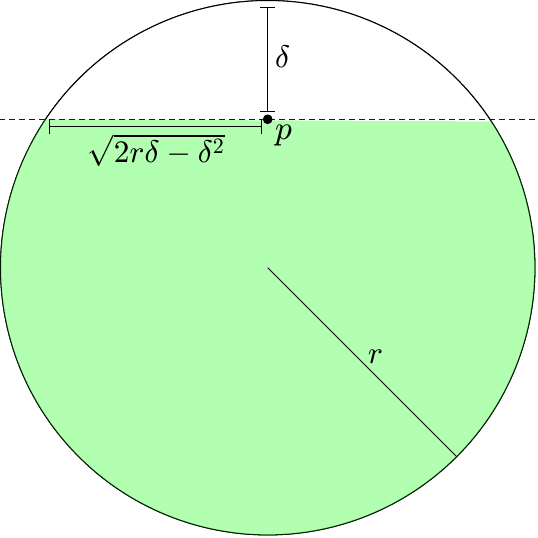}
\caption{ }
\end{subfigure}
\hfill
\begin{subfigure}[b]{0.61\textwidth}
\centering
\includegraphics[width=\textwidth]{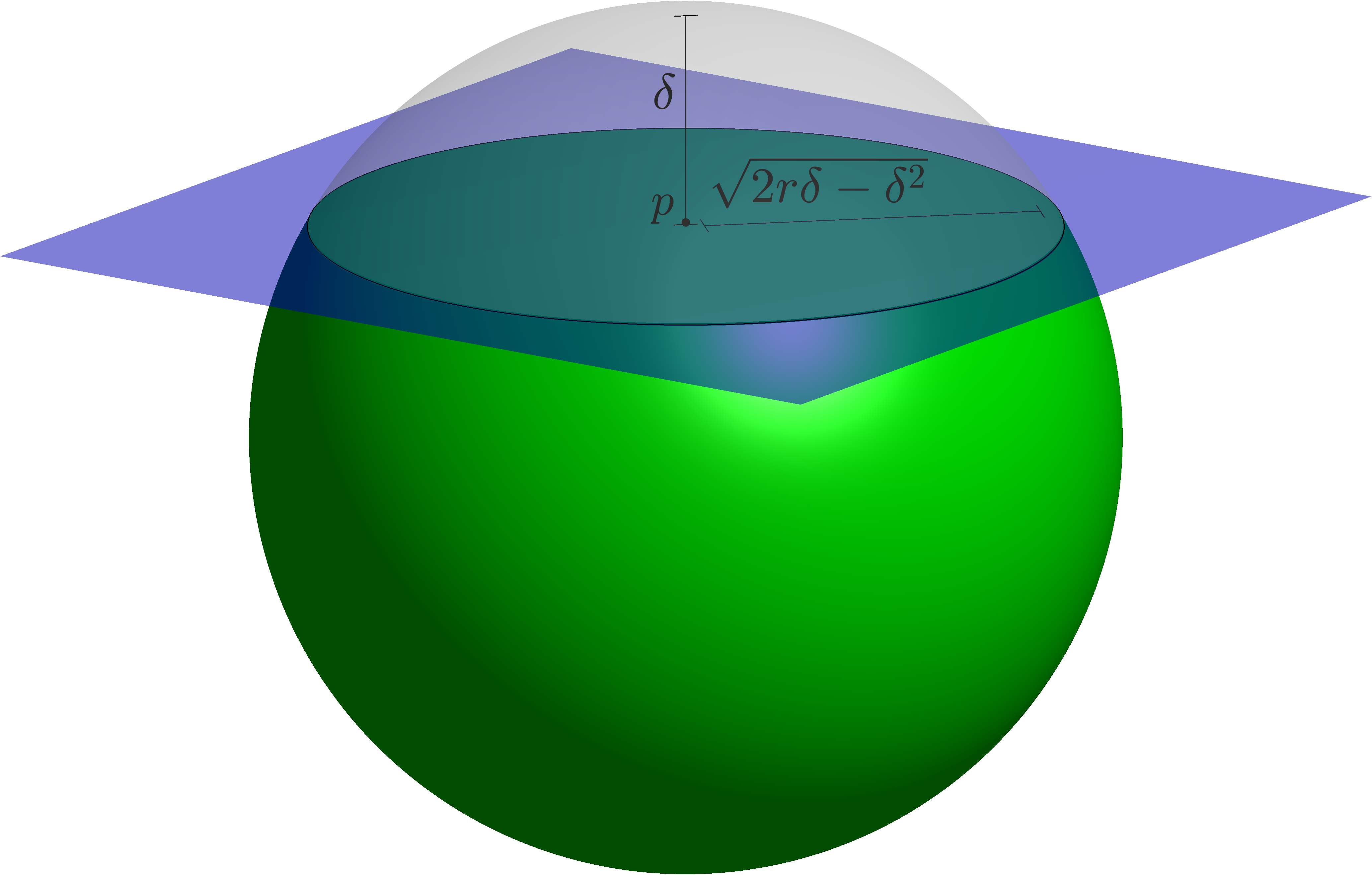}
\caption{ }
\end{subfigure}
\hfill
\caption{\textit{Spherical Caps.} The geometry of the spherical caps cut out as the intersection of the half-space $U^c$ and the ball $\bar{B}^d_r$ \textbf{(a)} for $d=2$ and \textbf{(b)} for $d=3$.}
\label{fig:P1_geom}
\end{figure}

Now, assume $x \in U^c \cap \bar{B}^d_r$ and let $y\in \IR$. Then, the point $(x,y)$ will be projected on $\mathbf{P}$ by $\pi$ along a direction $\xi$. Suppose now that we choose the projection direction as 
\begin{align}
\boxed{
\xi=(\beta,\nabla_{\beta}\phi(p))
}
\end{align}
where $\nabla_{\beta}\phi(p)=\beta \cdot \nabla\phi(p)$ is the directional derivative of $\phi$ at the center point $p$ with respect to the unit normal $\beta$. Then, if $t=\text{dist}(x,P)$ is the distance between $x$ and $P$ we can express the action of $\pi$ on this point as
\begin{align}
(x,y) \stackrel{\pi} \longmapsto (x+t\beta,y+t\nabla_{\beta}\phi(p))\in \mathbf{P}
\label{eq:map-point}
\end{align}
Since $\delta$ is the height of the spherical cap, we know that $t\leq \delta$. Figure \ref{geom} shows the geometrical setup in the 2-dimensional plane spanned by $\beta$ and the vector pointing to $x$.
\begin{figure} 
\centering
\includegraphics[scale=0.75]{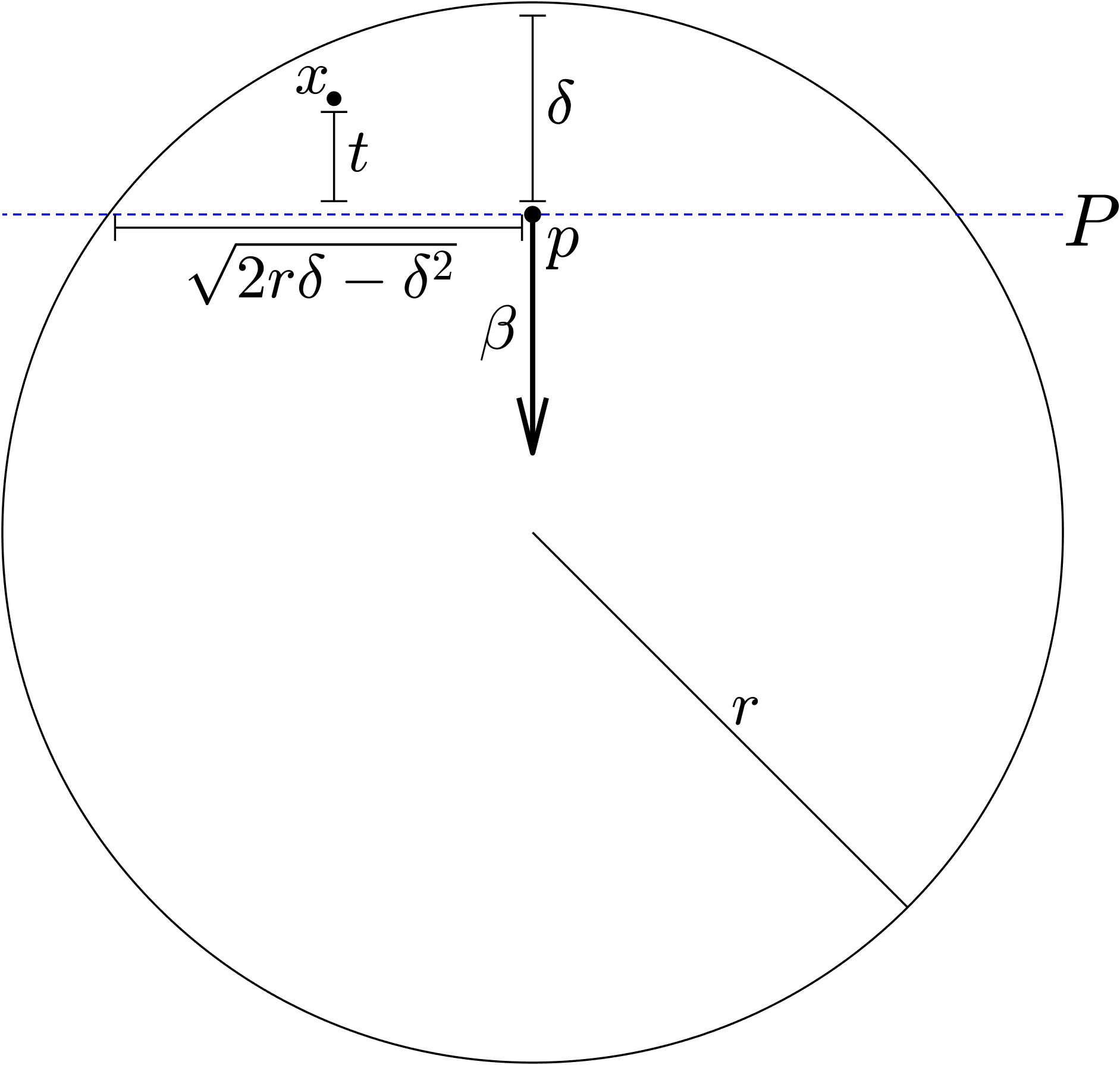}
\caption{\textit{Geometrical Setup.} A 2-dimensional slice of the sphere $\bar{B}^d_r$ illustrating how a point $x$, at a distance $t$ from the hyperplane $P$, is mapped along the unit normal $\beta$ to $P$.}
\label{geom}
\end{figure}
In particular, the point $(x,\phi(x))\in \Gamma_{\phi}$ will mapped according to
\begin{align}
(x,\phi(x)) \stackrel{\pi} \longmapsto (x+t\beta,\phi(x)+t\nabla_{\beta}\phi(p))\in \mathbf{P}
\label{eq:map-graph}
\end{align}
\paragraph{The Size of the \boldmath{$\mathbf{\varepsilon}$}-band.} Note that the first $d$ coordinates of the image in \eqref{eq:map-graph} is $x+t\beta\in P$ and the value of $\phi$ at this point is precisely $\phi(x+t\beta)$. Therefore we introduce the quantity
\begin{align}
|\phi(x)+t\nabla_{\beta}\phi(p)-\phi(x+t\beta)|
\label{eq:quantity}
\end{align}
measuring the offset between the projection of $(x,\phi(x))$ and the corresponding point $(x+t\beta,\phi(x+t\beta))$ on the graph $\Gamma_{\phi}$. Note that $t$ depends on $x$. If we would set
\begin{align}
\varepsilon=\sup_{x\in U^c\cap \bar{B}^d_r} |\phi(x)+t\nabla_{\beta}\phi(p)-\phi(x+t\beta)|
\label{eq:epsilon}
\end{align}
then $\pi(\mathbf{U}^c \cap \Gamma(\bar{B}^d_r))\subseteq \Sigma_{\phi}^\varepsilon(P\cap \bar{B}^d_r)$ as $\varepsilon$ is an upper bound on the offset between the graph over $P$ and the projected points. Therefore, we will now continue by estimating the quantity in \eqref{eq:quantity}. We have the following proposition.
\begin{prop}
Suppose $x\in U^c \cap \bar{B}_r^d$ and $t=\text{dist}(x,P)$. If $d>1$, then
\begin{align}
\label{eq:one_projection}
\boxed{
|\phi(x)+t\nabla_{\beta}\phi(p)-\phi(x+t\beta)|\leq C_1(d-1)\sqrt{r}\delta^{1/2}t
}
\end{align}
where $C_1$ is a constant independent of $\delta$, $d$ and $r$.
\label{proposition:one_projection}
\end{prop}
The proof of Proposition~\ref{proposition:one_projection} can be found in Appendix~\ref{Appendix B}. The idea is to use Taylor's theorem and the fact that the second order derivatives of $\phi$ are bounded by assumption. The following corollary follows directly from Proposition~\ref{proposition:one_projection}.
\begin{cor}
Let $\delta$ be the height of the spherical cap $U^c \cap \bar{B}_r^d$ whose base has center $p$ and let $\beta$ be the inward pointing unit normal to $U$. If $d>1$ and we choose the projection direction $\xi=(\beta,\nabla_{\beta}\phi(p))$ then
\begin{align}
\pi\big(\mathbf{U}^c\cap \Gamma_{\phi}(\bar{B}^d_r)\big)\subseteq \Sigma_{\phi}^\varepsilon(P)
\end{align}
where $\varepsilon=C_1(d-1)\sqrt{r}\delta^{3/2}$ for a constant $C_1$.
\label{cor:one_projection}
\end{cor}
\begin{proof}
From Proposition~\ref{proposition:one_projection} we know that the offset between the projection of a point $(x,\phi(x))$ where $x\in U^c \cap B^d_r$ and the corresponding point in the intersection $\mathbf{P} \cap \Gamma_{\phi}$ satisfies the inequality
\begin{align}
|\phi(x)+t\nabla_{\beta}\phi(p)-\phi(x+t\beta)|\leq C_1(d-1)\sqrt{r}\delta^{1/2}t
\end{align}
where $t$ is the distance between $x$ and $P$. Taking the supremum on both sides of the inequality over all $x$ in $U^c\cap \bar{B}^d_r$ we get an upper bound on the maximal offset of the projected points and the graph in the hyperplane. By noting that $t=\text{dist}(x,P)$ is the only factor dependent on $x$ on the right hand side of the inequality and $t\leq \delta$ for all $x\in U^c\cap \bar{B}^d_r$ we get 
\begin{align}
\sup_{x\in U^c\cap \bar{B}^d_r}| \phi(x)+t\nabla_{\beta}\phi(p)-\phi(x+t\beta)|\leq  C_1(d-1)\sqrt{r}\delta^{3/2}
\end{align}
and the corollary follows.
\end{proof}
Both Proposition~\ref{proposition:one_projection} and Corollary~\ref{cor:one_projection} hold as long as $d>1$. If $d=1$ we get a more favorable scaling of $\varepsilon$, see Remark~\ref{rem:d=1} in Appendix~\ref{Appendix B}. We will assume that $d>1$ henceforth.
\begin{rem}
Both Proposition~\ref{proposition:one_projection} and Corollary~\ref{cor:one_projection} concern the graph of $\phi$ restricted to closed balls of radius $r<R$. However, the same estimates hold (which is apparent in the proofs) when replacing $\bar{B}^d_r$ with $\BR$, i.e., the domain of $\phi$. This fact will also be needed later.
\end{rem}
\subsection{Projection on a Polytope}
\label{Sec:Projection on a Polytope}
In the previous section we saw the effect on $\Gamma_{\phi}(\bar{B}^d_r)$ after a single projection. Now, we will turn to the more general case when we apply a group of layers. To that end, let $\Pi:B^d_R\times \IR \rightarrow \IR^{d+1}$ be defined by the composition
\begin{align}
\boxed{
\Pi(x)=\pi_m\circ \hdots \circ \pi_1(x)
}
\label{eq:Pi}
\end{align}
for some $m$. As before, we denote by $U_i$, $P_i$ the associated half-space and hyperplane to layer $\pi_i$ respectively and we let $\beta_i$ be the inward pointing unit normal to $U_i$.

\paragraph{Construction.} We will define the layers such that all the corresponding hyperplanes are tangents to $\bar{B}^d_{r-\delta}$ for some $0<\delta <r$ and each half-space contains the origin. Thus, we can express them explicitly as
\begin{align}
\begin{split}
U_i&=\{x\in \IR^d:\beta_i\cdot x + (r-\delta)\geq 0\}\\
P_i&=\{x\in \IR^d:\beta_i\cdot x + (r-\delta)=0\}
\end{split}
\label{eq:hyperplanes_halfspaces}
\end{align}
Further, let $p_i$ be the center point of the base of the spherical cap $U_i^c \cap \bar{B}^d_r$ so that the projection direction $\xi_i$ of layer $\pi_i$ is given by $\xi_i=(\beta_i,\nabla_{\beta_i}\phi(p_i))$. Note that $p_i$ is precisely the point where $P_i$ touches $\bar{B}^d_{r-\delta}$. The intersection of these half-spaces will be a convex polytope $\mathcal{P}$ and we will assume it satisfies
\begin{align}
\boxed{
\bar{B}^d_{r-\delta}\subset \mathcal{P}\subset \bar{B}^d_r
}
\label{eq:inclusion}
\end{align}
The first inclusion always holds as each half-space contains $\bar{B}^d_{r-\delta}$ by construction, but for the second inclusion to hold we assume we have enough half-spaces and that they are distributed properly. Later, we will give a more detailed construction assuring the validity of the second inclusion as well, see Figure~\ref{fig:polytope}.
\begin{figure} 
\centering
\includegraphics[scale=0.75]{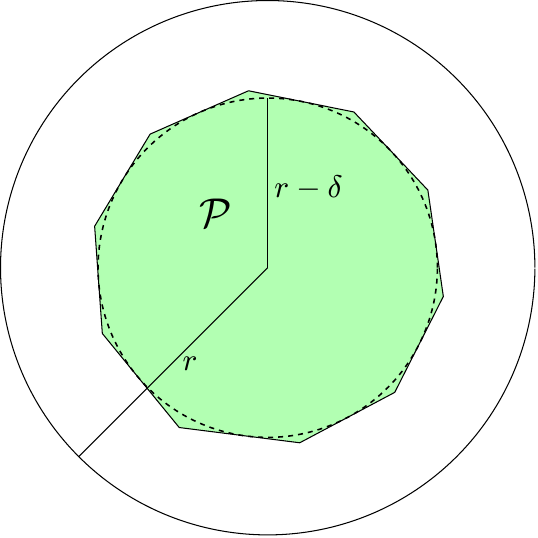}
\caption{\textit{Polytope.} The polytope $\mathcal{P}$ contains the ball $\bar{B}^d_{r-\delta}$ by construction and is assumed to be contained in $\bar{B}^d_r$.}
\label{fig:polytope}
\end{figure}

\paragraph{Mapping a Point.} In this setting a point $(x_0,\phi(x_0))\in \Gamma_{\phi}(\bar{B}^d_r)$, with $x_0\in \bar{B}^d_r \setminus \mathcal{P}$, may be projected several times. Therefore, we introduce the notation
\begin{align}
\begin{split}
(x_i,y_i)&=\pi_i(x_{i-1},y_{i-1}), \qquad \text{for} \quad i=1,2,\hdots,m\\
y_0&=\phi(x_0)
\end{split}
\label{eq:map_sequence}
\end{align}
so that $(x_m,y_m)=\Pi\big(x_0,\phi(x_0)\big)$. Then, the sequence $(x_i,y_i)_{i=0}^m$ describes the trajectory of the point $(x_0,\phi(x_0))$ in $\IR^{d+1}$ obtained by applying the layers in $\Pi$ successively. Note that this trajectory is in general dependent on the order of the layers in \eqref{eq:Pi}. If we now let 
\begin{align}
t_i=
\left\{
\begin{aligned}
&\mathrm{dist}(x_{i-1},P_i), \quad &&\text{if } x_{i-1}\in U^c_i
\\
&0, \quad &&\text{if } x_{i-1}\in U_i
\end{aligned}
\right.
\label{eq:def_ti}
\end{align}
then by \eqref{eq:map-point} we have
\begin{align} 
\boxed{
x_i=x_{i-1}+t_i\beta_i
}
\label{eq:map-point_2}
\end{align}
Thus, if $x_{i-1}\in U_i$ then $t_i=0$ so $x_i=x_{i-1}$. Otherwise, $t_i> 0$ and then $x_i=x_{i-1}+t_i\beta_i \in P_i$. The sequence $(x_i)_{i=0}^m$ defines a piecewise linear path in $\IR^d$ and its total length is given by
\begin{align}
\label{eq:S_n}
S_m=\sum_{i=1}^m t_i
\end{align}
An example of such a path is illustrated in Figure~\ref{fig:point-sequence}.
\begin{figure} 
\centering
\includegraphics[scale=0.75]{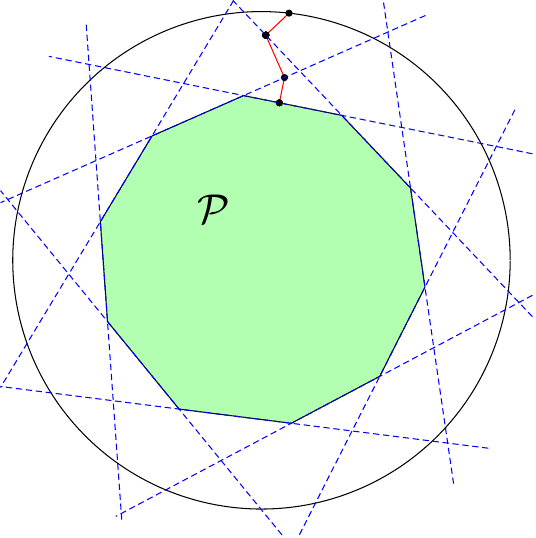}
\caption{\textit{Trajectory of a Point.} When we apply several layers, a single point may be projected several times. The figure illustrates a trajectory obtained by the dynamics governed by equation \eqref{eq:map-point_2} for a point starting at the boundary of $\bar{B}^d_r$. The path generally depends on the order in which the layers are applied.}
\label{fig:point-sequence}
\end{figure}

By \eqref{eq:inclusion} we can write  $\Gamma_{\phi}(\bar{B}^d_r)=\Gamma_{\phi}(\bar{B}^d_r \setminus \mathcal{P})\cup \Gamma_{\phi}(\mathcal{P})$ and as $\Pi$ is the identity on $\Gamma_{\phi}(\mathcal{P})$, by the definition of $\mathcal{P}$, we have $\Pi\big( \Gamma_{\phi}(\mathcal{P})\big)=\Gamma_{\phi}(\mathcal{P})$. Now, we would like to show that the projected part of the graph lies within an $\varepsilon$-band of $\phi$ over $\partial \mathcal{P}$ (recall Definition~\ref{def:eps_band}), i.e.,
\begin{align}
\Pi\big( \Gamma_{\phi}(\bar{B}^d_r\setminus \mathcal{P})\big)\subseteq \Sigma_{\phi}^\varepsilon(\partial \mathcal{P})
\label{eq:inclusion_polytope}
\end{align}
and also estimate the size of $\varepsilon$. We do this in two steps. 
\begin{enumerate}
\item We derive a condition on $\delta$ guaranteeing that $x_m\in \partial \mathcal{P}$. This will assure that a point $(x_0,\phi(x_0))$ with $x_0\in \bar{B}^d_r\setminus \mathcal{P}$ is mapped to a point in the set $\Sigma_{\phi}^\varepsilon(\partial \mathcal{P}))$ for some $\varepsilon\geq 0$.
\item We continue by deriving a uniform bound on $\varepsilon$, independent of the starting point $x_0$, such that the inclusion in \eqref{eq:inclusion_polytope} holds. We will do this by estimating $|y_m-\phi(x_m)|$ as this quantity measures the offset between the graph over $\partial \mathcal{P}$ and the image of the point $(x_0,\phi(x_0))$ under $\Pi$.
\end{enumerate}

We will start by deriving a constraint on $\delta$ in terms of $r$ guaranteeing that $x_m\in \partial \mathcal{P}$. In doing so, we will need the following lemma.
\begin{lem}
\label{lem:PjPi_intersect}
If two hyperplanes $P_j$ and $P_i$ in \eqref{eq:hyperplanes_halfspaces} intersect in $\bar{B}^d_r$ then
\begin{align}
\beta_j\cdot \beta_i\geq \frac{r^2-4r\delta+2\delta^2}{r}
\label{eq:betajbetai_2}
\end{align}
\end{lem}
The proof can be found in Appendix~\ref{Appendix B}. We are now ready to prove the following lemma giving a constraint on $\delta$.
\begin{lem}[\boldmath{$\mathbf{\delta}$}-Condition]
\label{lemma:condition_delta}
If $\delta$ satisfies the inequality
\begin{align}
\boxed{
0<\delta\leq r\left(1-\frac{1}{\sqrt{2}}\right)
}
\label{eq:condition_delta}
\end{align}
then $x_m\in \partial \mathcal{P}$.
\end{lem}
\begin{proof}
Consider the sequence $(x_i)_{i=0}^m$ defined by \eqref{eq:map-point_2} where $x_0\in \bar{B}^d_r\setminus \mathcal{P}$. Since $x_0\notin \mathcal{P}$ it is clear that $x_i$ cannot be contained in the interior $\mathcal{P}^\circ$ for any $1\leq i\leq m$ because every time some point in the sequence is projected, it will be projected onto a supporting hyperplane to $\mathcal{P}$. If we can show that
\begin{align}
x_i\in \bigcap_{j=1}^i U_j
\label{eq:inclusion_subpoly}
\end{align}
for all $1\leq i \leq m$, then in particular $x_m\in \mathcal{P}$ and hence $x_m\in \partial \mathcal{P}$. 

We will now show, by induction on $i$, that if $\delta$ fulfills \eqref{eq:condition_delta} then \eqref{eq:inclusion_subpoly} holds. Clearly, $x_1\in U_1$ by \eqref{eq:map-point_2}. Suppose now that $x_{i-1}\in \bigcap_{j=1}^{i-1} U_j$. Again, by \eqref{eq:map-point_2} we have that $x_{i}\in  U_{i}$, but we also need $x_{i}\in U_j$ for every $1\leq j \leq i-1$. 
\begin{itemize}
\item If $t_{i}=0$ then $x_{i}=x_{i-1}$, and in this case we can directly conclude that \eqref{eq:inclusion_subpoly} holds by the hypothesis.
\item If instead $t_{i}>0$, then $x_{i}\in P_{i}\cap \bar{B}^d_r$. Now, let  $1\leq j\leq i-1$ and consider the case when $(P_i\cap \bar{B}^d_r) \subset U_j$. Then, we can directly conclude that $x_{i}\in U_j$. However, we will have to investigate the case when $(P_{i}\cap \bar{B}^d_r) \not \subset U_j$. 

This situation arises when the two hyperplanes $P_{j}$ and $P_i$ intersect inside $\bar{B}^d_r$. Thus, we need to show that $x_i\in U_j$, i.e., $\beta_j\cdot x_{i}+(r-\delta)\geq 0$ whenever $P_j$ and $P_i$ intersect inside $\bar{B}^d_r$. Using \eqref{eq:map-point_2} we can rewrite this condition as
\begin{align}
\beta_j\cdot x_{i-1}+(r-\delta) + t_{i}\beta_j\cdot \beta_{i}\ \geq 0
\end{align}
but since $x_{i-1}\in U_j$ by assumption we have that $\beta_j\cdot x_{i-1}+(r-\delta)\geq 0$. As $t_i>0$ in the case we are considering, a sufficient condition is to ensure that 
\begin{align}
\beta_j\cdot \beta_i\geq 0
\label{eq:betajbetai} 
\end{align}
By Lemma~\ref{lem:PjPi_intersect}, we know that $\beta_j\cdot \beta_i\geq \frac{r^2-4r\delta+2\delta^2}{r}$. Now, if $\delta$ satisfies \eqref{eq:condition_delta} then it follows that $r^2-4\delta r + 2\delta^2\geq 0$. Hence, this condition on $\delta$ guarantees the validity of \eqref{eq:betajbetai} which in turn assures the relation in \eqref{eq:inclusion_subpoly}. 
\end{itemize}
Then it follows by induction that $x_m\in\partial \mathcal{P}$ as desired.
\end{proof}
Under the restriction provided in Lemma~\ref{lemma:condition_delta} all points $(x_0,\phi(x_0))\in \Gamma_{\phi}(\bar{B}^d_r\setminus \mathcal{P})$ will be mapped to $\Sigma_{\phi}^\varepsilon(\partial \mathcal{P}))$ as long as $\varepsilon$ is sufficiently large.

\paragraph{The Size of the \boldmath{$\mathbf{\varepsilon}$}-band.}
We proceed by deriving a bound on the quantity $|y_m-\phi(x_m)|$ describing how much the last value $y_m$ deviates from the function value at the endpoint $x_m\in \partial \mathcal{P}$. This will help us estimate the size of the resulting $\varepsilon$-band. Even though we cannot directly apply Proposition~\ref{proposition:one_projection}, as it is only valid for a single projection, it will be useful in proving the following lemma.
\begin{lem} The following inequality holds
\label{lemma:mulit_proj}
\begin{align} 
|y_m-\phi(x_m)|\leq C_1(d-1)\sqrt{r}\delta^{1/2}S_m
\label{eq:bound_path}
\end{align}
where $S_m$ is the length of the path defined by the sequence $(x_i)_{i=0}^m$.
\end{lem}
Lemma~\ref{lemma:mulit_proj} shows that the deviation between the graph and the image of a point under the projections depends on the total path length. In the proof, found in Appendix~\ref{Appendix B}, Proposition~\ref{proposition:one_projection} is repeatedly applied. We will now continue by estimating the path length. First, we have the following result.
\begin{lem} Let $t_i$ be defined as in \eqref{eq:def_ti}, then
\begin{align}
\label{eq:bound_ti}
t_i\leq \frac{\|x_{i-1}\|^2-\|x_i\|^2}{2(r-\delta)}
\end{align}
\label{lemma:bound_ti}
\end{lem}
The proof is based on a geometrical argument and can be found in Appendix~\ref{Appendix B}. Using Lemma \ref{lemma:bound_ti}, we can easily bound the sum $S_m$ in \eqref{eq:S_n} and thus give an upper bound on $|y_m-\phi(x_m)|$.
\begin{lem} If $0<\delta\leq r\left(1-\frac{1}{\sqrt{2}}\right)$ then
\label{lemma:ym}
\begin{align}
|y_m-\phi(x_m)|\leq C_3(d-1)\sqrt{r}\delta^{3/2}
\label{eq:ym}
\end{align}
\label{lemam:pathlength}
for a constant $C_3$ independent of $r$, $\delta$ and $d$.
\end{lem}
\begin{proof}
From Lemma \ref{lemma:bound_ti} we get
\begin{align}
S_m=\sum_{i=1}^m t_i\leq \frac{1}{2(r-\delta)}\sum_{i=1}^m \bigg(\|x_{i-1}\|^2-\|x_i\|^2\bigg)=\frac{\|x_{0}\|^2-\|x_m\|^2}{2(r-\delta)}
\label{eq:bound_Sm}
\end{align}
but we know that $x_0\in \bar{B}^d_r$ so $\|x_0\|\leq r$. Moreover, $x_m\in \partial \mathcal{P}$ and since $\bar{B}_{r-\delta}\subset \mathcal{P}$, we have that $\|x_m\|\geq r-\delta$. Hence, we arrive at
\begin{align}
S_m=\sum_{i=1}^m t_i\leq \frac{r^2-(r-\delta)^2}{2(r-\delta)}
\label{eq:S_m-bound}
\end{align}
For values of $\delta$ close to $r$ the right hand side of the inequality above gets very large. On the other hand, we know that $\delta$ must be chosen such that it satisfies the constraint in \eqref{eq:condition_delta} to ensure $x_m\in \partial\mathcal{P}$. It is easily seen that for values of $\delta$ in this range, we have
\begin{align}
\frac{r^2-(r-\delta)^2}{2(r-\delta)}\leq \frac{1+\sqrt{2}}{2} \delta
\label{eq:ineq_C_2}
\end{align}
Together with \eqref{eq:S_m-bound} we get bound
\begin{align}
S_m=\sum_{i=1}^mt_i\leq C_2 \delta
\label{eq:bound_pathlength}
\end{align}
where $C_2=\frac{1+\sqrt{2}}{2}$. Then, by Lemma~\ref{lemma:mulit_proj} we arrive at
\begin{align}
|y_m-\phi(x_m)|\leq C_3(d-1)\sqrt{r}\delta^{3/2}
\end{align}
with $C_3=C_2C_1$.
\end{proof}
We summarize the findings of this section in the following proposition characterizing the image of $\Gamma_{\phi}(\bar{B}^d_r)$ under the map $\Pi$, also illustrated in Figure~\ref{fig:Projection-first-poly}.

\begin{prop}[Image of the Graph] Let $\Pi:B^d_R\times \IR\rightarrow \IR^{d+1}$ be defined by $\Pi=\pi_m \circ \hdots \circ \pi_1$ where the corresponding hyperplanes to the layers are all tangent to $\bar{B}^d_{r-\delta}$ for some $\delta$ satisfying $0<\delta\leq r\left(1-\frac{1}{\sqrt{2}}\right)$. Suppose the convex polytope $\mathcal{P}$, obtained as the intersection of the half-spaces associated to the layers, satisfies $\bar{B}^d_{r-\delta}\subset \mathcal{P}\subset \bar{B}^d_r$. Then, our choice of the projection directions for each layer in $\Pi$ yields
\label{proposition:multi_projections}
\begin{align}
\begin{split}
\boxed{
\begin{array}{c}
\Pi\big( \Gamma_{\phi}(\bar{B}^d_r\setminus \mathcal{P})\big)\subseteq \Sigma_{\phi}^\varepsilon(\partial \mathcal{P}) \\
\Pi\big(\Gamma_{\phi}(\mathcal{P}) \big)=\Gamma_{\phi}(\mathcal{P})
\end{array}
}
\end{split}
\end{align}
where $\varepsilon= C_3(d-1)\sqrt{r}\delta^{3/2}$ for a constant $C_3$ independent of $r$, $\delta$ and $d$. 
\end{prop}
\begin{proof}
By the definition of $\mathcal{P}$ it is an immediate consequence that
\begin{align}
\Pi\big(\Gamma_{\phi}(\mathcal{P}) \big)=\Gamma_{\phi}(\mathcal{P})
\end{align}
Further, it follows directly from Lemma~\ref{lemma:condition_delta} that $\Pi\big( \Gamma_{\phi}(\bar{B}^d_r\setminus \mathcal{P})\big)\subseteq \partial \mathcal{P}\times \IR$ and as the bound on $|y_m-\phi(x_m)|$ given in Lemma~\ref{lemma:ym} holds independently of the starting point $x_0\in \bar{B}^d_r\setminus \mathcal{P}$ we get
\begin{align}
\Pi\big( \Gamma_{\phi}(\bar{B}^d_r\setminus \mathcal{P})\big)\subseteq \Sigma_{\phi}^\varepsilon(\partial \mathcal{P}) 
\end{align}
for $\varepsilon=C_3(d-1)\sqrt{r}\delta^{3/2}$.
\end{proof}
\begin{figure} 
\centering
\includegraphics[scale=0.5]{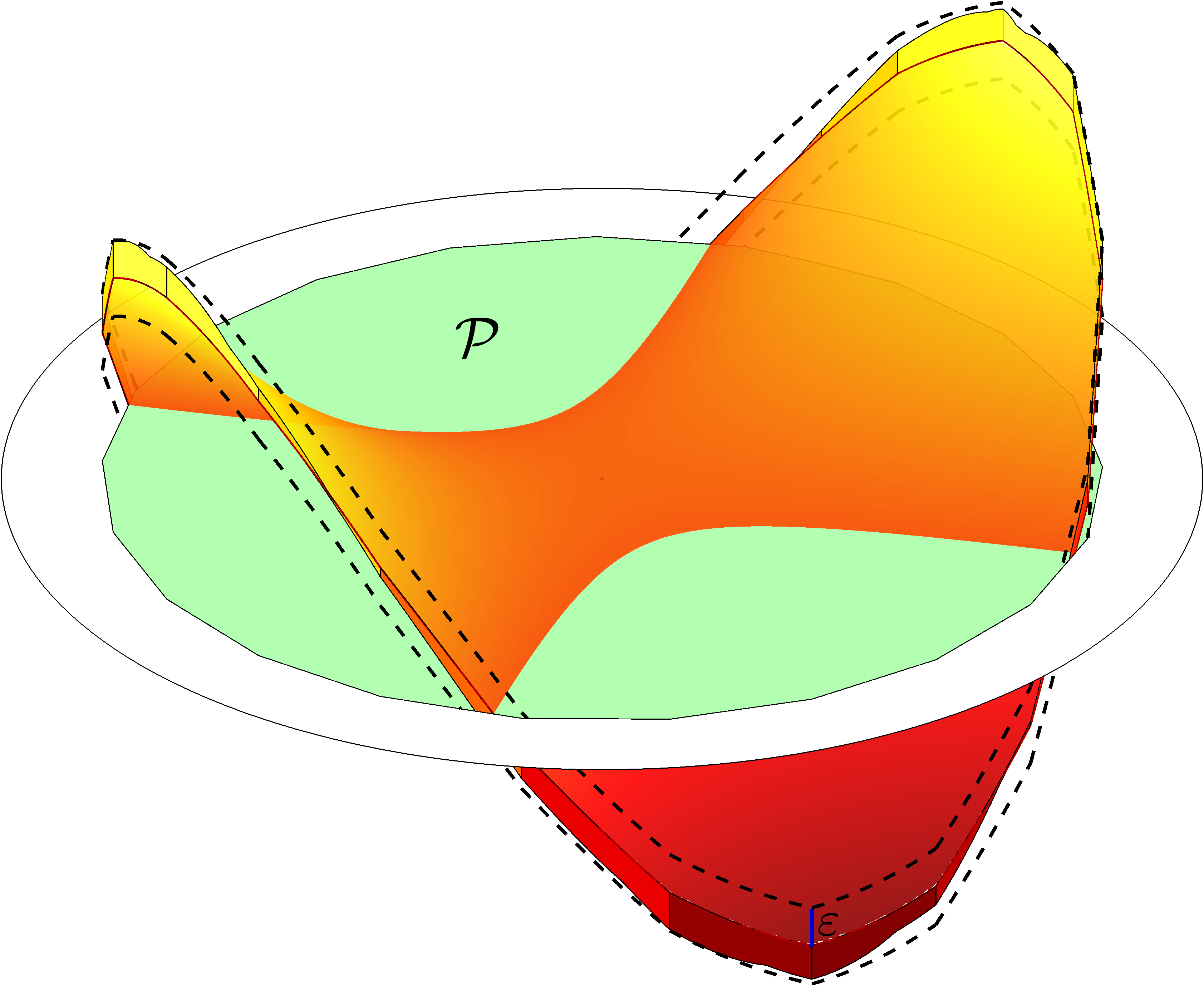}
\caption{\textit{Image of the Graph.} After applying the map $\Pi$, the image of $\Gamma_{\phi}(\bar{B}^d_r)$ will consist of an unaffected part over the polytope $\mathcal{P}$ (defined by the layers in $\Pi$) and a projected part contained in an $\varepsilon$-band of $\phi$ over the boundary $\partial \mathcal{P}$. }
\label{fig:Projection-first-poly}
\end{figure}
\subsection{Projections on a Sequence of Polytopes}
\label{sec:projection on sequence}
In the preceding section we considered the case when the graph of $\phi$ restricted to the ball $\bar{B}^d_r$ of a general radius $r$ is projected by a chain of layers where the corresponding half-spaces defined a polytope tangent to the concentric ball $\bar{B}^d_{r-\delta}$. We showed that the projected part will be contained in the $\varepsilon$-band of $\phi$ over the boundary of the polytope if $\varepsilon = C_3(d-1)\sqrt{r}\delta^{3/2}$ as long as $0<\delta\leq r\left(1-\frac{1}{\sqrt{2}}\right)$. We will now continue by describing a procedure to project the graph repeatedly in a similar way.
\paragraph{Sequence of Polytopes.} Note that we can express a network of the form in \eqref{eq:modified_network} as
\begin{align}
\boxed{
\tilde{F}=\tilde{L} \circ \Pi_M\circ \hdots \circ \Pi_2\circ \Pi_1
}
\label{eq:Pi_composition}
\end{align}
where each $\Pi_{k}:\IR^{d+1}\rightarrow \IR^{d+1}$ is a composition of $m_k$ layers
\begin{align}
\Pi_k=\pi^{k}_{m_k}\circ \hdots \circ \pi^{k}_2 \circ \pi^{k}_1 , \quad 1\leq k\leq M
\label{Pi_k-composition}
\end{align}
The total number of layers $N$ in the network $\tilde{F}$ is then 
\begin{align}
N=\sum_{k=1}^Mm_k
\label{eq:tot layers}
\end{align}
We denote the $d$-dimensional half-space and hyperplane associated to layer $\pi^{k}_i$ by $U^k_i$ and $P^k_i$ respectively, where $1\leq k \leq M$ and $1\leq i \leq m_k$. For every $\Pi_k$, we can now define the convex polytope
\begin{align}
\boxed{
\mathcal{P}_k=\bigcap_{i=1}^{m_k} U^k_i
}
\label{eq:def poly}
\end{align}
assumed to be bounded and to each such polytope we let
\begin{align}
r_k=\max(\{\|x\|:x\in \mathcal{P}_k\})
\label{eq:def_r_k}
\end{align}
Thus, $r_k$ is the radius of the smallest closed ball containing $\mathcal{P}_k$. To simplify the notation we will let $\mathcal{P}_0=\bar{B}^d_R$ (recall that $\BR$ is the domain of $\phi$) and $r_0=R$. The goal is to construct each $\Pi_k$ such that its image under the composition of all $M$ maps satisfies
\begin{align}
\Pi_M\circ \hdots \circ\Pi_2 \circ\Pi_1\big( \Gamma_{\phi}\big)\subseteq \Sigma_{\phi}^{\varepsilon}(\partial \mathcal{P}_M) \cup\Gamma_{\phi}(\mathcal{P}_M)
\label{eq:image_under_network}
\end{align}
for some $\varepsilon\geq 0$ we want to estimate. Moreover, we would like
\begin{align}
\lim_{M\to \infty}r_M=0
\label{eq:lim-r_M}
\end{align}
so that $\mathcal{P}_M$ can be made arbitrarily small by choosing $M$, and thus the number of layers $N$, large enough. If $\mathcal{P}_M$ is sufficiently small, we can approximate the remaining part of the graph, namely $\Gamma_{\phi}(\mathcal{P}_M)$, by the hyperplane $\ker(\tilde{L})$ where $\tilde{L}$ is the last affine map in \eqref{eq:Pi_composition}. Then, as the decision boundary $\Gamma$ is the preimage of $\ker(\tilde{L})$ we will see that it will approximate $\Gamma_{\phi}$ on $B^d_R$ to an accuracy related to the value of $\varepsilon$ in \eqref{eq:image_under_network}. 

\paragraph{Construction.} We will now describe how we construct each polytope by specifying each half-space in their definition \eqref{eq:def poly}. The main idea is similar to the construction of the polytope $\mathcal{P}$ in Section~\ref{Sec:Projection on a Polytope}. We will define them recursively. Given $\mathcal{P}_{k}\subset \bar{B}^d_{r_{k}}$ we construct $\mathcal{P}_{k+1}$ by letting the half-spaces associated to the layers in $\Pi_{k+1}$ be tangent to a smaller ball of radius $r_{k}-\delta_{k}$ for some $\delta_{k}>0$. We will use the notion of an $\epsilon$-net.
\begin{definition}[\boldmath{$\mathbf{\epsilon}$}-net]
\label{def:eps-net}
Given a bounded set $V\subseteq \IR^{d}$ we say that a finite subset $N_{\epsilon}\subseteq V$ is an $\epsilon$-net of $V$ if for all points $p\in V$ there is a point $q\in N_{\epsilon}$ such that $\|p-q\|\leq \epsilon$. 
\end{definition} 
\begin{rem}
The variable $\epsilon$ in Definition~\ref{def:eps-net} should not be confused with $\varepsilon$ used in Definition~\ref{def:eps_band} to indicate the size of the $\varepsilon$-bands.
\end{rem}
Given $\mathcal{P}_k$ and $r_k$, let $N_{\epsilon_{k}}$ be an $\epsilon_{k}$-net of the sphere $S^{d-1}_{r_{k}}$ enclosing $\mathcal{P}_{k}$. Define $m_{k+1}=|N_{\epsilon_{k}}|$ (we will give an upper bound on this quantity later) and label the points as $N_{\epsilon_{k}}=\{q_1, q_2, \hdots, q_{m_{k+1}}\}$. To each point $q_i \in N_{\epsilon_{k}}$ we define the half-space $U^{k+1}_i$ in the following way. Let $p^{k+1}_i$ be the point on the sphere $S^{d-1}_{r_{k}-\delta_{k}}$ intersecting the ray from the origin passing through $q_i$. Define $U^{k+1}_i$ such that it is tangent to $S^{d-1}_{r_{k}-\delta_{k}}$ precisely at $p^{k+1}_i$ and oriented such that it contains the origin.

Thus, $(U^{k+1}_i)^c\cap \bar{B}^d_{r_{k}}$ is a spherical cap of height $\delta_{k}$ where $p^{k+1}_i$ is the center of its base. Figure~\ref{fig:p_l} illustrates how the half-spaces are defined given the $\epsilon_k$-net.
\begin{figure} 
\centering
\begin{subfigure}[b]{0.45\textwidth}
\centering
\includegraphics[width=\textwidth]{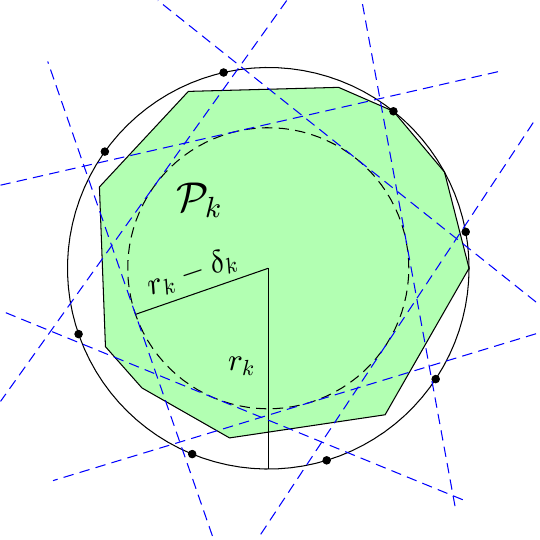}
\caption{ }
\end{subfigure}
\hfill
\begin{subfigure}[b]{0.45\textwidth}
\centering
\includegraphics[width=\textwidth]{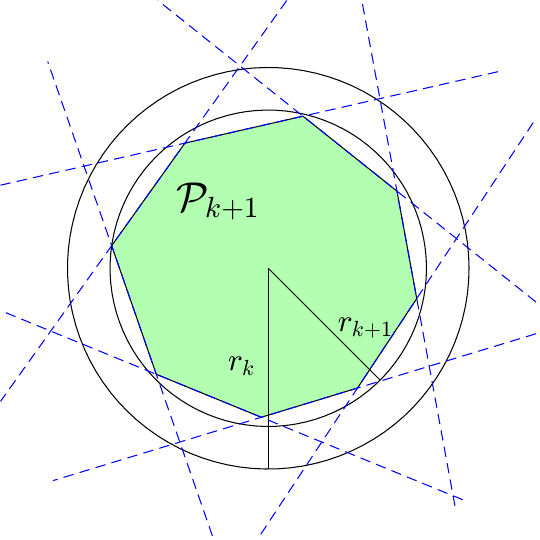}
\caption{ }
\end{subfigure}
\hfill
\caption{\textit{Construction of the Polytopes.} \textbf{(a)} First, we create an $\epsilon_k$-net of the sphere $S^{d-1}_{r_k}$ (the black points). We associate a half-space to each point in this set such that the corresponding hyperplanes are all tangent to the ball $\bar{B}^d_{r_{k}-\delta_k}$. \textbf{(b)} The polytope $\mathcal{P}_k$ is obtained as the intersection of these half-spaces.}
\label{fig:p_l}
\end{figure}

In this way, $\mathcal{P}_{k+1}$ is constructed by taking the intersection of $m_{k+1}$ tangential half-spaces to $\bar{B}^d_{r_{k}-\delta_k}$ and
\begin{align}
\bar{B}^d_{r_{k}-\delta_k}\subset \mathcal{P}_{k+1},\quad \text{for } 0\leq k \leq M-1
\label{eq:B_r-delta-inclusion}
\end{align}
The reason we introduce an index on the parameter $\delta_k$ is that we have to guarantee that the condition \eqref{eq:condition_delta} holds for all $r_k$ to get well-behaved projections. Given a predefined discretization parameter $0<\delta\leq R\left(1-\frac{1}{\sqrt{2}}\right)$ we then define $\delta_k$ according to
\begin{align}
\boxed{
\delta_k= 
\begin{cases}
\delta \quad &\text{if }\delta \leq r_k\left(1-\frac{1}{\sqrt{2}}\right)\\
r_k\left(1-\frac{1}{\sqrt{2}}\right) \quad &\text{if }\delta > r_k\left(1-\frac{1}{\sqrt{2}}\right)
\end{cases}
}
\label{eq:definition_delta_k}
\end{align}
Thus, when the radius is sufficiently small we will need to reduce $\delta_k$. 

\paragraph{Properties of the Construction.} Now, in order to prove \eqref{eq:image_under_network} and \eqref{eq:lim-r_M}, we will start by proving that $\mathcal{P}_{k+1}\subset \mathcal{P}_{k}$ and give a lower bound on $r_{k}-r_{k+1}$, i.e., how much the radii decrease at each step. Then we will proceed by proving the inclusion by utilizing Proposition~\ref{proposition:multi_projections} repeatedly. 

To guarantee that we get a nested sequence of polytopes we will need to choose $\epsilon_k$ small enough. For our purpose, choosing
\begin{align}
\boxed{
\epsilon_{k} = \sqrt{\frac{1}{2}\delta_{k} r_{k}}
}
\end{align}
will be sufficient. First, we have the following lemma.
\begin{lem}
The construction of the polytopes guarantees that the inclusions
\begin{align}
 \bar{B}^d_{r_{k}-\delta_k} \subset \mathcal{P}_{k+1} \subset \bar{B}^d_{r_{k}}
\label{eq:P_l_prime}
\end{align}
hold for all $0\leq k \leq M-1$.
\label{lemma:P_l_prime}
\end{lem}
Thus, each polytope lies inside the ball enclosing the previous polytope. The proof is presented in Appendix~\ref{Appendix B}. We can now prove the following lemma giving a lower bound on the reduction of the radii and guaranteeing that the polytopes are nested.
\begin{lem} The polytopes and the radii satisfy
\begin{align}
\begin{split}
\boxed{
\begin{array}{c}
\mathcal{P}_{k+1}\subset \mathcal{P}_{k}\\
r_{k+1}\leq r_{k}-\frac{3\delta_{k}}{4}
\end{array}
}
\end{split}
\end{align}
for all $0\leq k \leq M-1$.
\label{proposition:inclusion}
\end{lem}
\begin{proof}
We first show $r_{k+1}\leq r_{k}-\frac{3\delta_{k}}{4}$. Recall the definition of $r_{k+1}$ in \eqref{eq:def_r_k}. From Lemma~\ref{lemma:P_l_prime} we have $r_{k+1} \leq r_{k}$. Let $p^*\in \mathcal{P}_{k+1}$ be a point fulfilling $\|p^*\|=r_{k+1}$. Clearly, $p^*$ will be a vertex on the convex polytope $\mathcal{P}_{k+1}$ and will thus be contained in the intersection of a subset of the boundary hyperplanes associated to $\mathcal{P}_{k+1}$.

If $q \in S^{d-1}_{r_{k}}$ is the closest point on the sphere to the point $p^*$, then $\|p^*-q\|=r_{k}-r_{k+1}$. Next, let $q_j\in N_{\epsilon_k}$ be the closest point to $q$, then it follows that $p^*\in P^k_j$. Define $s=\|q-q_j\|$, then $s\leq \epsilon_k$ since $N_{\epsilon_k}$ is an $\epsilon_k$-net of $S^d_{r_{k}}$ and let $q'$ be the orthogonal projection of $q$ onto $P^k_j$. Figure~\ref{fig:p_l_prime} shows the geometrical setup.
\begin{figure} 
\centering
\begin{subfigure}[b]{0.45\textwidth}
\centering
\includegraphics[width=\textwidth]{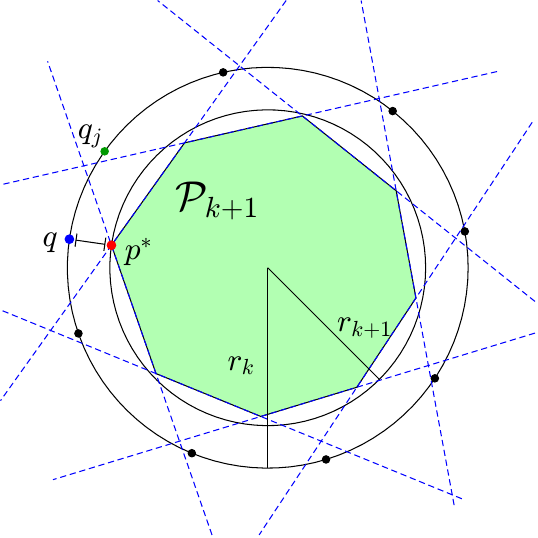}
\caption{ }
\end{subfigure}
\hfill
\begin{subfigure}[b]{0.45\textwidth}
\centering
\includegraphics[width=\textwidth]{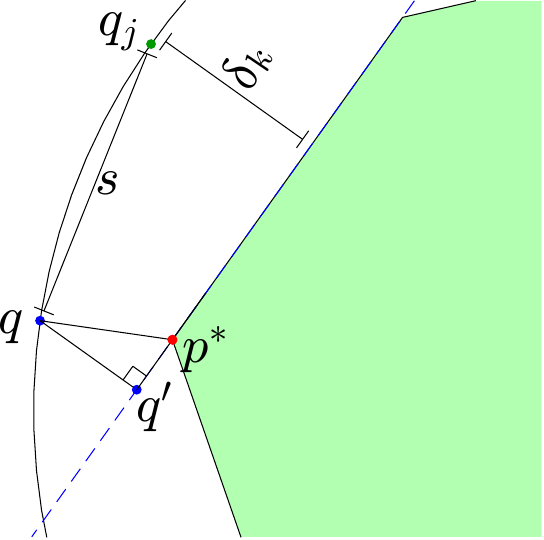}
\caption{ }
\end{subfigure}
\hfill
\caption{\textit{Estimating the Radius.} \textbf{(a)} The point $p^*$ is a vertex in $\mathcal{P}_{k+1}$ satisfying $r_{k+1}=\|p^*\|$. If we let $q$ be the closest point to $p^*$ on $S^{d-1}_{r_k}$, then the closest point $q_j$ in the $\epsilon_k$-net will be within a distance $\epsilon_k$ from $q$. \textbf{(b)} We can give a lower bound on $r_k-r_{k+1}$ by computing the length $\|q-q'\|$ where $q'$ is the orthogonal projection of $q$ onto the hyperplane $P^k_j$.}
\label{fig:p_l_prime}
\end{figure}

Consider the right triangle with vertices at $q$, $q'$ and $p^*$. As the line between $p^*$ and $q$ is the hypotenuse in this triangle we have that $\|p^*-q\|\geq \|q-q'\|$. Moreover, since $\|q-q_j\|=s$ we have that $\|q-q'\| = \delta_k-\frac{s^2}{2r_{k}} \geq \delta_k - \frac{\epsilon_k^2}{2r_{k}}$. This gives us the bound $r_{k}-r_{k+1}=\|p^*-q\|\geq  \delta_k - \frac{\epsilon_k^2}{2r_{k}}$. By inserting $\epsilon_k=\sqrt{\frac{1}{2}\delta_k r_{k}}$ we obtain
\begin{align}
r_{k+1} \leq r_{k}- \frac{3\delta_k}{4}
\label{eq:r_prime_less}
\end{align}
as desired.

We continue by proving the inclusion $\mathcal{P}_{k+1}\subset \mathcal{P}_k$. In case $k=0$ the inclusion $\mathcal{P}_1\subset \mathcal{P}_0$ follows from Lemma~\ref{lemma:P_l_prime} since we defined $\mathcal{P}_0=\bar{B}^d_{r_0}$ with $r_0=R$. To show the inclusion $\mathcal{P}_{k+1}\subset \mathcal{P}_k$ for $1\leq k\leq M-1$ we note that by construction $\bar{B}^d_{r_{k-1}-\delta_{k-1}}\subset \mathcal{P}_k$. Hence, if we can show that $\mathcal{P}_{k+1}\subset \bar{B}^d_{r_{k-1}-\delta_{k-1}}$ the inclusion follows. Thus, it is enough to show that
\begin{align}
r_{k+1}\leq r_{k-1}-\delta_{k-1}
\label{eq:inclusion-condition-r_k+1}
\end{align} 
By \eqref{eq:r_prime_less} we have $r_{k+1}\leq r_k-\frac{3\delta_k}{4}$ and by applying \eqref{eq:r_prime_less} again on $r_k$ we find that
\begin{align}
r_{k+1}\leq  r_{k-1}-\frac{3\delta_{k-1}}{4}-\frac{3\delta_k}{4} 
\label{eq:ineq-r_k+1-r_k-1}
\end{align}
We now consider two cases.
If $\delta_k=\delta_{k-1}$, then \eqref{eq:inclusion-condition-r_k+1} follows directly by \eqref{eq:ineq-r_k+1-r_k-1}. Otherwise, we have that $\delta_{k}\neq \delta_{k-1}$. In this situation, we have by \eqref{eq:definition_delta_k} that
\begin{align}
\delta_{k-1}\leq r_{k-1}\left(1-\frac{1}{\sqrt{2}}\right)
\label{eq:delta_k-1-upperbound}
\end{align}
and 
\begin{align}
\delta_k=r_k\left(1-\frac{1}{\sqrt{2}}\right)
\label{eq:delta_k_upperbound}
\end{align}
From \eqref{eq:B_r-delta-inclusion} we know that $r_k\geq r_{k-1}-\delta_{k-1}$. Using this in \eqref{eq:delta_k_upperbound} we obtain
\begin{align}
\begin{split}
\delta_k&\geq (r_{k-1}-\delta_{k-1})\left(1-\frac{1}{\sqrt{2}}\right)=r_{k-1}\left(1-\frac{1}{\sqrt{2}}\right)-\delta_{k-1}\left(1-\frac{1}{\sqrt{2}}\right)\\
&\geq\delta_{k-1} -\delta_{k-1}\left(1-\frac{1}{\sqrt{2}}\right)=\frac{\delta_{k-1}}{\sqrt{2}}
\end{split}
\end{align}
where we used \eqref{eq:delta_k-1-upperbound} in the second inequality. This lower bound together with \eqref{eq:ineq-r_k+1-r_k-1} implies \eqref{eq:inclusion-condition-r_k+1}. In both cases we can conclude that $\mathcal{P}_{k+1}\subset \mathcal{P}_k$.
\end{proof}
By equation \eqref{eq:B_r-delta-inclusion} and the definition of $\delta_k$, it is clear that each $\mathcal{P}_k$ will have a non-empty interior, so in particular $r_k>0$. Moreover, for $r_k\leq \frac{\delta}{1-\frac{1}{\sqrt{2}}}$ we have by \eqref{eq:definition_delta_k} and Lemma~\ref{proposition:inclusion} that
\begin{align}
r_{k+1}\leq \bigg(\frac{3+\sqrt{3}}{4\sqrt{2}}\bigg)r_k, \quad \text{for all $k$}
\end{align}
As $\bigg(\frac{3+\sqrt{3}}{4\sqrt{2}}\bigg)<1$ the limit in \eqref{eq:lim-r_M} follows.

\paragraph{Image of the Graph.}  Similarly, as in the preceding sections, we will choose the projection direction $\xi^{k}_i$ for layer $\pi^{k}_j$ as 
\begin{align}
\boxed{
\xi^k_i=(\beta^{k}_i,\nabla_{\beta^{k}_i}\phi (p^k_i))
}
\end{align}
where $\beta^{k}_i$ is the inward pointing unit normal to $U^k_i$ and $p^i_k$ is the base of the spherical cap $(U^{k}_i)^c\cap \bar{B}^d_{r_{k-1}}$. We are now ready to prove the following lemma.
\begin{lem}
\label{lemma:projection_multiple_polytopes}
Suppose that $0<\delta_k\leq r_k\left(1-\frac{1}{\sqrt{2}}\right)$ for all $0\leq k \leq M-1$, then our construction yields 
\begin{align}
\Pi_M\circ \hdots \circ\Pi_2 \circ\Pi_1\big( \Gamma_{\phi}\big)\subseteq \Sigma_{\phi}^{\varepsilon_M}(\partial \mathcal{P}_M) \cup\Gamma_{\phi}(\mathcal{P}_M)
\label{eq:image_under_network_lemma}
\end{align}
for
\begin{align}
\varepsilon_M= C_3(d-1)\sum_{k=0}^{M-1}\sqrt{r_{k}}\delta_k^{3/2}
\label{eq:epsilon_M-bound}
\end{align}
\end{lem}
\begin{proof}
Clearly, we have that
\begin{align}
\Pi_1(\Gamma_{\phi})\subseteq \Sigma_\phi^{\varepsilon_1}(\partial\mathcal{P}_1)\cup \Gamma(\mathcal{P}_1)
\end{align}
with $\varepsilon_1= C_3(d-1)\sqrt{r_0}\delta_0^{3/2}$ (recall that we defined $r_0=R$). This is an immediate consequence of Proposition~\ref{proposition:multi_projections}. Suppose now that
\begin{align}
\Pi_{k}\circ \hdots \circ \Pi_1(\Gamma_\phi)\subseteq \Sigma_\phi^{\varepsilon_{k}}(\partial\mathcal{P}_{k})\cup \Gamma_{\phi}(\mathcal{P}_{k})
\qquad\text{for some $\varepsilon_{k}\geq 0$.}
\end{align}
Applying an additional polytopal projection $\Pi_{k+1}$ gives the inclusion
\begin{align}
\begin{split}
\Pi_{k+1}\circ\Pi_{k}\circ \hdots \circ \Pi_1(\Gamma_\phi) &\subseteq \Pi_{k+1}\big(\Sigma_\phi^{\varepsilon_{k}}(\partial\mathcal{P}_{k})\cup \Gamma_{\phi}(\mathcal{P}_{k})\big)\\
&=\Pi_{k+1}\big(\Sigma_\phi^{\varepsilon_{k}}(\partial\mathcal{P}_{k}) \big) \cup \Pi_{k+1}\big(\Gamma_{\phi}(\mathcal{P}_{k})\big)
\end{split}
\label{eq:Pi_k-inclusion}
\end{align}
We start by investigating the image $\Pi_{k+1}\big(\Gamma_\phi(\mathcal{P}_{k})\big)$. By Lemma~\ref{lemma:P_l_prime} we have 
\begin{align}
\bar{B}^d_{r_{k}-\delta_{k}}\subset \mathcal{P}_{k+1}\subset \bar{B}^d_{r_{k}}
\label{eq:P_k-inclusion}
\end{align}
Moreover, by Lemma~\ref{proposition:inclusion} we know that $\mathcal{P}_k=(\mathcal{P}_k\setminus \mathcal{P}_{k+1})\cup \mathcal{P}_{k+1}$, so in particular
\begin{align}
\Gamma_\phi(\mathcal{P}_k)=\Gamma_\phi(\mathcal{P}_k\setminus \mathcal{P}_{k+1})\cup \Gamma_\phi(\mathcal{P}_{k+1})
\end{align}
Now, as we have the inclusion $ \mathcal{P}_{k}\setminus \mathcal{P}_{k+1} \subset \bar{B}^d_{r_{k}}\setminus \mathcal{P}_{k+1}$ we can apply Proposition~\ref{proposition:multi_projections} to conclude
\begin{align}
\Pi_{k+1}\big(\Gamma_\phi(\mathcal{P}_{k})\big) \subseteq \Sigma_\phi^{\varepsilon}(\partial\mathcal{P}_{k+1})\cup \Gamma_{\phi}(\mathcal{P}_{k+1})
\label{eq:P_k-inclusion-map}
\end{align}
with $\varepsilon = C_3(d-1)\sqrt{r_{k}}\delta_{k}^{3/2}$. 

Next, we consider $\Pi_{k+1}\big(\Sigma_\phi^{\varepsilon_{k}}(\partial\mathcal{P}_{k}) \big)$. Note that a point $(x_0,y_0)\in \Sigma_\phi^{\varepsilon_{k}}(\partial\mathcal{P}_{k})$ can be expressed as $(x_0,\phi(x_0)+h)$ where $|h|\leq \varepsilon_{k}$. By \eqref{eq:map-point} it is clear that for any layer $\pi$ we are considering it always holds that $\pi(x,y+h)=\pi(x,y)+(0,h)$. Thus, if $(x_0,\phi(x_0)+h)\in \Sigma_\phi^{\varepsilon_{k}}(\partial\mathcal{P}_{k})$ then
\begin{align}
\Pi_{k+1}\big(x_0,\phi(x_0)+h\big)=(x_{m_{k+1}},y_{m_{k+1}})+(0,h)=(x_{m_{k+1}},y_{m_{k+1}}+h)
\end{align}
where we define $(x_{m_{k+1}},y_{m_{k+1}})=\Pi_{k+1}\big(x_0,\phi(x_0)\big)$ in accordance with \eqref{eq:map_sequence}. Since \eqref{eq:P_k-inclusion} holds and $\delta_{k}\leq r_{k}\left(1-\frac{1}{\sqrt{2}}\right)$ we have by Lemma~\ref{lemma:condition_delta} that $x_{m_{k+1}}\in \partial \mathcal{P}_{k+1}$. Moreover, by the triangle inequality we get
\begin{align}
|y_{m_{k+1}}+h-\phi(x_{m_{k+1}})|\leq |y_{m_{k+1}}-\phi(x_{m_{k+1}})|+|h|\leq C_3(d-1)\sqrt{r_{k}}\delta_{k}^{3/2}+|h|
\end{align}
where we used Lemma~\ref{lemma:ym} in the last inequality. As $|h|\leq \varepsilon_{k}$ we can conclude

\begin{align}
\Pi_{k+1}\big(\Sigma_\phi^{\varepsilon_{k}}(\partial\mathcal{P}_{k}) \big)\subseteq \Sigma_\phi^{\varepsilon+\varepsilon_{k}}(\partial\mathcal{P}_{k+1})
\label{eq:Sigma_inclusion}
\end{align}
with $\varepsilon= C_3(d-1)\sqrt{r_{k}}\delta_{k}^{3/2}$.

Combining \eqref{eq:P_k-inclusion-map} and \eqref{eq:Sigma_inclusion} in equation \eqref{eq:Pi_k-inclusion} we can deduce
\begin{align}
\Pi_{k+1}\circ\Pi_{k}\circ \hdots \circ \Pi_1(\Gamma_\phi) &\subseteq \Sigma_{\phi}^{\varepsilon_{k+1}}(\partial \mathcal{P}_{k+1})\cup \Gamma_{\phi}(\mathcal{P}_{k+1})
\end{align}
where we defined
\begin{align}
\varepsilon_{k+1}=  \varepsilon_{k}+C_3(d-1)\sqrt{r_{k}}\delta_{k}^{3/2}
\end{align}
Note, since $\varepsilon_1 = C_3(d-1)\sqrt{r_{0}}\delta_{0}^{3/2}$ the explicit formula for $\varepsilon_{k+1}$ simply reduces to
\begin{align}
\varepsilon_{k+1}= C_3(d-1)\sum_{l=0}^{k}\sqrt{r_{l}}\delta_{l}^{3/2}
\end{align}
Now the inclusion in \eqref{eq:image_under_network_lemma} and equation \eqref{eq:epsilon_M-bound} follow by induction on $k$.
\end{proof}
\begin{cor} The size of $\varepsilon_M$ can be bounded above by
\begin{align}
\varepsilon_M\leq  C_3(d-1)M\sqrt{R}\delta^{3/2}
\end{align}
where $M$ is the number of polytopes.
\label{cor:epsilon_M}
\end{cor}
\begin{proof}
The inequality follows directly from equation \eqref{eq:epsilon_M-bound} by noting that $r_k\leq R$ and $\delta_k\leq \delta$ for all $0\leq k \leq M-1$.
\end{proof}

\paragraph{Number of polytopes.} The size of the resulting $\varepsilon_M$-band depends on the number of polytopes in our construction which in turn depends on how small we require the last radius $r_M$ to be. We would like to choose $M$ large enough so that the $r_M$ is small enough to let us approximate $\Gamma_{\phi}(\mathcal{P}_M)$ by the hyperplane $P_{\tilde{L}}=\ker(\tilde{L})$ to a desired accuracy. For our purpose it will be sufficient to require $r_M\leq \delta$ so that $\mathcal{P}_M\subset \bar{B}^d_{\delta}$. 

The number of polytopes, $M$, needed in our construction to ensure $r_M\leq \delta$ will depend on the parameter $\delta$ so we write $M=M(\delta)$. We have the following bound on $M(\delta)$.
\begin{lem}[Number of Polytopes]
\label{lem:num_polytopes}
The number of polytopes $M(\delta)$ needed in our construction to ensure $r_M\leq \delta$ satisfies
\begin{align}
\boxed{
M(\delta)\leq \frac{7R}{3\delta}
}
\label{eq:num_polytopes}
\end{align}
\end{lem}
The proof relies on the lower bound on the reduction of the radii given in Lemma~\ref{proposition:inclusion}. The details can be found in Appendix~\ref{Appendix B}. We summarize our findings in the following proposition.
\begin{prop}[Image of the Graph] \label{proposition:multi-polytopes}
Our construction of the mappings $\Pi_k$, for $1\leq k\leq M$, guarantees that 
\begin{align}
\boxed{
\begin{array}{c}
\Pi_M\circ \hdots \circ\Pi_2 \circ\Pi_1\big( \Gamma_{\phi}\big)\subseteq \Sigma_{\phi}^{\varepsilon}(\partial \mathcal{P}_M) \cup\Gamma_{\phi}(\mathcal{P}_M)\\
\varepsilon= C_4(d-1)R^{3/2}\delta^{1/2}\\
\mathcal{P}_M\subset \bar{B}^d_{\delta}\\
M\leq \frac{7R}{3\delta}
\end{array}
}
\label{eq:result_summary}
\end{align}
for a constant $C_4$ independent of $R$, $d$ and $\delta$.
\end{prop}
\begin{proof}
Using Lemma~\ref{lem:num_polytopes} in the bound given in Corollary~\ref{cor:epsilon_M} yields
\begin{align}
\begin{split}
\varepsilon_M&\leq C_3(d-1)\frac{7R}{3\delta}\sqrt{R}\delta^{3/2}=C_4(d-1)R^{3/2}\delta^{1/2}
\end{split}
\end{align}
where $C_4=\frac{7}{3}C_3$. Now, the proposition follows by Lemma~\ref{lemma:projection_multiple_polytopes} and  \ref{lem:num_polytopes}.
\end{proof}

Thus, to ensure that $\mathcal{P}_M\subset \bar{B}^d_{\delta}$ we will need at least $\frac{7R}{3\delta}$ number of polytopes in our construction, and the size of the resulting $\varepsilon$-band of $\phi$ over $\partial \mathcal{P}_M$ scales as $\delta^{1/2}$. The sequence of polytopes for a given $\delta$ is depicted in Figure~\ref{fig:Sequence of Polytopes} together with the resulting $\varepsilon$-band. Proposition~\ref{proposition:multi-polytopes} shows that we can construct a network such that the graph $\Gamma_{\phi}$ is mapped in a controlled way and that the size of the $\varepsilon$-band can be made arbitrarily small by choosing $\delta$ sufficiently small. Before turning to what the implications are for the decision boundary with respect to our construction, we will first estimate the total number of layers needed in the realization of the network we have described.
\begin{figure} 
\centering
\begin{subfigure}[b]{0.40\textwidth}
\centering
\includegraphics[width=\textwidth]{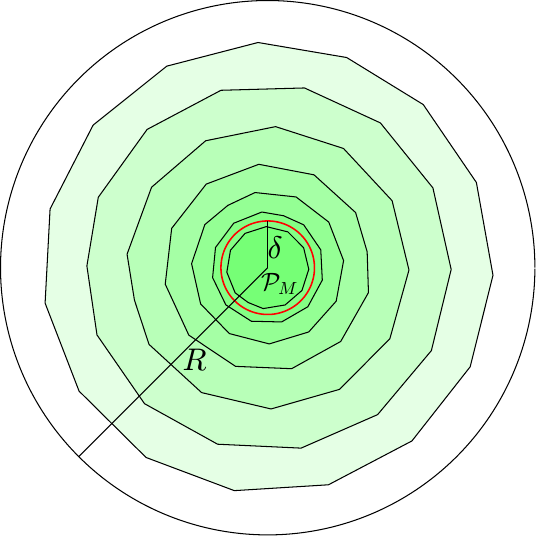}
\caption{ }
\end{subfigure}
\hfill
\begin{subfigure}[b]{0.59\textwidth}
\centering
\includegraphics[width=\textwidth]{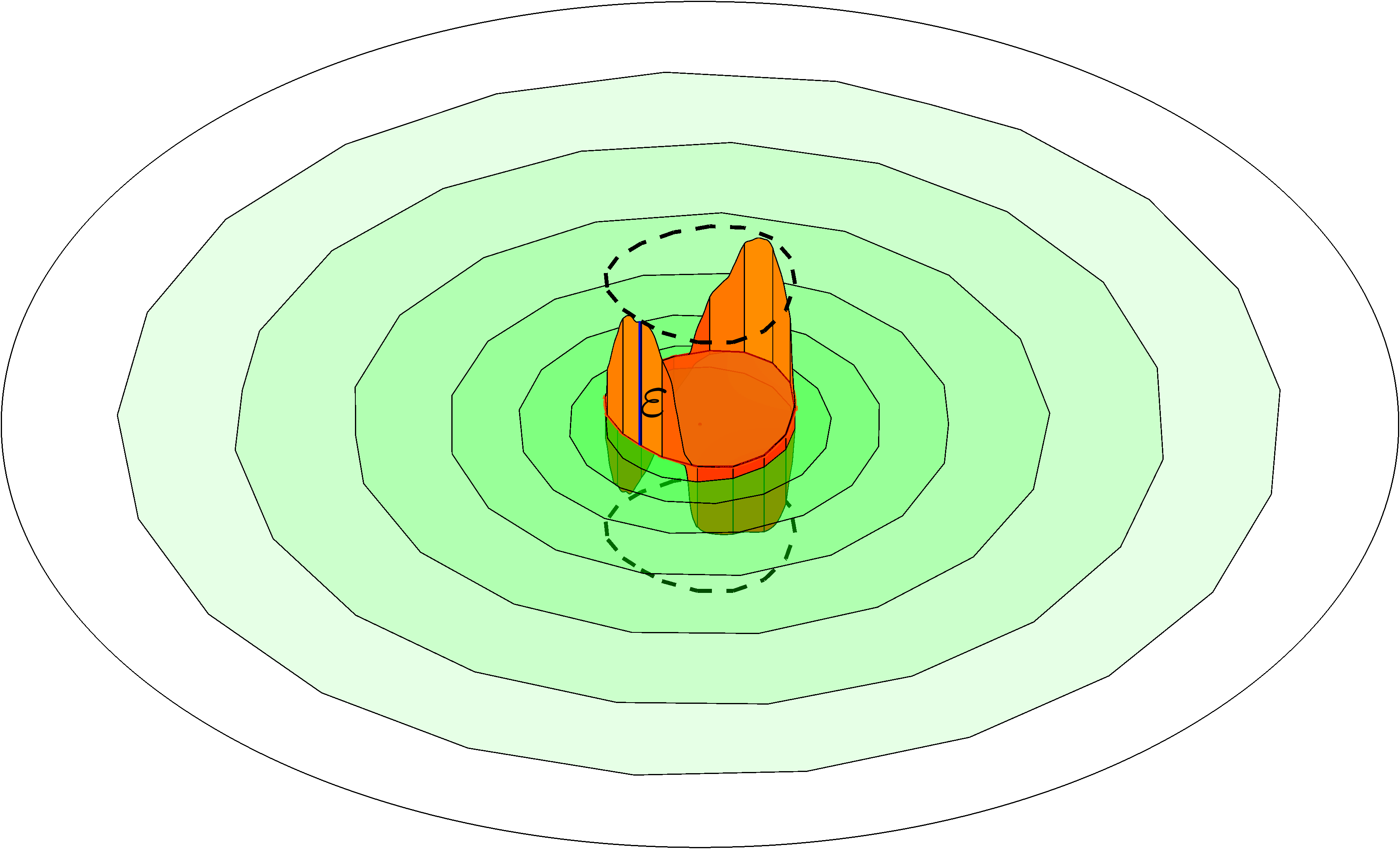}
\caption{ }
\end{subfigure}
\hfill
\caption{\textit{Sequence of Polytopes.} \textbf{(a)} The number of polytopes is chosen such that the last one, $\mathcal{P}_M$, is contained in a ball of radius $\delta$. \textbf{(b)} The image of $\Gamma_{\phi}$ under the $M$ maps consist of one unaffected part over the interior of $\mathcal{P}_M$ and a projected part inside an $\varepsilon$-band of $\phi$ over $\partial \mathcal{P}_M$.}
\label{fig:Sequence of Polytopes}
\end{figure}

\subsection{Total Number of Layers}
Recall that the total number of layers $N$ is equal to the sum in \eqref{eq:tot layers}. Previously, we derived the bound in Lemma~\ref{lem:num_polytopes} on the number of polytopes $M$ which equals the number of maps in \eqref{eq:Pi_composition}. In each map $\Pi_k$ we have $m_{k}$ layers and in our construction we defined $m_k$ as the cardinality of an $\epsilon_{k-1}$-net of the sphere $S^{d-1}_{r_{k-1}}$. 
\paragraph{Separations.} To prove Lemma~\ref{proposition:inclusion}, we defined $\epsilon_k=\sqrt{\frac{1}{2}\delta_k r_k}$. Thus, to give a bound on the number of layers needed in our construction we will need to estimate the cardinality of an $\epsilon$-net of a sphere. We begin by defining the related concept of an $\epsilon$-separation.
\begin{definition}[\boldmath{$\mathbf{\epsilon}$}-separation]
Given a set $V\subseteq \IR^{d}$ we say that a subset $N^{sep}_{\epsilon}\subseteq V$ is an $\epsilon$-separation of $V$ if for any distinct pair of points $p,q\in N^{sep}_{\epsilon}$ we have that $\|p-q\|> \epsilon$. If in addition there is no superset $M^{sep}_{\epsilon}\supset N^{sep}_{\epsilon}$ which is also an $\epsilon$-separation of $V$, we say that $N^{sep}_{\epsilon}$ is a maximal $\epsilon$-separation of $V$.
\end{definition}
A standard way of estimating the cardinality of $\epsilon$-nets is to give a bound on maximal $\epsilon$-separations and then apply the following Lemma.
\begin{lem}
A maximal $\epsilon$-separation of a set $V$ is also an $\epsilon$-net of $V$.
\label{lemma:sep-net}
\end{lem}
\begin{proof}
Let $N^{sep}_{\epsilon}$ be a maximal $\epsilon$-separation of $V$. Suppose there is a point $p\in V$ such that $\|p-q\|>\epsilon$ for all $q \in N^{sep}_{\epsilon}$. Then $N^{sep}_{\epsilon}\cup \{p\}$ would still be an $\epsilon$-separation of $V$ contradicting the maximality of $N^{sep}_{\epsilon}$. Thus, there can't exist such a point $p\in V$, and hence $N^{sep}_{\epsilon}$ is an $\epsilon$-net of $V$. 
\end{proof}
We will now give a bound on $|N_{\epsilon}|$ in case $N_{\epsilon}$ is an $\epsilon$-net of a sphere $S^{d-1}_r$ of radius $r$.
\begin{lem}
\label{lem:cardinality of net}
There is an $\epsilon$-net $N_{\epsilon}$ of the sphere $S^{d-1}_r$ such that $|N_{\epsilon}|\leq 2d\big(1+\frac{2r}{\epsilon}\big)^{d-1}$. 
\end{lem}
In the proof, outlined in Appendix~\ref{Appendix B}, we use the relation between $\epsilon$-nets and $\epsilon$-separations presented in Lemma~\ref{lemma:sep-net}. Using Lemma~\ref{lem:cardinality of net} we can now give a bound on the number of layers needed in our construction.
\begin{prop}[Number of Layers]
\label{cor:tot layers}
The total number of layers $N$ in our construction satisfies
\begin{align}
\boxed{
N\leq Cd\bigg(\frac{32R}{\delta}\bigg)^{(d+1)/2}
}
\end{align}
for a constant $C$ independent of $\delta$, $R$, $d$ and $\phi$.
\end{prop}
\begin{proof}
From Lemma~\ref{lem:cardinality of net} we have that 
\begin{align}
m_k\leq 2d\bigg(1+\frac{2r_{k-1}}{\epsilon_{k-1}}\bigg)^{d-1}
\end{align}
Plugging in $\epsilon_{k-1}=\sqrt{\frac{1}{2}\delta_{k-1}r_{k-1}}$ gives
\begin{align}
\begin{split}
m_k&\leq 2d\bigg(1+\frac{2r_{k-1}}{\sqrt{\frac{1}{2}\delta_{k-1}r_{k-1}}}\bigg)^{d-1}=2d\bigg(1+2\sqrt{2}\sqrt{\frac{r_{k-1}}{\delta_{k-1}}}\bigg)^{d-1}\\
&\leq 2d\bigg(1+2\sqrt{2}\sqrt{\frac{R}{\delta}}\bigg)^{d-1} \leq 2d\bigg(4\sqrt{2}\sqrt{\frac{R}{\delta}}\bigg)^{d-1}=2d\bigg(\frac{32R}{\delta}\bigg)^{(d-1)/2}
\end{split}
\end{align}
Now, by \eqref{eq:tot layers} we get
\begin{align}
N=\sum_{k=1}^Mm_k\leq M 2d\bigg(\frac{32R}{\delta}\bigg)^{(d-1)/2}
\end{align}
and from Lemma~\ref{lem:num_polytopes} we then arrive at
\begin{align}
N\leq \frac{14}{3}d\bigg(\frac{32R}{\delta}\bigg)^{(d+1)/2}
\end{align}
\end{proof}
\section{Decision Boundary}
\label{sec:Decision Boundary}
We can express the decision boundary $\Gamma$ in terms of our canonical network architecture \eqref{eq:modified_network} as
\begin{align}
\Gamma=\big(\Pi_M\circ \hdots \circ  \Pi_2 \circ \Pi_1\big)^{-1}[\ker(\tilde{L})]=\Pi_1^{-1}\circ \Pi_2^{-1}\circ \hdots \circ \Pi_M^{-1}[\ker(\tilde{L})]
\label{eq:Gamma_F preimage of P_L}
\end{align}
where each
\begin{align}
\Pi^{-1}_k=(\pi^k_1)^{-1}\circ (\pi^k_2)^{-1} \circ \hdots \circ (\pi^k_{m_k})^{-1}
\label{eq:Pi_k preimage pi^k_i}
\end{align} 
and $(\pi^k_{i})^{-1}[V]$ denotes the preimage of a set $V$ under layer $\pi^k_{i}$. 
\paragraph{Defining the Affine Function.} Up to now we have described the construction of each layer in our network but the last affine function $\tilde{L}$. We will now show that we can choose the hyperplane 
\begin{align}
\boxed{
P_{\tilde{L}}=\ker(\tilde{L})
}
\end{align}
so that its restriction to $\mathcal{P}_M$, denoted by $P_{\tilde{L}}\big(\mathcal{P}_M\big)$, is contained in the $\varepsilon$-band of $\phi$ over $\mathcal{P}_M$ for the same value of $\varepsilon$ as in Proposition~\ref{proposition:multi-polytopes}, see Figure~\ref{fig:Hyperplane}. Since the graph of an affine function in $d$-variables is a hyperplane in $\IR ^{d+1}$ we can choose
\begin{align}
P_{\tilde{L}}=\{(x,y)\in\IR^{d+1} : y=\tilde{l}(x)\}
\label{eq:defiition P_L}
\end{align}
for the affine function $\tilde{l}:B^d_R\to \IR$ given by 
\begin{align}
\tilde{l}(x)=\phi(0)+\sum_{i=1}^d\frac{\partial \phi(0)}{\partial x_i}x_i
\end{align}
As $\phi\in C^2\big(B^d_R\big)$ we have by Taylor's theorem that  
\begin{align}
\phi(x)=\phi(0)+\sum_{i=1}^d\frac{\partial \phi(0)}{\partial x_i}x_i+\frac{1}{2}\sum_{i,j=1}^dR_{i,j}(x)x_ix_j
\end{align}
for any $x\in \mathcal{P}_M\subseteq \bar{B}^d_{\delta}$ and where $|R_{i,j}(x)|\leq D$. Hence,
\begin{align}
|\phi(x)-\tilde{l}(x)|\leq\frac{D}{2}\bigg|\sum_{i,j=1}^dx_ix_j\bigg|=\frac{D}{2}\bigg(\sum_{i=1}^dx_i\bigg)^2
\end{align}
for all $x\in \mathcal{P}_M$. Since $x\in \mathcal{P}_M\subset \bar{B}^d_{\delta}$ it follows that $\big(\sum_{i=1}^dx_i\big)^2\leq d\delta^2$ so we get the upper bound
\begin{align}
|\phi(x)-\tilde{l}(x)|\leq \frac{D}{2}d\delta^2, \quad \text{for all $x\in \mathcal{P}_M$}
\end{align}
Hence, we can write
\begin{align}
\|\phi-\tilde{l}\|_{C(\mathcal{P}_M)}\leq \frac{D}{2}d\delta^2
\end{align}
Recalling the value of $\varepsilon$ in Proposition~\ref{proposition:multi-polytopes} and since $C_4=\frac{7}{3}C_3=\frac{7}{3}C_2C_1=\frac{7}{3}\frac{(1+\sqrt{2})}{2}C_1=\frac{7(1+\sqrt{2})}{2\sqrt{2}}D$ (see the proof of Proposition~\ref{proposition:one_projection} in Appendix~\ref{Appendix B}) it follows that
\begin{align}
\|\phi-\tilde{l}\|_{C(\mathcal{P}_M)}\leq \frac{D}{2}d\delta^2 \leq \varepsilon
\end{align}
Thus, we have 
\begin{align}
P_{\tilde{L}}\big(\mathcal{P}_M \big)\subseteq  \Sigma_{\phi}^{\varepsilon}( \mathcal{P}_M)
\label{eq:P_L-inclusion}
\end{align}
\begin{figure} 
\centering
\includegraphics[width=0.25\linewidth]{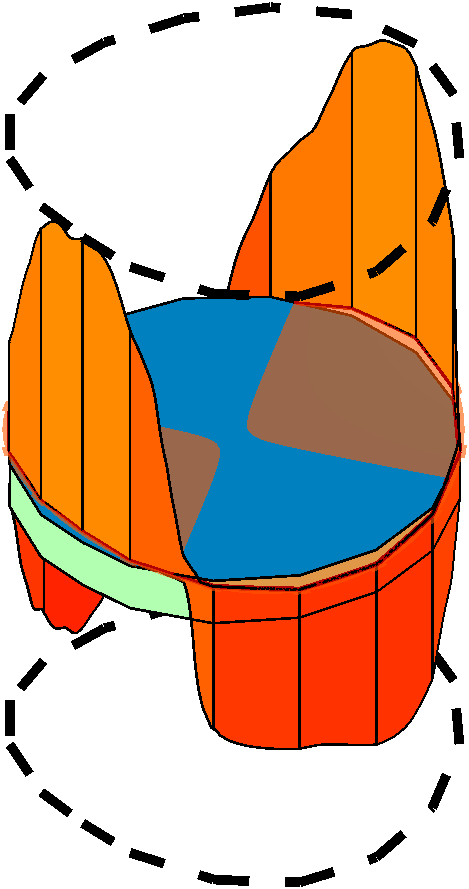}
\caption{\textit{Approximation by Hyperplane.} The hyperplane $P_{\tilde{L}}=\ker(\tilde{L})$ approximates the graph of $\phi$ over the last polytope $\mathcal{P}_M$. Its restriction $P_{\tilde{L}}(\mathcal{P}_M)$ (in blue) is contained in $\Sigma^{\varepsilon}_{\phi}(\mathcal{P}_M)$ for the value of $\varepsilon$ given in Proposition~\ref{proposition:multi-polytopes}.}
\label{fig:Hyperplane}
\end{figure}

\paragraph{Preimages.} By construction, a point $(x,y)$ in the domain of the network $\tilde{F}$ with $x\in B^d_R\setminus \mathcal{P}_M$ will be mapped into the set $\partial\mathcal{P} \times \IR$ by the composition $\Pi_M\circ \hdots \circ  \Pi_2 \circ \Pi_1$. Thus, any point $(x',y')$ with $x'\notin \mathcal{P}_M$ will have an empty preimage under this composition. Therefore, we can express the decision boundary as the preimage of just $P_{\tilde{L}}\big(\mathcal{P}_M \big)$ as this is exactly the part of the hyperplane $P_{\tilde{L}}$ with a nonempty preimage. Then, by the inclusion in \eqref{eq:P_L-inclusion} we have that
\begin{align}
\Gamma\subseteq \big(\Pi_1^{-1}\circ \Pi_2^{-1}\circ \hdots \circ \Pi_M^{-1}\big)[\Sigma_{\phi}^{\varepsilon}( \mathcal{P}_M)]
\label{eq:Gamma_F inclusion}
\end{align}
for the value of $\varepsilon$ given in Proposition~\ref{proposition:multi-polytopes}. As each $\Pi_k$ is a composition of several layers $\pi^k_i$ we will need to understand their preimages as these are the building blocks in our network.
\begin{lem}[Preimages]
\label{lem:preimage pi}
Let $g:B^d_R\to \IR$ be a continuous function and denote its graph by $\Gamma_g$. Consider a layer $\pi^k_i:B^d_R \times \IR \to \IR^{d+1}$ in our network $\tilde{F}$. Then, the preimage of $\Gamma_g$ under $\pi^k_i$ fulfills
\begin{align}
(\pi^k_i)^{-1}[\Gamma_g]=\Gamma_{\tilde{g}}
\label{eq: preimage of Gamma_g}
\end{align}
where $\Gamma_{\tilde{g}}$ is the graph of a continuous function $\tilde{g}:B^d_R\to \IR$. Moreover, given any $\tau\geq 0$ we have
\begin{align}
(\pi^k_i)^{-1}[\Sigma^{\tau}_{g}(B^d_R)]=\Sigma^{\tau}_{\tilde{g}}(B^d_R)
\label{eq: preimage of Sigma}
\end{align}
\end{lem}
The details of the proof can be found in Appendix~\ref{Appendix C}. The idea is that the preimage of the intersection of the graph and the hyperplane associated to a layer is spanned by a constant vector, the negative projection direction for that layer. The part of the graph contained in the halfspace will just be its own preimage as the layer is the identity when restricted to this set. This is illustrated in Figure~\ref{fig:Preimagefirst}. 
\begin{figure} 
\centering
\begin{subfigure}[b]{0.45\textwidth}
\centering
\includegraphics[width=\textwidth]{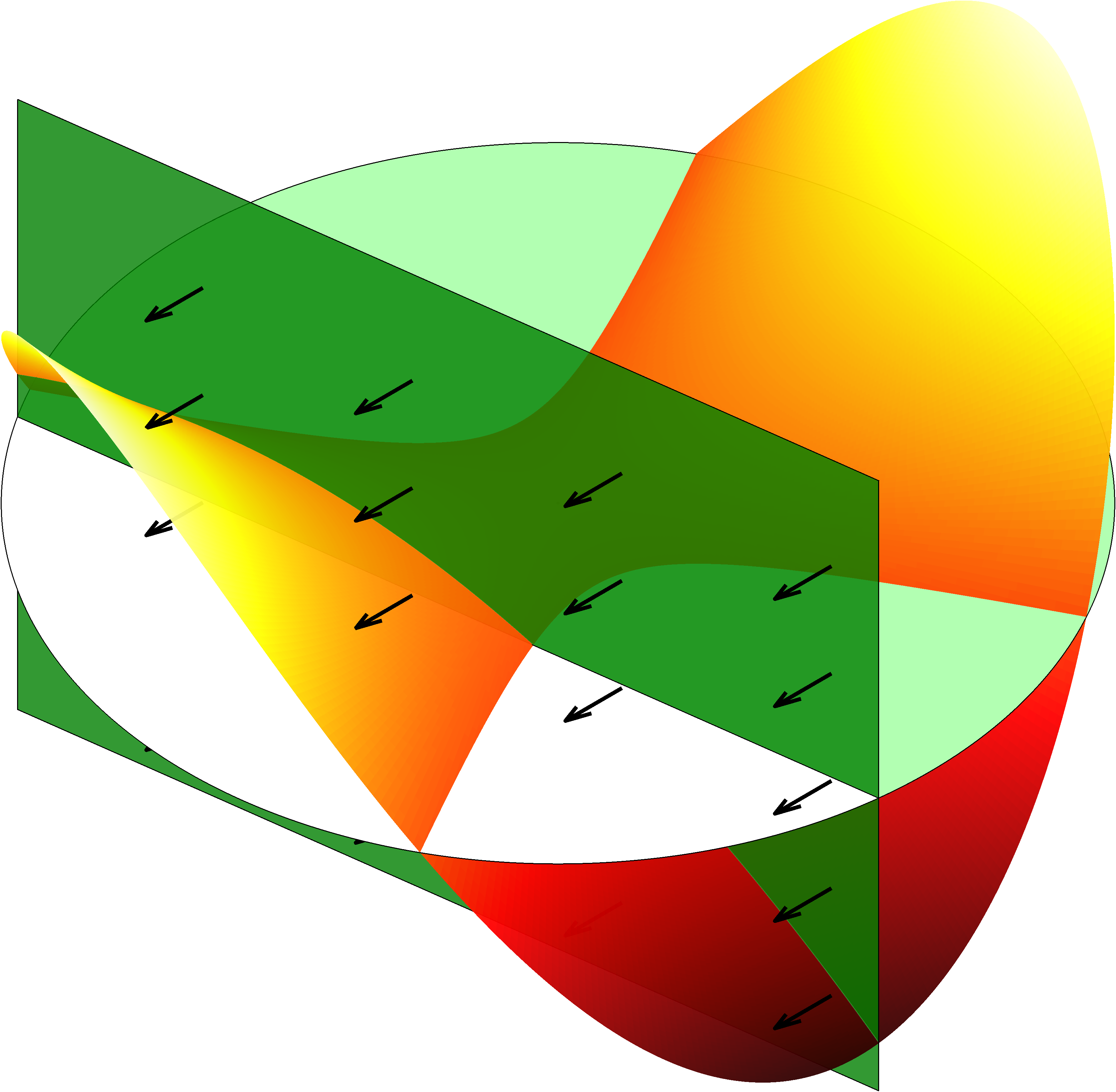}
\caption{ }
\end{subfigure}
\hfill
\begin{subfigure}[b]{0.45\textwidth}
\centering
\includegraphics[width=\textwidth]{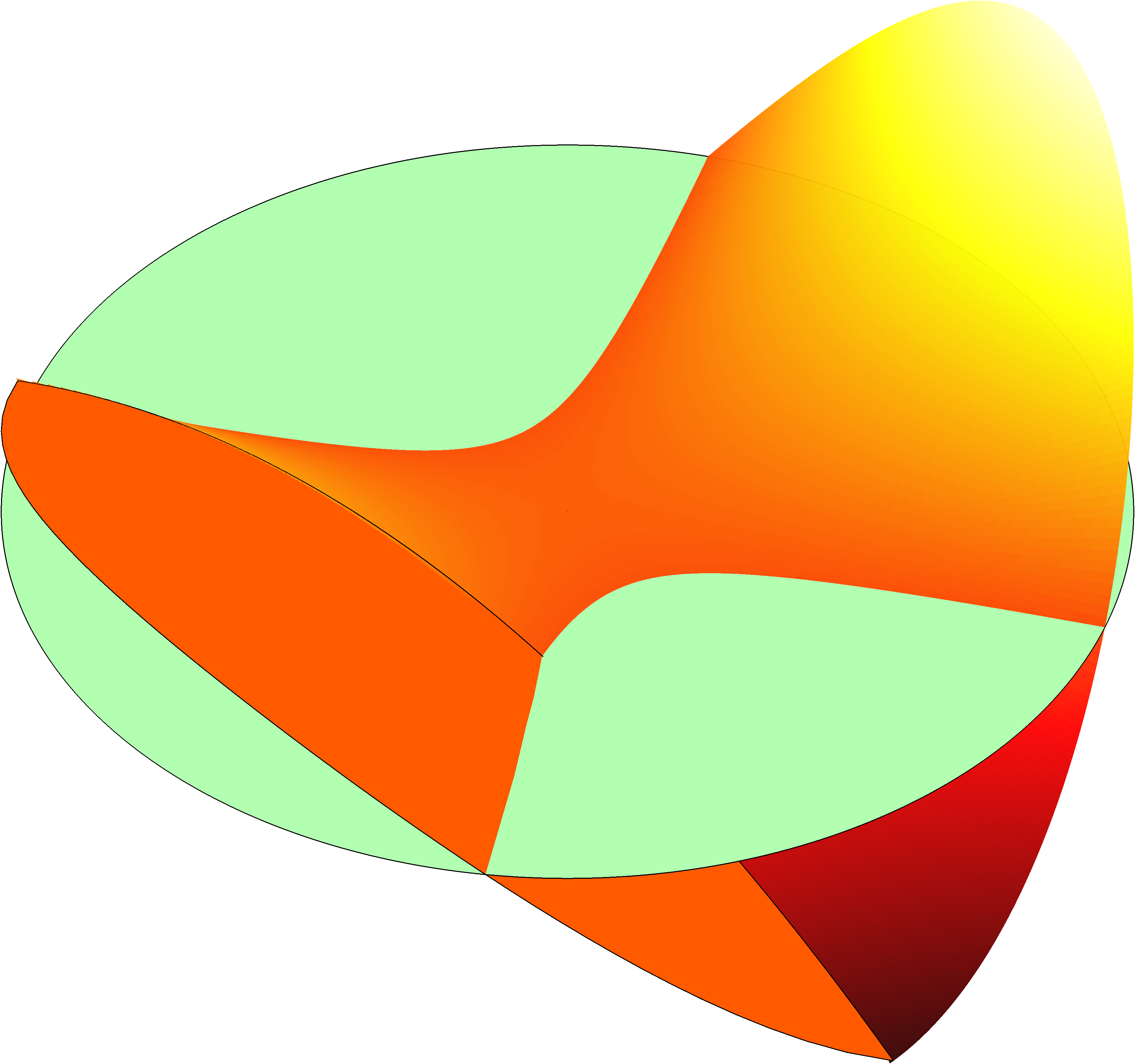}
\caption{ }
\end{subfigure}
\hfill
\caption{\textit{Preimage of Graphs.} \textbf{(a)} The preimage of a graph under a layer $\pi^k_i$ in the network. The part of the graph over the interior of the half-space $U^k_i$ will have a trivial preimage whereas the part over the complement $(U^k_i)^c$ will have an empty preimage. However, the preimage of the intersection of the graph with the hyperplane $P^k_i$ will be spanned by $-\xi^k_i$ as depicted. \textbf{(b)} The entire preimage will define a graph of a new continuous function which coincide with the original graph on the halfspace $U^k_i$.}
\label{fig:Preimagefirst}
\end{figure}

\paragraph{Decision Boundary.} Using Lemma~\eqref{lem:preimage pi} we can prove that the decision boundary $\Gamma$ is the graph of a continuous piecewise linear function.
\begin{cor}
There is a continuous piecewise linear function $\hat{\phi}:B^d_R\to \IR$ with graph $\Gamma_{\hat{\phi}}$ such that $\Gamma=\Gamma_{\hat{\phi}}$.
\end{cor}
\begin{proof}
From \eqref{eq:Gamma_F preimage of P_L} and \eqref{eq:Pi_k preimage pi^k_i} we see that $\Gamma$ is the preimage of the hyperplane $P_{\tilde{L}}$ under a sequence of layers $\pi^k_i$. Recall equation \eqref{eq:defiition P_L} where we defined $P_{\tilde{L}}$ as the graph of the affine function $\tilde{l}:B^d_R\to \IR$. Thus, taking the preimage of $P_{\tilde{L}}$ by a layer will generate a graph of a new continuous function by Lemma~\ref{lem:preimage pi}. In particular, it follows from the proof of Lemma~\ref{lem:preimage pi} that the preimage of the graph of a piecewise linear function under a layer will be the graph to a piecewise linear function. Hence, by repeatedly applying this lemma, for all layers in the network, we will get a piecewise linear continuous function $\hat{\phi}:B^d_R \to \IR$ such that its graph $\Gamma_{\hat{\phi}}$ coincides with the decision boundary $\Gamma$.
\end{proof}
This corollary guarantees that the decision boundary of our network can be represented as the graph of a piecewise linear continuous function defined on $B^d_R$. Therefore, we will hereafter denote the decision boundary $\Gamma$ by $\Gamma_{\hat{\phi}}$ where $\hat{\phi}$ is the function which existence is established by the corollary above. We choose to denote this function by $\hat{\phi} $ as it will be a piecewise linear approximation of $\phi$. To prove the main result of this paper the following proposition will be useful.
\begin{prop}
\label{lem:Gamma Gamma_phi inclusion}
There is a continuous function $\tilde{\phi}:B^d_R\to \IR$ such that
\begin{align}
\begin{array}{c}
\Gamma_{\hat{\phi}} \subseteq \Sigma^{\varepsilon}_{\tilde{\phi}}(B^d_R)\\
\Gamma_{\phi} \subseteq \Sigma^{\varepsilon}_{\tilde{\phi}}(B^d_R)
\end{array}
\end{align}
where $\varepsilon=C_4(d-1)R^{3/2}\delta^{1/2}$.
\end{prop}  
\begin{proof}
First, we make the observation that
\begin{align}
\big(\Pi_1^{-1}\circ \Pi_2^{-1}\circ \hdots \circ \Pi_M^{-1}\big)[\Sigma^{\varepsilon}_{\phi}(\mathcal{P}_M)] =\big(\Pi_1^{-1}\circ \Pi_2^{-1}\circ \hdots \circ \Pi_M^{-1}\big)[\Sigma^{\varepsilon}_{\phi}(B^d_R)] 
\end{align}
as the preimage of the set $\Sigma^{\varepsilon}_{\phi}(B^d_R) \setminus \Sigma^{\varepsilon}_{\phi}(\mathcal{P}_M)$ is empty since the image of the domain under this composition is contained in $\partial \mathcal{P}_M\times \IR$. Thus, from equation \eqref{eq:Gamma_F inclusion} we have the inclusion
\begin{align}
\Gamma_{\hat{\phi}}=\Gamma\subseteq \big(\Pi_1^{-1}\circ \Pi_2^{-1}\circ \hdots \circ \Pi_M^{-1}\big)[\Sigma^{\varepsilon}_{\phi}(B^d_R)]
\label{eq:Gamma_f inclusion}
\end{align}
with $\varepsilon=C_4(d-1)R^{3/2}\delta^{1/2}$. Now, we define $\tilde{\phi}$ to be the function whose graph $\Gamma_{\tilde{\phi}}$ is given by
\begin{align}
\Gamma_{\tilde{\phi}}=\big(\Pi_1^{-1}\circ \Pi_2^{-1}\circ \hdots \circ \Pi_M^{-1}\big)[\Gamma_{\phi}]
\end{align}
Note that the preimage on the right-hand side above is guaranteed to generate a graph of a continuous function on $B^d_R$ by repeatedly applying the first part of Lemma~\ref{lem:preimage pi} to each layer in the composition. Similarly, applying the second part of Lemma~\ref{lem:preimage pi} repeatedly we get
\begin{align}
\big(\Pi_1^{-1}\circ \Pi_2^{-1}\circ \hdots \circ \Pi_M^{-1}\big)[\Sigma^{\varepsilon}_{\phi}(B^d_R)]=\Sigma^{\varepsilon}_{\tilde{\phi}}(B^d_R)
\label{eq:Sigma phi tilde}
\end{align}
Combining the inclusion in \eqref{eq:Gamma_f inclusion} and equation \eqref{eq:Sigma phi tilde} gives
\begin{align}
\Gamma_{\hat{\phi}}\subseteq \Sigma^{\varepsilon}_{\tilde{\phi}}(B^d_R)
\end{align}
which we wanted. To prove the second inclusion we first note that by Proposition~\ref{proposition:multi-polytopes} we know
\begin{align}
\Pi_M\circ\hdots \circ \Pi_2\circ \Pi_1\big(\Gamma_{\phi}\big)\subseteq \Sigma^\varepsilon_{\phi}(\partial \mathcal{P}_M)\cup \Gamma_{\phi}(\mathcal{P}_M)
\end{align}
for the given $\varepsilon$. But as $\Sigma^\varepsilon_{\phi}(\partial \mathcal{P}_M)\cup \Gamma_{\phi}(\mathcal{P}_M)\subset \Sigma^\varepsilon_{\phi}(\mathcal{P}_M)\subset \Sigma^\varepsilon_{\phi}(B^d_R)$ it is also true that
\begin{align}
\Pi_M\circ\hdots \circ \Pi_2\circ \Pi_1\big(\Gamma_{\phi}\big)\subseteq \Sigma^\varepsilon_{\phi}(B^d_R)
\end{align}
Thus, in particular
\begin{align}
\Gamma_{\phi}\subseteq \big(\Pi_1^{-1}\circ \Pi_2^{-1}\circ \hdots \circ \Pi_M^{-1}\big)[\Sigma^\varepsilon_{\phi}(B^d_R)]
\end{align}
and hence
\begin{align}
\Gamma_{\phi}\subseteq \Sigma^{\varepsilon}_{\tilde{\phi}}(B^d_R)
\end{align}
by \eqref{eq:Sigma phi tilde}.
\end{proof}
We are now ready to state and prove the main result of this paper.
\begin{thm}[Error Estimate]
\label{thm:main-result}
Let $\phi:B^d_R \to \IR$ be a $C^2$-function with bounded second-order partial derivatives. For each positive $\delta\leq R\left(1-\frac{1}{\sqrt{2}}\right)$, we can construct a fully-connected ReLU network $F:B^d_R \times \IR \to \IR$ of width $d+1$ and depth at most $Cd\big(\frac{32R}{\delta}\big)^{\frac{d+1}{2}}$, for a constant $C$, such that the decision boundary $\Gamma$ of $F$ can be represented as the graph $\Gamma_{\hat{\phi}}$ of a continuous piecewise linear function $\hat{\phi}:B^d_R \to \IR$ satisfying 
\begin{align}
\label{eq:main-result}
 \boxed{
\sup_{x\in B^d_R}|\phi(x)-\hat{\phi}(x)|\leq C_5(d-1)R^{3/2}\delta^{1/2}
}
\end{align}
where $C_5$ is a constant independent on $\delta, R$ and $d$.
\end{thm}

\begin{proof}[Proof of Theorem~\ref{thm:main-result}]
Let $\tilde{\phi}:B^d_R\to \IR$ be the function in Proposition~\ref{lem:Gamma Gamma_phi inclusion}. As noted in Remark~\ref{rem:band}, it follows that
\begin{align}
\begin{split}
\sup_{x\in B^d_R}|\hat{\phi}(x)-\tilde{\phi}(x)|&\leq \varepsilon\\
\sup_{x\in B^d_R} |\phi(x)-\tilde{\phi}(x)| &\leq \varepsilon
\end{split}
\end{align}
for $\varepsilon=C_4(d-1)R^{3/2}\delta^{1/2}$. By the triangle inequality we get
\begin{align}
\sup_{x\in B^d_R} |\phi(x)-\hat{\phi}(x)|\leq 2\varepsilon 
\label{eq:2epsilon}
\end{align}
Plugging in the value of $\varepsilon$ in \eqref{eq:2epsilon} gives
\begin{align}
\sup_{x\in B^d_R} |\phi(x)-\hat{\phi}(x)| \leq C_5(d-1)R^{3/2}\delta^{1/2}
\end{align}
where $C_5=2C_4$. The bound on the number of layers follows from Proposition~\ref{cor:tot layers}. As our modified network architecture can be rewritten as a standard fully-connected ReLU network, by Lemma~\ref{lem:Equivalence of Architectures}, the theorem follows.
\end{proof}
An example of the resulting decision boundary generated by the sequence of polytopes in Figure~\ref{fig:Sequence of Polytopes} is shown in Figure~\ref{fig:decision boundary-approx}. By choosing $\delta$ small enough we can approximate $\Gamma_{\phi}$ to any accuracy.
\begin{figure} 
\centering
\includegraphics[scale=0.45]{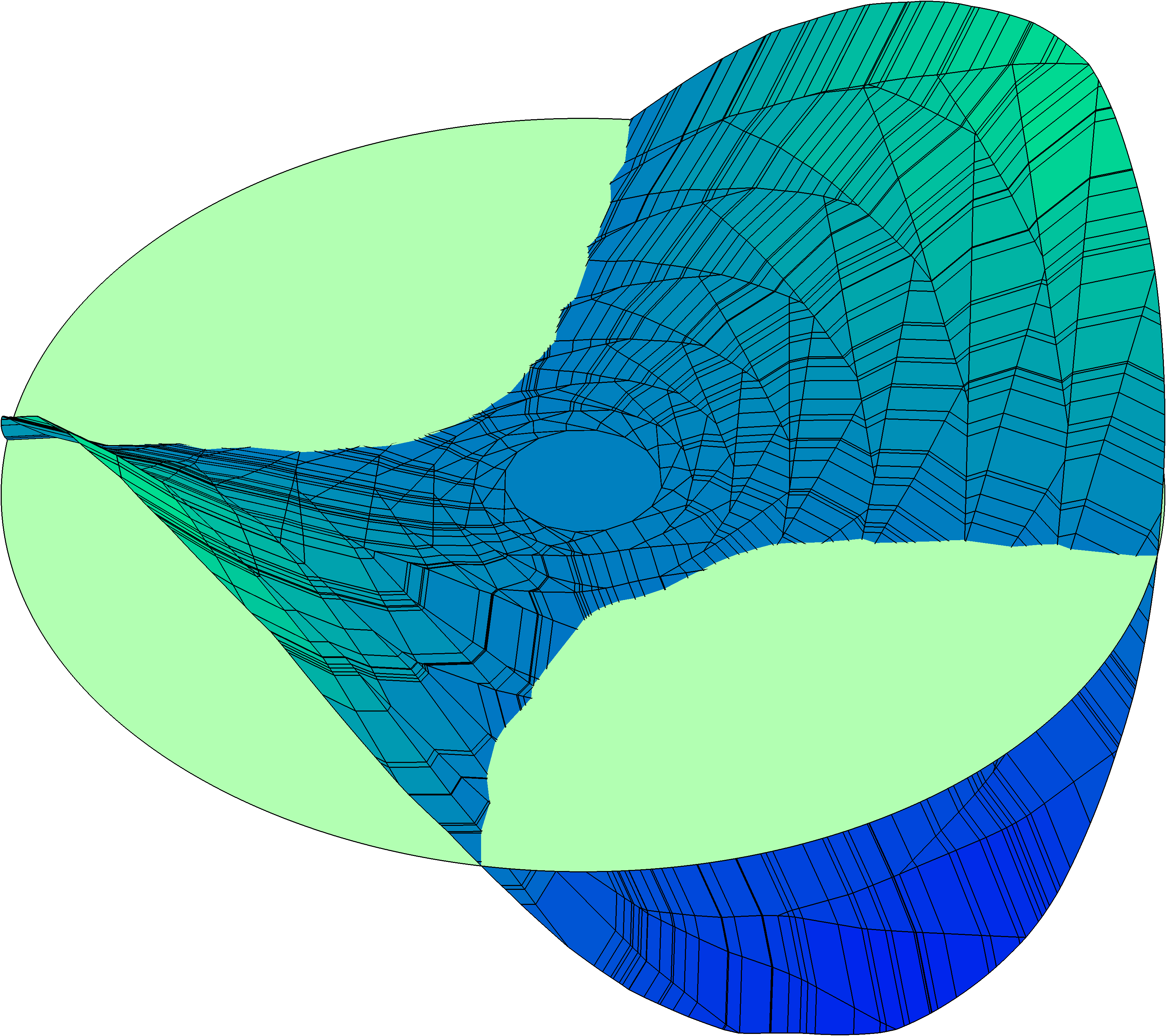}
\caption{\textit{Decision Boundary.} The decision boundary $\Gamma$ of the network, as the preimage of $\ker(\tilde{L})$, will be a piecewise linear approximation of the original graph $\Gamma_{\phi}$. It is the graph of a continuous piecewise linear function $\hat{\phi}$ which approximates $\phi$ on $B^d_R$. The level of approximation depends on the discretization parameter $\delta$.}
\label{fig:decision boundary-approx}
\end{figure}

\begin{cor}
Given $\varepsilon >0$ there is a fully-connected ReLU network $F:B^d \times \IR \to \IR$ of width $d+1$ and depth at most $C\big(\frac{R^2}{\varepsilon}\big)^{d+1}$, for a constant $C=C(\phi,d)$, such that the corresponding function $\hat{\phi}:B^d_R \to \IR$ $\varepsilon$-approximates $\phi$
\begin{align}
\sup_{x\in B^d_R} |\phi(x)-\hat{\phi}(x)| \leq \varepsilon
\end{align}
\end{cor}
\begin{proof}
The result follows immediately from Theorem~\ref{thm:main-result} and the proof of Proposition~\ref{cor:tot layers} by choosing $\delta=\big(\frac{\varepsilon}{C_5(d-1)R^{3/2}}\big)^2$.
\end{proof}
\begin{rem}[Proof of Lemma~\ref{lem:implicit-approximation-capacity}]
Modulo a trivial assumption on the sign of $F$, Lemma~\ref{lem:implicit-approximation-capacity} directly follows from this corollary. The main difference to this corollary is that the formulation of Lemma~\ref{lem:implicit-approximation-capacity} avoids introducing the auxiliary continuous piecewise linear function $\hat{\phi}$.
\end{rem}
Figure~\ref{fig:Gammaphi_finer_approximation} shows a finer approximation $\Gamma$ to $\Gamma_{\phi}$ (using a smaller value of $\delta$) and the corresponding decision boundary when restricting $\Gamma$ to $B^d_R$, i.e., the network's approximation to $\gamma_\phi$.
\begin{figure} 
\centering
\begin{subfigure}[b]{0.49\textwidth}
\centering
\includegraphics[width=\textwidth]{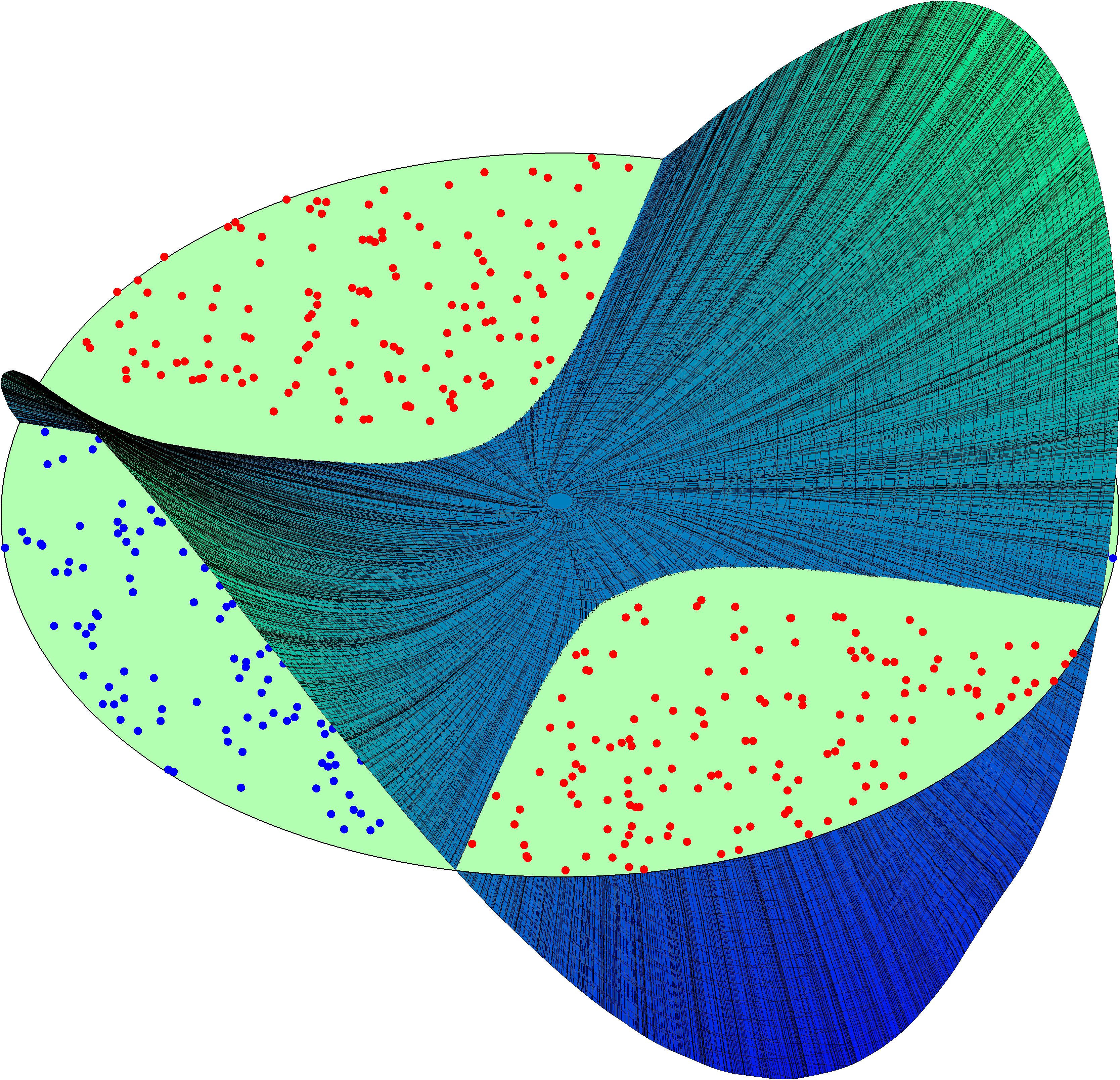}
\caption{ }
\end{subfigure}
\hfill
\begin{subfigure}[b]{0.45\textwidth}
\centering
\includegraphics[width=\textwidth]{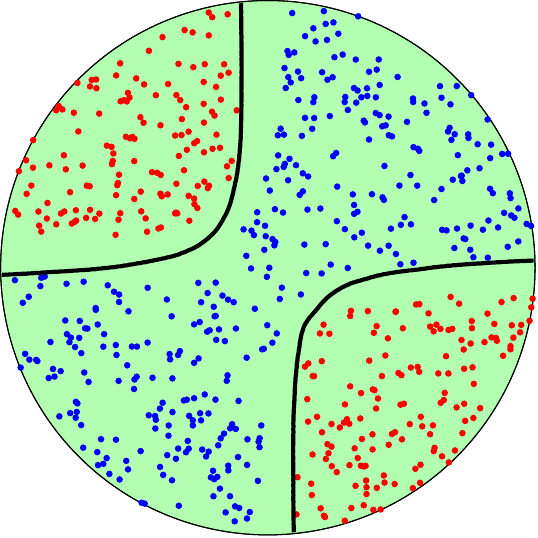}
\caption{ }
\end{subfigure}
\hfill
\caption{\textit{Finer Approximation.} \textbf{(a)} A smaller value of $\delta$ generates a network whose zero contour $\Gamma$ is a higher resolution approximation to $\Gamma_{\phi}$. \textbf{(b)} The restriction of $\Gamma$ to $B^d_R$ approximates the decision boundary $\gamma_\phi \subset B^d_R$ separating $X_1$ and $X_2$ in $\IR^d$, see Figure~\ref{fig:Datapoints}.} 
\label{fig:Gammaphi_finer_approximation}
\end{figure}

\section{Conclusions} \label{sec:conclusions}
We have developed an approximation theory for deep fully-connected ReLU networks. Based on the geometrical description of ReLU layers we proposed a modified network architecture. Each such network can explicitly be rewritten as a standard fully-connected ReLU network. A key observation was how these layers could be reduced to projections on hyperplanes along given directions.
 
Our result shows that if there exists a hypersurface separating two classes of points in a binary classification problem, then the decision boundary of a fully-connected ReLU network can approximate this surface to any given accuracy as long as it can be represented as a level-set of a $C^2$-function $\phi:B^d_R\to \IR$. Our proof is constructive and we have derived an explicit upper bound on the approximation error in terms of the discretization parameter $\delta$, input dimension $d$, level-set function $\phi$ and radius of the domain $R$. We have described in detail how the graph $\Gamma_{\phi}$ is mapped when propagating through the layers in the network and how the decision boundary emerges as the preimage of a hyperplane identified as the kernel of the final affine layer.

To achieve a given approximation error $\varepsilon$ our construction requires $C\big(\frac{R^2}{\varepsilon}\big)^{d+1}$ layers. This estimate is in line with the related work\cite{hanin2017approximating}, and the exponent $d$ seems unavoidable without any further assumption on the level-set function $\phi$. 

Apart from the rich expressiveness of the class of functions defined by a network architecture, the success of deep learning also stems from the fact that these networks can be trained properly in practice. This is often done by using some variant of the stochastic gradient descent (SGD) algorithm \cite{duchi2011adaptive, kingma2014adam} to solve an optimization problem. One key aspect that might partly explain why these learning algorithms are able to find feasible solutions could be that there are many different values of the network parameters that are sufficiently good for the problem at hand. 

In line with this, we discussed in Section~\ref{sec:modified-structure} that many different affine maps generate the same cone and many cones can be reduced to the same projections in our construction. Further, the order of the layers in each map $\Pi_k$ does not matter and the exact placement of the hyperplanes in the layers is not very important as long as they cover the ball in the $\epsilon$-net sense. In our proof, we explicitly defined the projection directions as the directional derivative of $\phi$ at a given point to obtain a good approximation accuracy. But as noted in Remark~\ref{rem:local accuracy}, in classification problems high accuracy is only essential in regions where the two sets of data are close to each other. Therefore, one may conclude that the choice of projection directions can be less restrictive on other parts of the domain. Similarly, if there are many possible separating hypersurfaces (recall Figure~\ref{fig:region}) it is enough to find a set of network parameters realizing an approximation of one of them.

On the other hand, from the geometrical description provided in Section~\ref{sec:2} it is clear that the layers cannot be chosen without any caution. If points from different classes are mapped to the same point it will not be possible to obtain a separating decision boundary. For instance, imagine if the cone of the first layer is positioned in the center of a point cloud of the data. Then there would be a high risk that the data gets mixed when it is projected onto the cone boundary. Especially in parts of the induced partition which are mapped to low-dimensional faces. Therefore, it seems reasonable to start projecting small parts from the boundary of the domain inwards towards the center as in the proposed construction in this paper.

\bigskip
\paragraph{Acknowledgement.} This research was supported in part by the Wallenberg AI, Autonomous Systems and
Software Program (WASP) funded by the Knut and Alice Wallenberg Foundation; the Swedish Research
Council Grants Nos.\  2017-03911,  2021-04925;  and the Swedish
Research Programme Essence. 

\bibliographystyle{habbrv}
\footnotesize{
\bibliography{ref_approx}
}

\bigskip
\noindent
\footnotesize{\bf Authors' addresses:}

\smallskip
\noindent
Jonatan Vallin  \quad \hfill \addressumushort\\
{\tt jonatan.vallin@umu.se}

\smallskip
\noindent
Karl Larsson, \quad \hfill \addressumushort\\
{\tt karl.larsson@umu.se}

\smallskip
\noindent
Mats G. Larson,  \quad \hfill \addressumushort\\
{\tt mats.larson@umu.se}

\normalsize

\pagebreak
\appendix
\section{Proofs}
In this appendix, we present the proofs of auxiliary lemmas and propositions that were omitted from the main text.
\small
\subsection{Proofs of Section~\ref{sec:structure}}
\label{Appendix A}
\begin{lemn}[\ref{lemma:ReLU_on_compact} (Projections on Hyperplanes)]
Let $K\subset \IR^{d+1}$ be a bounded set. Given a closed half-space $\mathbf{U}$ with the hyperplane $\mathbf{P}$ as its boundary and a vector $\xi \in \IR^{d+1}$ not parallel to $\mathbf{P}$, we can construct a projection $\pi:\IR^{d+1} \to S$ defined as in \eqref{eq:projection_pi} such that $\pi$ projects $  \mathbf{U}^c \cap K$ on $\mathbf{P}$ along $\xi$ while acting as the identity on $ \mathbf{U} \cap K$.
\end{lemn}
\begin{proof}
The action of $\pi$ on $K$ will depend on the corresponding cone $S$ and its intersection with $K$. As noted before, the cone can be derived given an affine map and conversely. We will show how to construct a matrix $A$ and a vector $b$ such that the affine map $x\mapsto Ax+b$ generates a cone $S$ for which the corresponding projection $\pi$ satisfies the conditions.
Fix an index $j\in I$ and define the $j$:th row vector $a_j$ in $A$ and the $j$:th scalar $b_j$ in $b$ such that $\mathbf{U}_j=\{x\in \IR^{d+1}:a_j\cdot x+b_j\geq 0\}$ coincides with the given half-space $\mathbf{U}$. Then, we decompose the set $K$ into $K^+=\mathbf{U} \cap K$ and $K^-=\mathbf{U}^c \cap K$.

We want to show that we can choose the other row vectors $a_i$ and the vector components $b_i$ for $i \in I\setminus\{j\}$ such that $\pi$ realizes a map projecting $K^-$ on the hyperplane $\mathbf{P}$ along the given direction $\xi$ while leaving $K^+$ unchanged. We continue by defining a set of $d$ linearly independent vectors $a_i$, $i\in I\setminus \{j\}$, in $\IR^{d+1}$ such that the line 
\begin{align}
L_j=\bigcap_{i\in I\setminus \{j\}} \{x\in \IR^{d+1}:a_i\cdot x=0\}
\label{eq:line_j}
\end{align} 
is parallel to the given direction $\xi$. The dual vector $a_j^*$ is thus parallel to $\xi$. Then, we can choose $b_i\in \IR$, $i\in I\setminus \{j\}$, such that 
\begin{align}
K \subset \bigcap_{i\in I\setminus \{j\}} \mathbf{U}_i^{o} 
\label{eq:omega_contained}
\end{align}
where each $\mathbf{U}_i^{o}=\{x\in \IR^{d+1}:a_i\cdot x+b_i > 0\}$ is the interior of the corresponding closed half-space $\mathbf{U}_i$ defined in \eqref{eq:halfspace}. Note that this is always possible since $K$ is a bounded set in $\IR^{d+1}$ and therefore we can translate each open half-space $\mathbf{U}^{\circ}_i$ in \eqref{eq:omega_contained}, by tuning $b_i$, until $K\subset \mathbf{U}^{\circ}_i$. 

With this set of parameters, the matrix $A$ will have full rank since the linearly independent vectors $\{a_i:i\in I\setminus \{j\}\}$ are all orthogonal to $\xi$ (by construction) which in turn is not orthogonal to $a_j$ (since $\xi$ is, by assumption, not parallel to the hyperplane $\mathbf{P}$ with normal $a_j$). Thus, with these parameters we will get a cone $S=\bigcap_{i\in I}\mathbf{U}_i$ such that
\begin{align}
\begin{split}
K^+&\subset S\\
K^-&\subset S_{(I\setminus \{j\},\{j\})}
\end{split}
\end{align}
Then, according to equation \eqref{eq:projection_pi} the corresponding projection $\pi$ will act as the identity on $K^+$ whereas $K^-$ will be projected along $\xi$ onto the hyperplane $\mathbf{P}$. The construction is depicted in Figure~\ref{fig:projection-bounded}.
\end{proof}

\begin{rem}
The vectors $a_i$, $i\in I\setminus \{j\}$ in the proof are not uniquely determined as the only requirements are that they are linearly independent and that the line in \eqref{eq:line_j} is parallel to the given vector $\xi$. In the same way, there are many different values of the scalars $b_i$, $i\in I\setminus \{j\}$ yielding the same action on $K$ by $\pi$. The condition in \eqref{eq:omega_contained} is met as long as we choose the scalars such that all the corresponding open half-spaces $\mathbf{U}_i^{\circ}$, $i\in I\setminus \{j\}$ contain $K$. Thus, moving these half-spaces such that their corresponding hyperplanes are even further away from $K$ will not change how $K$ is mapped by $\pi$. 
\label{rem:compact}
\end{rem}

\subsection{Proofs of Section~\ref{sec: projecting the graph}}
\label{Appendix B}
\begin{propn}[\ref{proposition:one_projection}]
Suppose $x\in U^c \cap \bar{B}_r^d$ and $t=\text{dist}(x,P)$. If $d>1$, then
\begin{align}
\recallLabel{eq:one_projection}
\boxed{
|\phi(x)+t\nabla_{\beta}\phi(p)-\phi(x+t\beta)|\leq C_1(d-1)\sqrt{r}\delta^{1/2}t
}
\end{align}
where $C_1$ is a constant independent of $\delta$, $d$ and $r$.
\end{propn}
\begin{proof}
We first introduce a new local coordinate system of $\IR^d$ situated at the center point $p$ of the base of the spherical cap. We name the corresponding new variables $u_1,u_2,\hdots,u_d$. Let the coordinate axis for the first coordinate $u_1$ be defined along the negative $\beta$-direction, i.e., in the direction pointing out from the half-space $U$. The other coordinate axes can then be chosen arbitrarily as long as the resulting coordinate system is orthogonal. Thus, the coordinate axes $u_2,\hdots,u_d$ meet at the center $p$ and lie in the hyperplane $P$. Figure \ref{fig:local_coordinates} illustrates the construction of the local coordinate system.
\begin{figure} 
\centering
\begin{subfigure}[b]{0.38\textwidth}
\centering
\includegraphics[width=\textwidth]{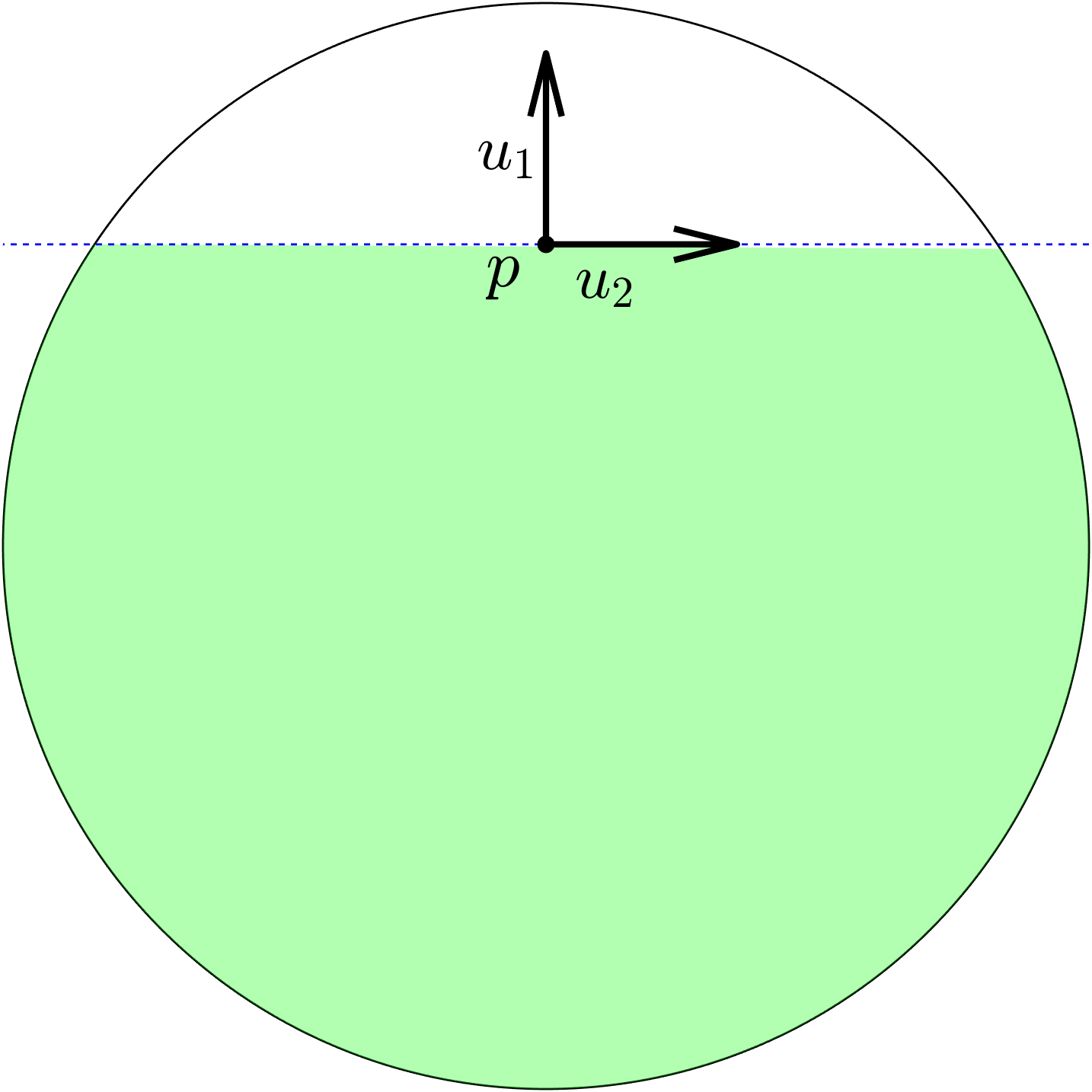}
\caption{ }
\end{subfigure}
\hfill
\begin{subfigure}[b]{0.61\textwidth}
\centering
\includegraphics[width=\textwidth]{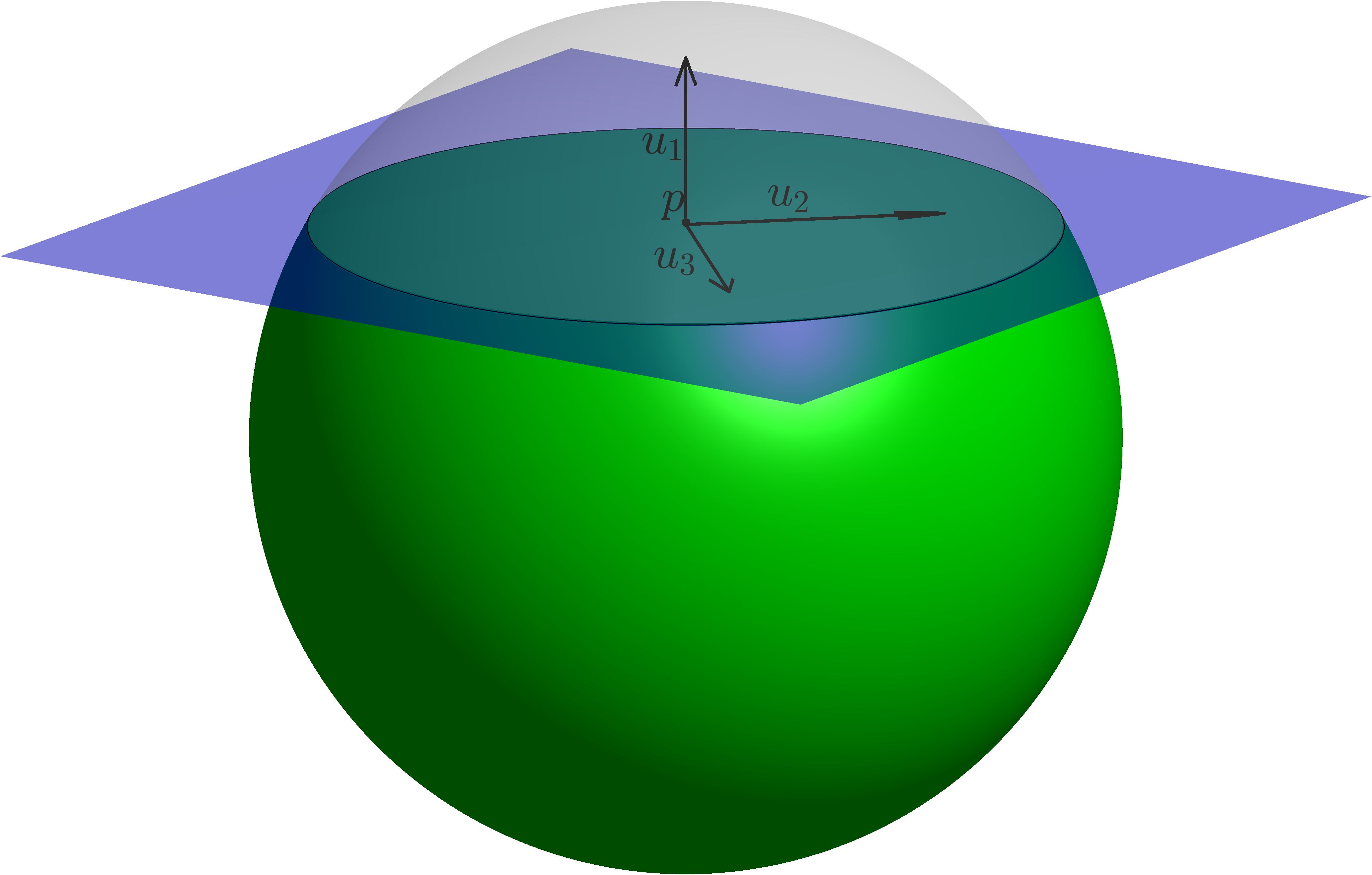}
\caption{ }
\end{subfigure}
\hfill
\caption{\textit{Local Coordinate System.} The construction of the new coordinate system in case \textbf{(a)} $d=2$ and   \textbf{(b)} $d=3$. The first coordinate axis is defined along the negative $\beta$-direction whereas the other coordinate axes are parallel to the hyperplane $P$. The new axes meet at the center point $p$.}
\label{fig:local_coordinates}
\end{figure}
We note that the origin in this coordinate system coincides with the point $p$ and if we let $\tilde{\phi}$ denote the function $\phi$ expressed in the new coordinates we see that $\nabla_{\beta} \phi(p)=-\frac{\partial \tilde{\phi}(0)}{\partial u_1}$. Since $\text{dist}(x,P)=t$, the point $x$ expressed in the new coordinate system will have coordinates $(t,u_2,\hdots,u_d):=(t, u_{2:d})$ for some $u_2,u_3,\hdots, u_d$ satisfying
\begin{align}
|u_i|\leq \sqrt{2r\delta-\delta^2}\leq \sqrt{2r}\delta^{1/2}, \qquad \text{for } i=2,3\hdots,d
\end{align}
as the point lies inside the spherical cap (see Figure \ref{fig:P1_geom}).
Similarly, as $x+t\beta$ lie in $P$ and $\beta$ points in the negative $u_1$ direction, this point will have coordinates $(0,u_{2:d})$. We can now express the quantity in \eqref{eq:quantity} in terms of the new variables
\begin{align}
|\phi(x)+t\nabla_{\beta}\phi(p)-\phi(x+t\beta)|=|\tilde{\phi}(t,u_{2:d})-t\frac{\partial \tilde{\phi}(0)}{\partial u_1}-\tilde{\phi}(0,u_{2:d})|
\label{eq:quantity_new}
\end{align}
By Taylor's theorem we can write 
\begin{align}
\tilde{\phi}(t,u_{2:d})=\tilde{\phi}(0,u_{2:d})+\frac{\partial \tilde{\phi}(0,u_{2:d})}{\partial u_1}t + \frac{1}{2}\frac{\partial^2 \tilde{\phi}(l_1t,u_{2:d})}{\partial u^2_1}t^2
\label{eq:taylor}
\end{align}
for some $l_1\in (0,1)$. Similarly, expanding $\frac{\partial \tilde{\phi}(0,u_{2:d})}{\partial u_1}$ in the expression above as a Taylor polynomial in the variables $u_2,\hdots,u_d$ with a first order remainder term gives
\begin{align}
\frac{\partial \tilde{\phi}(0,u_{2:d})}{\partial u_1}=\frac{\partial \tilde{\phi}(0,0)}{\partial u_1}+\sum_{i=2}^d\frac{\partial^2 \tilde{\phi}(0,l_2 u_{2:d})}{\partial u_i \partial u_1}u_i
\label{eq:taylor2}
\end{align}
for  some $l_2\in (0,1)$. Combining \eqref{eq:taylor} and \eqref{eq:taylor2} we get
\begin{align}
\tilde{\phi}(t,u_{2:d})=\tilde{\phi}(0,u_{2:d}) + \frac{\partial \tilde{\phi}(0,0)}{\partial u_1}t+\sum_{i=2}^d\frac{\partial^2 \tilde{\phi}(0,l_2 u_{2:d})}{\partial u_i \partial u_1}u_it +\frac{1}{2}\frac{\partial^2 \tilde{\phi}(l_1t,u_{2:d})}{\partial u^2_1}t^2
\end{align}
Using this in equation \eqref{eq:quantity_new} we obtain
\begin{align}
\begin{split}
|\phi(x)+t\nabla_{\beta}(p)-\phi(x+t\beta)|&=|\sum_{i=2}^d\frac{\partial^2 \tilde{\phi}(0,l_2 u_{2:d})}{\partial u_i \partial u_1}u_it +\frac{1}{2}\frac{\partial^2 \tilde{\phi}(l_1t,u_{2:d})}{\partial u^2_1}t^2|\\
&\leq \sum_{i=2}^d\bigg|\frac{\partial^2 \tilde{\phi}(0,l_2 u_{2:d})}{\partial u_i \partial u_1}\bigg||u_i|t+\frac{1}{2}\bigg|\frac{\partial^2 \tilde{\phi}(l_1t,u_{2:d})}{\partial u^2_1}\bigg|t^2\\
&\leq \sum_{i=2}^d D\sqrt{2r}\delta^{1/2}t + \frac{D}{2}t^2\\
&= (d-1)D\sqrt{2r}\delta^{1/2}t + \frac{D}{2}t^2\\
& \leq \frac{3D}{\sqrt{2}}(d-1)\sqrt{r}\delta^{1/2}t
\end{split}
\label{eq:bound_w}
\end{align}
where we used that $t\leq\delta< r$ and assumed that $d>1$ in the last inequality. Recall that $D$ is the constant bounding the second order derivatives of $\phi$. Thus,
\begin{align}
|\phi(x)+t\nabla_{\beta}(p)-\phi(x+t\beta)|\leq C_1(d-1)\sqrt{r}\delta^{1/2}t
\end{align}
where $C_1=\frac{3D}{\sqrt{2}}$.
\end{proof}
\begin{rem}
\label{rem:d=1}
If $d=1$, then from the second last inequality in \eqref{eq:bound_w} and the argument in the proof of the Corollary~\ref{cor:one_projection} shows that $\varepsilon$ scales as $\delta^2$ which is more favorable than $\sqrt{r}\delta^{3/2}$ as we get in the higher-dimensional cases. This is because if $d=1$, the size of the piece we cut off by the half-space only scales as $\delta$ since we do not have any of the additional directions perpendicular to $\beta$ that introduced a factor of $\sqrt{r}\delta^{1/2}$ into our estimate.
\end{rem}

\begin{lemn}[\ref{lem:PjPi_intersect}]
If two hyperplanes $P_j$ and $P_i$ in \eqref{eq:hyperplanes_halfspaces} intersect in $\bar{B}^d_r$ then
\begin{align}
 \recallLabel{eq:betajbetai_2}
\beta_j\cdot \beta_i\geq \frac{r^2-4r\delta+2\delta^2}{r}
\end{align}
\end{lemn}
\begin{proof}
Note, from \eqref{eq:hyperplanes_halfspaces} we have that $P_j$ and $P_j$ intersect in $\bar{B}^d_r$ if and only if the minimizer $x^*$ to the constrained optimization problem
\begin{align}
\begin{split}
\min_{x\in \IR^d} \quad &\frac{1}{2}\|x\|^2\\
\text{s.t:} \quad  &\beta_j\cdot x + (r-\delta)=0\\
 &\beta_i\cdot x + (r-\delta)=0
\end{split}
\label{eq:min_problem}
\end{align} 
satisfies $\|x^*\|^2\leq r^2$. To find the minimizer of problem \eqref{eq:min_problem} we construct the Lagrangian
\begin{align}
\mathcal{L}(x,\lambda)=\frac{1}{2}\|x\|^2 + \lambda_1(\beta_j\cdot x + (r-\delta))+\lambda_2(\beta_j\cdot x + (r-\delta)) 
\end{align}
with $\lambda=(\lambda_1,\lambda_2)$. Then, by solving the linear system
\begin{align}
\begin{split}
\nabla_x \mathcal{L}(x,\lambda)&=0\\
\nabla_\lambda \mathcal{L}(x,\lambda)&=0
\end{split}
\end{align}
we get the minimizer $x^*=-\frac{r-\delta}{1+\beta_j \cdot \beta_i}(\beta_j+\beta_i)$ with squared norm $\|x^*\|^2=\frac{2(r-\delta)^2}{1+\beta_j\cdot \beta_i}$. Setting $\|x^*\|^2\leq r^2$ and solving for $\beta_j\cdot \beta_i$ yields \eqref{eq:betajbetai_2}.
\end{proof}

\begin{lemn}[\ref{lemma:mulit_proj}] The following inequality holds
\begin{align} 
\recallLabel{eq:bound_path}
|y_m-\phi(x_m)|\leq C_1(d-1)\sqrt{r}\delta^{1/2}S_m
\end{align}
where $S_m$ is the length of the path defined by the sequence $(x_i)_{i=0}^m$.
\end{lemn}
\begin{proof}
From \eqref{eq:map-point} we have
\begin{align}
\begin{split}
y_m-\phi(x_m)&=y_{m-1}+t_m\nabla_{\beta_m}\phi(p_{m}) - \phi(x_m)\\
&=y_{m-1}-\phi(x_{m-1}) + \phi(x_{m-1})+t_m\nabla_{\beta_m}\phi(p_{m}) - \phi(x_m)
\end{split}
\end{align}
Rewriting the first two terms $y_{m-1}-\phi(x_{m-1})$ in the same way and then repeating the procedure $m$-times gives us
\begin{align}
y_m-\phi(x_m)=\sum_{i=1}^m \big(\phi(x_{i-1})+t_{i}\nabla_{\beta_i}\phi(p_i)-\phi(x_i)\big)
\end{align}
where we used that $y_0-\phi(x_0)=0$ by \eqref{eq:map_sequence}. By the triangle inequality we then have
\begin{align}
|y_m-\phi(x_m)|\leq \sum_{i=1}^m|\phi(x_{i-1})+t_{i}\nabla_{\beta_i}\phi(p_i)-\phi(x_i)|
\end{align}
Note that each term for which $t_i=0$ vanishes since then $x_{i}=x_{i-1}$ by \eqref{eq:map-point_2}. If $t_i\neq 0$ then $x_{i-1}\in U^c_i\cap \bar{B}_r$ and hence, we can apply Proposition~\ref{proposition:one_projection}. Thus, we can bound each term by
\begin{align}
\begin{split}
|\phi(x_{i-1})+t_{i}\nabla_{\beta_i}\phi(p_i)-\phi(x_i)|&=|\phi(x_{i-1})+t_{i}\nabla_{\beta_i}\phi(p_i)-\phi(x_{i-1}+t_i\beta_i)|\\
&\leq C_1(d-1)\sqrt{r}\delta^{1/2}t_i
\end{split}
\end{align}
We can now bound the sum as
\begin{align}
|y_m-\phi(x_m)|\leq C_1(d-1)\sqrt{r}\delta^{1/2}\sum_{i=1}^mt_i=C_1(d-1)\sqrt{r}\delta^{1/2}S_m
\end{align}
as desired.
\end{proof}

\begin{lemn}[\ref{lemma:bound_ti}] Let $t_i$ be defined as in \eqref{eq:def_ti}, then
\begin{align}
\recallLabel{eq:bound_ti}
t_i\leq \frac{\|x_{i-1}\|^2-\|x_i\|^2}{2(r-\delta)}
\end{align}
\end{lemn}
\begin{proof}
If $t_i=0$ the inequality follows trivially, so assume that $0<t_i=\text{dist}(x_{i-1},P_i)=\|x_{i-1}-x_i\|$.
We start by deriving an expression of $\|x_i\|^2$ in terms of $\|x_{i-1}\|^2$ and $t_i$. Note that all the points $p_i$, $x_i$ and $x_{i-1}$ lie in the 2-dimensional plane spanned by $\beta_i$ and the vector pointing from $p_i$ to $x_i$. Figure \ref{geom2} illustrates the geometrical setup in this plane.
\begin{figure} 
\centering
\includegraphics[scale=0.9]{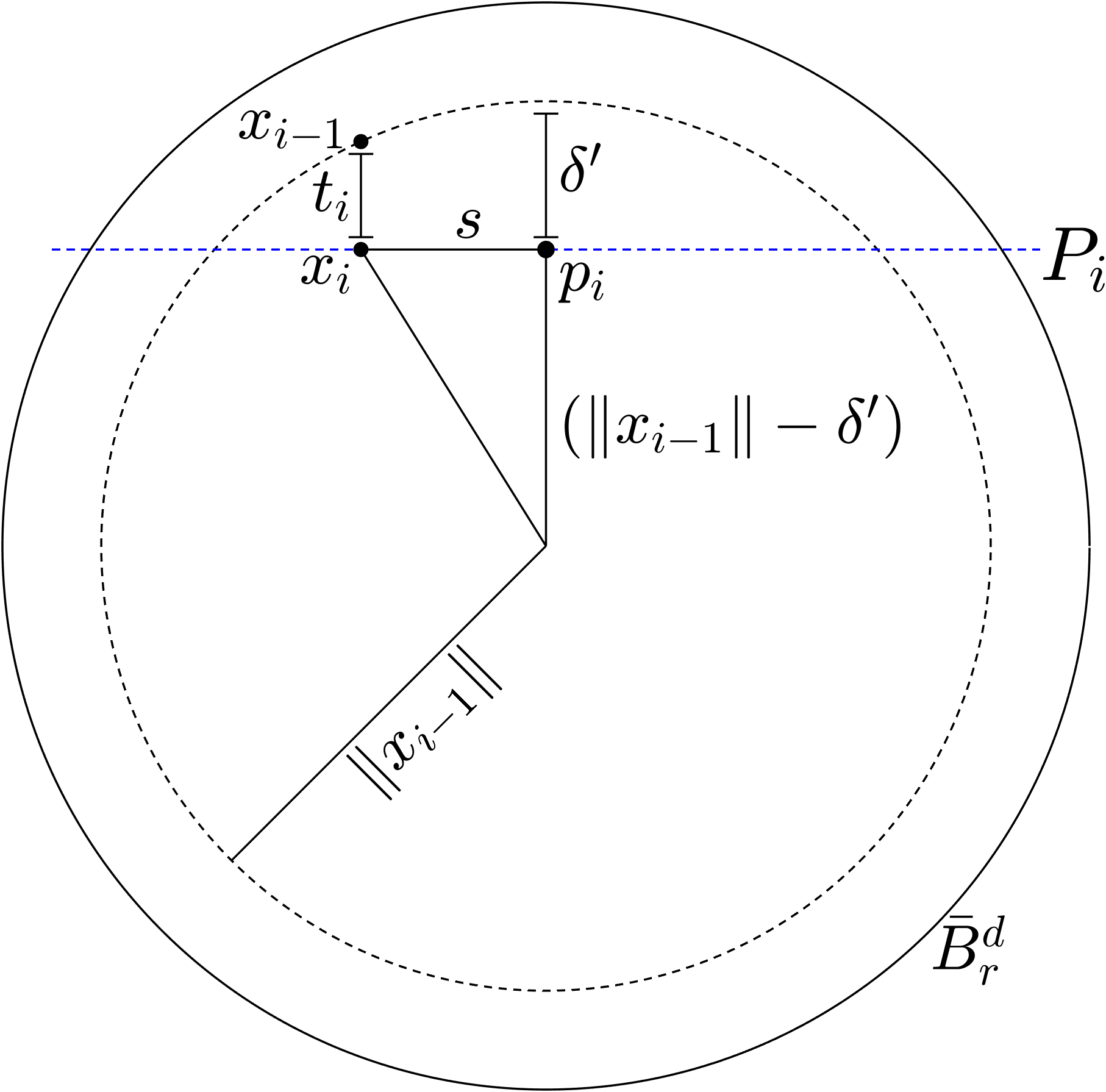}
\caption{\textit{Geometrical Setup.} The origin, the point $p_i$ and $x_i$ define a right triangle where the length of the hypotenuse is $\|x_i\|$. Calculating the side length $s$ allows us to compute $\|x_i\|$.}
\label{geom2}
\end{figure}
Let $\delta'$ be the height of the spherical cap $U^c_i\cap \bar{B}^d_{\|x_{i-1}\|}$ as depicted in the figure, then since $\delta$ is the height of the spherical cap $U^c_i\cap \bar{B}^d_{r}$ it follows that
\begin{align}
\delta'=\delta-(r-\|x_{i-1}\|)
\label{eq:delta_prime}
\end{align}
Now, consider the right triangle in this plane with vertices at the origin, $p_i$ and $x_i$. The hypotenuse has length $\|x_i\|$ which we are interested in and the side from the origin to $p_i$ has length $(\|x_{i-1}\|-\delta')$. We denote by $s$, the length of the third side between $p_i$ and $x_i$, and the Pythagorean theorem gives the relation
\begin{align}
\|x_i\|^2=(\|x_{i-1}\|-\delta')^2+s^2
\label{eq:x_i}
\end{align}

We can calculate the side length $s$ by considering the spherical cap resulting from translating the hyperplane $P_i$ a distance $t_i$ such that it passes through $x_{i-1}$. Then, the radius of its base is precisely $s$. Since the height of the new spherical cap is $\delta'-t_i$ we get that
\begin{align}
s=\sqrt{2\|x_{i-1}\|(\delta'-t_i)-(\delta'-t_i)^2}
\end{align} 
Plugging this into \eqref{eq:x_i} gives
\begin{align}
\begin{split}
\|x_i\|^2&=(\|x_{i-1}\|-\delta')^2+s^2=(\|x_{i-1}\|-\delta')^2+2\|x_{i-1}\|(\delta'-t_i)-(\delta'-t_i)^2\\
&= \|x_{i-1}\|^2-2\|x_{i-1}\|t_i+2\delta't_i-t_i^2
\end{split}
\end{align}
Using the relation in \eqref{eq:delta_prime} in the equation above yields
\begin{align}
\|x_i\|^2= \|x_{i-1}\|^2-2(r-\delta)t_i-t_i^2
\end{align}
and thus
\begin{align}
\frac{\|x_{i-1}\|^2-\|x_i\|^2}{2(r-\delta)}=\frac{2(r-\delta)t_i+t_i^2}{2(r-\delta)}=t_i+\frac{t_i^2}{2(r-\delta)}\geq t_i
\end{align}
which we wanted to show.

In case $x_{i-1}$ lies on the line passing through the origin and $p_i$, then the triangle described above is degenerate as $x_i=p_i$. However, in this case we have the equality $\|x_{i-1}\|-\|x_i\|=t_i$ and thus
\begin{align}
\begin{split}
\frac{\|x_{i-1}\|^2-\|x_i\|^2}{2(r-\delta)}&=\frac{(\|x_{i-1}\|-\|x_i\|)(\|x_{i-1}\|+\|x_i\|)}{2(r-\delta)}= \frac{t_i(\|x_{i-1}\|+\|x_i\|)}{2(r-\delta)}\\
&\geq \frac{t_i((r-\delta)+(r-\delta))}{2(r-\delta)}=t_i
\end{split}
\end{align}
where the last inequality follows from the fact that neither of $x_{i-1}$ and $x_i$ are contained in $B^d_{r-\delta}$. Hence, the inequality also holds in this special case.
\end{proof}

\begin{lemn}[\ref{lemma:P_l_prime}]
The construction of the polytopes guarantees that the inclusions
\begin{align}
\recallLabel{eq:P_l_prime}
 \bar{B}^d_{r_{k}-\delta_k} \subset \mathcal{P}_{k+1} \subset \bar{B}^d_{r_{k}}
\end{align}
hold for all $0\leq k \leq M-1$.
\end{lemn}
\begin{proof}
The first inclusion holds trivially by the construction of $\mathcal{P}_{k+1}$, recall \eqref{eq:B_r-delta-inclusion}. To prove the second inclusion we will show that $S^{d-1}_{r_{k}} \subset (\mathcal{P}_{k+1})^c$ because then $\mathcal{P}_{k+1}\subset (S^{d-1}_{r_{k}})^c$ but since $\mathcal{P}_{k+1}$ is connected and contains the origin, the last inclusion implies that $\mathcal{P}_{k+1}\subset \bar{B}^d_{r_{k}}$.

Let $N_{\epsilon_k}$ be an $\epsilon_k$-net of $S^{d-1}_{r_{k}}$. Take a point $q\in S^{d-1}_{r_{k}}$, then there is a point $q_i\in N_{\epsilon_k}\subset S^{d-1}_{r_{k}}$ such that $\|q_i-q\|\leq \epsilon_k$. With $\epsilon_k=\sqrt{\frac{1}{2}\delta_{k} r_{k}}$ we thus have $\|q_i-q\|\leq \sqrt{\frac{1}{2}\delta_{k} r_{k}}$. By the definition of $U^{k+1}_i$, we have that $U^{k+1}_i \cap \bar{B}^d_{r_{k}}$ is a spherical cap of height $\delta_{k}$ and the radius of the base is precisely $\sqrt{2r_{k}\delta_{k}-\delta_{k}^2}$. As $q_i$ is the point on the top of this spherical cap, and thus a distance $\delta_{k}$ from $U^{k+1}_i$, it follows from the Pythagorean theorem that
\begin{align}
\min(\{\|q_i-x\|:x\in S^{d-1}_{r_{k}}\cap U^{k+1}_i\})=\sqrt{2\delta_{k} r_{k}}
\end{align}
Thus, as $\|q_i-q\|\leq \sqrt{\frac{1}{2}\delta_{k} r_{k}}<\sqrt{2\delta_{k} r_{k}}$ we must have $q\in (U^{k+1}_i)^c\subset (\mathcal{P}_{k+1})^c$. Hence, $S^{d-1}_{r_{k}}\subset  (\mathcal{P}_{k+1})^c$ so $\mathcal{P}_{k+1} \subset \bar{B}^d_{r_{k}}$.
\end{proof}

\begin{lemn}[\ref{lem:num_polytopes} (Number of Polytopes)]
The number of polytopes $M(\delta)$ needed in our construction to ensure $r_M\leq \delta$ satisfies
\begin{align}
\recallLabel{eq:num_polytopes}
\boxed{
M(\delta)\leq \frac{7R}{3\delta}
}
\end{align}
\end{lemn}
\begin{proof}
Recall that $\delta_k$ is either set to $\delta$ or the value $r_k\left(1-\frac{1}{\sqrt{2}}\right)$ depending on the relative size of $r_k\left(1-\frac{1}{\sqrt{2}}\right)$ and $\delta$, see equation \eqref{eq:definition_delta_k}. Now, if we define 
\begin{align}
r^*=\frac{\delta}{1-\frac{1}{\sqrt{2}}}
\end{align}
then $\delta_k=\delta$ as long as $r_k \geq r^*$. We can also express
\begin{align}
M(\delta)=M_1(\delta)+M_2(\delta)
\label{eq:M=M_1+M_2}
\end{align}
where $M_1(\delta)$ represent the number of polytopes $\mathcal{P}_k$ in our construction for which $r_k\geq r^*$ and $M_2(\delta)$ represent the number of polytopes for which $r_k<r^*$.

Note, for all $k$ with $r_k \geq r^*$ we have by Lemma~\ref{proposition:inclusion} that $r_k\leq R-\frac{3}{4}\delta k$. Solving the equation $R-k\frac{3}{4}\delta= r^*$ for $k$ gives $k= \frac{4R}{3\delta}-\frac{4}{3\left(1-\frac{1}{\sqrt{2}}\right)}$. Thus we can write
\begin{align}
M_1(\delta)\leq \bigg\lceil \frac{4R}{3\delta}-\frac{4}{3\left(1-\frac{1}{\sqrt{2}}\right)}\bigg\rceil \leq  \frac{4R}{3\delta} - 3
\label{eq:bound-M_1}
\end{align}
where $\lceil \cdot \rceil$ is the ceiling function. We can now estimate $M_2(\delta)$ by considering how many polytopes we need in our construction to go from the radius $r^*$ to $\delta$. In this case we have that $\delta_k=r_k\left(1-\frac{1}{\sqrt{2}}\right)$ so by Lemma~\ref{proposition:inclusion}
\begin{align}
r_{k+1}\leq r_k-\frac{3}{4}\delta_k=r_k-\frac{3\left(1-\frac{1}{\sqrt{2}}\right)r_k}{4}=\frac{3+\sqrt{2}}{4\sqrt{2}}r_k
\label{eq:r_k+1_less}
\end{align}
If $k^*$ is the first index such that $r_{k^*}<r^*$ then for any $l\in \mathbb{N}$ we get 
\begin{align}
r_{k^*+l}\leq\big( \frac{3+\sqrt{2}}{4\sqrt{2}}\big)^lr_{k^*}<\big( \frac{3+\sqrt{2}}{4\sqrt{2}}\big)^lr^*
\end{align}
Solving $\big( \frac{3+\sqrt{2}}{4\sqrt{2}}\big)^lr^*= \delta$ gives $l= \frac{\ln\left(1-\frac{1}{\sqrt{2}}\right)}{\ln{\big(\frac{\sqrt{2}+3}{4\sqrt{2}}\big)}}$ so we get
\begin{align}
M_2(\delta)\leq \left\lceil \frac{\ln\left(1-\frac{1}{\sqrt{2}}\right)}{\ln{\big(\frac{\sqrt{2}+3}{4\sqrt{2}}\big)}}\right\rceil = 5
\label{eq:bound-M_2}
\end{align}
Combining \eqref{eq:bound-M_1} and \eqref{eq:bound-M_2} in \eqref{eq:M=M_1+M_2} gives the bound
\begin{align}
M(\delta)= M_1+M_2\leq \frac{4R}{3\delta} + 2\leq \frac{7R}{3\delta}
\label{eq:M(delta)-bound}
\end{align}
where the last inequality holds since $\delta \leq R\left(1-\frac{1}{\sqrt{2}}\right)$. 
\end{proof}

\begin{lemn}[\ref{lem:cardinality of net}]
There is an $\epsilon$-net $N_{\epsilon}$ of the sphere $S^{d-1}_r$ such that $|N_{\epsilon}|\leq 2d\big(1+\frac{2r}{\epsilon}\big)^{d-1}$. 
\end{lemn}
\begin{proof}
Let $N^{sep}_{\epsilon}$ be a maximal $\epsilon$-separation of $S^{d-1}_r$. Since it is an $\epsilon$-separation, the balls $\{B^d_{\frac{\epsilon}{2}}(p):p\in N^{sep}_{\epsilon}\}$ are disjoint. Moreover, the disjoint union $\bigcup_{p\in N^{sep}_{\epsilon}}B^d_{\frac{\epsilon}{2}}(p)$ is contained in the spherical shell $B^d_{r+\frac{\epsilon}{2}}\setminus B^d_{r-\frac{\epsilon}{2}}$. By comparing volumes we get the inequality
\begin{align}
|N^{sep}_{\epsilon}|\cdot \vol(B^d_{\frac{\epsilon}{2}})\leq \vol(B^d_{r+\frac{\epsilon}{2}})-\vol(B^d_{r-\frac{\epsilon}{2}})
\end{align}
where $\vol(B^d_r)$ denotes the volume of a $d$ dimensional ball of radius $r$. Now, we can estimate the right-hand side (assuming $d>1$)
\begin{align}
\begin{split}
\vol(B^d_{r+\frac{\epsilon}{2}})-\vol(B^d_{r-\frac{\epsilon}{2}})
&=\frac{\pi^{d/2}}{\Gamma(\frac{d}{2}+1)}\bigg((r+\epsilon/2)^d-(r-\epsilon/2)^d\bigg)
\\&
=\frac{\pi^{d/2}}{\Gamma(\frac{d}{2}+1)}\int_{r-\epsilon/2}^{r+\epsilon/2}d\cdot \rho^{d-1}d\rho\\
&\leq \frac{\pi^{d/2}}{\Gamma(\frac{d}{2}+1)}d(r+\epsilon/2)^{d-1}\int_{r-\epsilon/2}^{r+\epsilon/2} d\rho
\\&
=\frac{\pi^{d/2}d(r+\epsilon/2)^{d-1}\epsilon}{\Gamma(\frac{d}{2}+1)}
\end{split}
\end{align}
Dividing by $\vol(B^d_{\frac{\epsilon}{2}})$ gives the upper bound
\begin{align}
|N^{sep}_{\epsilon}|\leq 2d\bigg(1+\frac{2r}{\epsilon}\bigg)^{d-1}
\end{align}
The result now follows from Lemma~\ref{lemma:sep-net}. 
\end{proof}

\subsection{Proofs of Section~\ref{sec:Decision Boundary}}
\label{Appendix C}
\begin{lemn}[\ref{lem:preimage pi} (Preimages)]
Let $g:B^d_R\to \IR$ be a continuous function and denote its graph by $\Gamma_g$. Consider a layer $\pi^k_i:B^d_R \times \IR \to \IR^{d+1}$ in our network $\tilde{F}$. Then, the preimage of $\Gamma_g$ under $\pi^k_i$ fulfills
\begin{align}
\recallLabel{eq: preimage of Gamma_g}
(\pi^k_i)^{-1}[\Gamma_g]=\Gamma_{\tilde{g}}
\end{align}
where $\Gamma_{\tilde{g}}$ is the graph of a continuous function $\tilde{g}:B^d_R\to \IR$. Moreover, given any $\tau\geq 0$ we have
\begin{align}
\recallLabel{eq: preimage of Sigma}
(\pi^k_i)^{-1}[\Sigma^{\tau}_{g}(B^d_R)]=\Sigma^{\tau}_{\tilde{g}}(B^d_R)
\end{align}
\end{lemn}
\begin{proof}
To simplify the notation we will drop the sub- and superscript on the layer and simply denote it by $\pi$. Let $U$ be the closed half-space, associated with the layer, with boundary hyperplane $P$ and inward pointing unit normal $\beta$. Further let $\xi \in \IR^{d+1}$ be the projection direction and let $p\in \IR^d$ be the point such that $\xi=(\beta,\nabla_{\beta}\phi(p))$. Then it follows that 
\begin{align}
\pi^{-1}(\Gamma_g\cap \mathbf{U}^{\circ})=\Gamma_g\cap \mathbf{U}^{\circ}
\label{eq:preimage interior}
\end{align}
as $\pi$ act as the identity on the interior $\mathbf{U^{\circ}}\subset \IR^{d+1}$ and any point in $(\mathbf{U}^{\circ})^c$ will be mapped to $\mathbf{P}$ disjoint with $\mathbf{U}^{\circ}$. Therefore, the only part of $\Gamma_g$ with a non empty and a non trivial preimage is precisely $\Gamma_g\cap \mathbf{P}$. 

For a point $(x',g(x'))\in \Gamma_g\cap \mathbf{P}$, we have that $x'\in P \cap B^d_R$. From \eqref{eq:map-point} we can then immediately see that a point $(x,y)$ will be mapped to $(x',g(x'))$ if and only if
\begin{align}
\begin{split}
x&=x'-t\beta\\
y&=g(x')-t\nabla_{\beta}\phi(p)
\end{split}
\end{align}
for some $t\geq 0$. Therefore, we can express the preimage of this point as
\begin{align}
\pi^{-1}[(x',g(x'))]=\big\{(x,y)\in B^d_R\times \IR: x=x'-t\beta, y=g(x')-t\nabla_{\beta}\phi(p), t\geq 0\big\}
\label{eq:preimage of point}
\end{align}
and the entire preimage of $\Gamma_g\cap \mathbf{P}$ can be written as
\begin{align}
\begin{split}
\pi^{-1}&[\Gamma_g\cap \mathbf{P}]=\\
&\big\{(x,y)\in B^d_R\times \IR: x=x'-t\beta, y=g(x')-t\nabla_{\beta}\phi(p), t\geq 0, x'\in P\cap B^d_R\big\}
\end{split}
\label{eq:preimage intersection}
\end{align}
Fixing $t=0$ in the set builder above we see that, in particular,
\begin{align}
\Gamma_g\cap \mathbf{P} \subseteq \pi^{-1}&[\Gamma_g\cap \mathbf{P}]
\label{eq:intersectino in preimage}
\end{align}
Further, as $t\geq 0$ we will have that $x'-t\beta \in (P\cup U^c)\cap B^d_R$. This suggests that we can define the function $\tilde{g}:(P\cup U^c)\cap B^d_R \to \IR$ by
\begin{align}
\tilde{g}(x)=g(x')-t\nabla_{\beta}\phi(p)
\label{eq:tilde g definition}
\end{align}
where $x'\in P\cap B^d_R$ and $t\geq 0$ are defined by the equation $x=x'-t$. However, we must show that for all $x\in (P\cup U^c)\cap B^d_R$ we can find such a point $x'$ and non-negative scalar $t$ and that they are unique so the function $\tilde{g}$ is well-defined.

Given a point $x\in (P\cup U^c)\cap B^d_R$, we can define the unique hyperplane $P_x\subset (P\cup U^c)$ parallel to $P$ such that $x\in P_x\cap B^{d}_R$. Then we let $t=$dist$(P_x,P)\geq 0$ and we define $x'$ to be the orthogonal projection of $x$ onto $P$. Clearly, $x'\in P$ but as we also know that $U$ contains the origin (this is true for the associated half-spaces to all the layers in our construction) the ball $P_x\cap B^d_R$ has smaller or equal radius than the parallel ball $P\cap B^d_R$ and hence we have that $x'\in P \cap B^d_R$. Then, by construction we have $x=x'-t\beta$.

To show the uniqueness, suppose
\begin{align}
x'_1-t_1\beta=x'_2-t_2\beta
\end{align}
for $x'_1,x'_2\in P\cap B^d_R$ and $t_1,t_2 \geq 0$. We can rewrite the equation as 
\begin{align}
x'_1-x'_2=(t_1-t_2)\beta
\label{eq:uniqueness}
\end{align}
As both the points $x_1',x_2'$ are contained in $P$, the vector on the left-hand side in \eqref{eq:uniqueness} is perpendicular to the vector on the right-hand side. Thus, for the equality to hold we need both sides to vanish meaning that $x'_1=x'_2$ and $t_1=t_2$. Hence, we can conclude that $\tilde{g}$ defined in \eqref{eq:tilde g definition} is a valid and well-defined function. 

We can also note that if $x\in P \cap B^d_R$ then $\tilde{g}(x)=g(x)$ since $t=0$ in this case so $x=x'$. Therefore, we can continuously extend $\tilde{g}$ to the entire set $B^d_R$ by letting $\tilde{g}:B^d_R\to \IR$ be defined by
\begin{align}
\tilde{g}(x)=\begin{cases} 
g(x), \quad \text{if } x\in U\cap B^d_R \\
g(x')-t\nabla_{\beta}\phi(p), \quad \text{if } x\in U^c\cap B^d_R
\end{cases}
\label{eq:tilde g extended definition}
\end{align}
By comparing the definition of $\tilde{g}$ in \eqref{eq:tilde g extended definition} with \eqref{eq:preimage interior} and \eqref{eq:preimage intersection} we can easily see that

\begin{align}
\pi^{-1}[\Gamma_g]=\Gamma_{\tilde{g}}
\end{align}
The preimage of the graph of $\phi$ under such a layer is shown in Figure~\ref{fig:Preimagefirst}.

To prove the second part of the lemma, suppose $\tau \geq 0$ and consider the set $\Sigma^{\tau}_{g}(B^d_R)$. Similarly as before, we have that
\begin{align}
\pi^{-1}[\Sigma^{\tau}_{g}(B^d_R) \cap \mathbf{U}^{\circ}]=\Sigma^{\tau}_{g}(B^d_R) \cap \mathbf{U}^{\circ}=\Sigma^{\tau}_{\tilde{g}}(B^d_R) \cap \mathbf{U}^{\circ}
\label{eq:preimage of Sigma 1}
\end{align}
where the second equality follows from the fact that $\tilde{g}=g$ on $U\cap B^d_R$. Thus, the only part with non empty and non trivial preimage is exactly $\Sigma^{\tau}_{g}(B^d_R) \cap \mathbf{P}$. Therefore, consider a point $(x,y)\in \pi^{-1}[ \Sigma^{\tau}_{g}(B^d_R)\cap \mathbf{P}]$. Then by definition of preimages we have that $\pi(x,y)=(x',y')\in \Sigma^{\tau}_{g}(B^d_R)\cap \mathbf{P}$. By definition of $\Sigma^{\tau}_{g}(B^d_R)\cap \mathbf{P}$ we need $x'\in B^d_R \cap P$ and $y'=g(x')+h$ with $|h|\leq \tau$. Using equation \eqref{eq:map-point} we therefore need
\begin{align}
\begin{split}
x&=x'-t\beta\\
y&=g(x')+h-t\nabla_{\beta}\phi(p)=\tilde{g}(x)+h
\end{split}
\end{align}
for some $t\geq 0$. But $(x,\tilde{g}(x)+h)\in \Sigma^{\tau}_{\tilde{g}}(B^d_R)\cap (\mathbf{P}\cup \mathbf{U}^c)$ so $\pi^{-1}[\Sigma^{\tau}_{g}(B^d_R)\cap \mathbf{P}]\subseteq \Sigma^{\tau}_{\tilde{g}}(B^d_R)\cap \mathbf{U}^c$. If we similarly consider a point $(x,y)\in \Sigma^{\tau}_{\tilde{g}}(B^d_R)\cap (\mathbf{P}\cup \mathbf{U}^c)$ we have that $x\in P\cup U^c$ and $y=\tilde{g}(x)+h$ with $|h|\leq\tau$. As $x\in P\cup U^c$ we can express it as $x=x'-t\beta$ for some $x'\in P$ and $t\geq 0$. Thus, $\tilde{g}(x)=g(x')-t\nabla_{\beta}\phi(p)$. 

Now, evaluating $\pi(x,y)$ using \eqref{eq:map-point} yields
\begin{align}
\pi(x,y)=(x',g(x')+h)\in \Sigma^{\tau}_{g}(B^d_R)\cap \mathbf{P}
\end{align}
Thus, $(x,y)\in  \pi^{-1}[ \Sigma^{\tau}_{g}(B^d_R)\cap \mathbf{P}]$ so $\Sigma^{\tau}_{\tilde{g}}(B^d_R)\cap (\mathbf{P}\cup \mathbf{U}^c) \subseteq \pi^{-1}[ \Sigma^{\tau}_{g}(B^d_R)\cap \mathbf{P}]$. We can now conclude that
\begin{align}
\pi^{-1}[ \Sigma^{\tau}_{g}(B^d_R)\cap \mathbf{P}]=\Sigma^{\tau}_{\tilde{g}}(B^d_R)\cap (\mathbf{P}\cup \mathbf{U}^c)
\label{eq:preimage of Sigma 2}
\end{align}
The two equations \eqref{eq:preimage of Sigma 1} and \eqref{eq:preimage of Sigma 2} now imply \eqref{eq: preimage of Sigma}.
\end{proof}

\end{document}